\documentclass[compsoc, conference, a4paper, 10pt, times]{IEEEtran}

\usepackage[hidelinks]{hyperref}

\usepackage{subfloat}
\usepackage{amsfonts}
\usepackage{amsthm}
\usepackage{booktabs}
\usepackage{multirow}
\usepackage{multicol}
\usepackage{subfigure}
\usepackage{comment}
\usepackage[utf8]{inputenc}

\usepackage{lineno}
\usepackage{color}

\usepackage[shortlabels]{enumitem}
\usepackage{thm-restate}

\usepackage{algorithm2e}

\usepackage{fancyhdr}
\usepackage{graphicx}
\usepackage{overpic} 
\usepackage{amsmath} 
\usepackage{bm}      

\usepackage{diagbox}
\usepackage{cite}
\usepackage{ulem}
\usepackage{float}
\usepackage{textpos}
\usepackage{algorithmic}

\newtheorem{theorem}{Theorem}
\newtheorem{definition}{Definition}
\newtheorem{lemma}{Lemma}

\begin{document}

\title{A Certified Radius-Guided Attack Framework to Image Segmentation Models}



\author{%
  \IEEEauthorblockN{%
    Wenjie Qu\IEEEauthorrefmark{1}\textsuperscript{\textsection},
    Youqi Li\IEEEauthorrefmark{2}\textsuperscript{\textsection},
    Binghui Wang\IEEEauthorrefmark{3}\textsuperscript{$\clubsuit$}%
  }%
  \IEEEauthorblockA{\IEEEauthorrefmark{1} Huazhong University of Science and Technology, Wuhan, China, wen\_jie\_qu@outlook.com}%
  \IEEEauthorblockA{\IEEEauthorrefmark{2} School of Computer Sci. \& Tech., Beijing Institute of Technology, Beijing, China, liyouqi@bit.edu.cn}%
  \IEEEauthorblockA{\IEEEauthorrefmark{3} Department of Computer Science, Illinois Institute of Technology, Chicago, USA, bwang70@iit.edu}%
}

\maketitle
\begingroup\renewcommand\thefootnote{\textsection}
\footnotetext{Equal contribution}
\endgroup
\begingroup\renewcommand\thefootnote{$\clubsuit$}
\footnotetext{Wenjie did this research when he was an intern in Wang's group}
\endgroup

\maketitle

\begin{abstract}
Image segmentation is an important problem in many safety-critical applications such as medical imaging and autonomous driving. 
Recent studies show that modern image segmentation models are vulnerable to adversarial perturbations, while existing attack methods mainly follow the idea of attacking image classification models. We argue that image segmentation and classification have inherent differences, and design an attack framework specially for image segmentation models.  
Our goal is to thoroughly explore the vulnerabilities of  modern segmentation models, i.e., aiming to misclassify as many pixels as possible under a perturbation budget in both white-box and black-box settings.

Our attack framework is 
inspired by certified radius, which was originally used by \textit{defenders} 
to defend against adversarial perturbations to classification models. 
We are the first, from the \textit{attacker} perspective,
to leverage the properties of certified radius and propose a certified radius guided attack framework {against} 
image segmentation models. 
Specifically, 
we first adapt randomized smoothing, the state-of-the-art certification method for classification models, to derive the 
pixel's  
certified radius.  
A larger certified radius of a pixel 
means the pixel is \textit{theoretically} more robust to adversarial 
perturbations. 
This observation inspires us to focus more on disrupting pixels with relatively smaller certified radii. 
Accordingly, we design a pixel-wise certified radius guided loss, when plugged into \textit{any} existing white-box 
attack, yields our certified radius-guided white-box attack.

Next, we propose the first black-box attack to image segmentation models via bandit. 
A key challenge 
is no gradient information is available. To address it, we design a novel gradient estimator, based on bandit feedback, which 
is query-efficient and provably unbiased and stable. We use this gradient estimator to design a projected bandit gradient descent (PBGD) attack. We further use pixels' certified radii and design 
a certified radius-guided PBGD (CR-PBGD) attack.  
We prove our PBGD and CR-PBGD attacks can achieve asymptotically optimal attack performance with an optimal rate.  
We evaluate our certified-radius guided white-box and black-box attacks on multiple modern image segmentation models and datasets. Our results validate the effectiveness of our certified radius-guided attack framework. 

\end{abstract}

\section{Introduction}
\label{sec:intro}

Image segmentation 
(also called pixel-level classification) 
is the task of labeling pixels in an image such that pixels belonging to the same object are assigned a same label. 
Image segmentation is an important problem in many safety-critical applications such as medical imaging~\cite{harmon2020artificial} (e.g., tumor detection), autonomous driving~\cite{deng2017cnn} (e.g., traffic sign detection). 
However, recent studies~\cite{xie2017adversarial,fischer2017adversarial,cisse2017houdini,hendrik2017universal,arnab2018robustness} showed that modern image segmentation models~\cite{zhao2017pyramid,zhao2018psanet,sun2019high} are vulnerable to adversarial perturbations: A carefully designed imperceptible perturbation to a testing image can mislead an image segmentation model to 
misclassify a substantial number of pixels in this image. 
Such vulnerabilities would  cause serious consequences
in the  safety-critical applications. 
For example, an attacker can attack a traffic sign prediction system such that a ``STOP" sign image 
after segmentation will be misclassified as, e.g., ``SPEED" or ``LEFTTURN"---This causes safety issues. 
For another example, insurance companies often use their disease diagnosis systems to test medical images before reimbursing a medical claim. However, an attacker (e.g., an insider) can fool the disease diagnosis systems (e.g., no tumor diagnosed as tumor) by imperceptibly modifying the attacker's medical images and sending fraud insurance claims—This causes insurance companies' financial loss.

While these recent attack methods to image segmentation models have been proposed, they majorly follow the idea of attacking image \textit{classification} models based on, e.g.,  the Fast Gradient Sign (FGSM)~\cite{goodfellow2014explaining,kurakin2016adversarial} and Projected Gradient Descent (PGD)~\cite{madry2018towards} attacks. However, we emphasize that image segmentation models and {classification models} have inherent differences: image classification models have a \textit{single prediction for an entire image}, while image segmentation models have \textit{a prediction for each pixel}. The per-pixel prediction 
can provide more information for an attacker to be exploited, e.g., an attacker can {collectively} leverage \textit{all} pixels' predictions to perform the attack.    
{In this paper, 
we would like to design an {optimized} attack framework specially for image segmentation models by leveraging the pixels' predictions.}
Specifically, given a target image segmentation model and a testing image, our goal is to generate an adversarial perturbation to this image under a perturbation budget, such that \textit{as many pixels as possible} 
in the image are wrongly predicted by the target model.

{
To achieve the goal, our attack needs to first uncover the \textit{vulnerable} pixels in the testing image, where a pixel is more vulnerable means it is easier to be misclassified when facing an adversarial perturbation, and then distributes the perturbation budget on more vulnerable pixels in order to misclassify more pixels. 
However, a key challenge is how can we \textit{inherently} characterize the  vulnerability of pixels.
To address it, we are 
inspired by certified robustness/radius~\cite{scheibler2015towards,carlini2017provably,ehlers2017formal,katz2017reluplex,wong2017provable,wong2018scaling,raghunathan2018certified,raghunathan2018semidefinite,cheng2017maximum,fischetti2018deep,bunel2018unified,dvijotham2018dual,gehr2018ai2,mirman2018differentiable,singh2018fast,weng2018towards,zhang2018efficient,lecuyer2018certified,li2018second,cohen2019certified,lee2019stratified,zhai2020macer,levine2020robustness,yang2020randomized,kumar2020curse,mohapatra2020higher,wang2020certifying,jia2020certifiedcommunity,wang2021certified,jia2020certified,jia2022almost,hong2022unicr}.}
Certified radius was originally used by \textit{defenders} to guarantee the robustness of image \textit{classification} models against adversarial perturbations. Given a classification model and a testing image, the certified radius of this image is the maximum (e.g., $l_p$) norm of a worst-case perturbation such that when the worst-case 
perturbation is added to the testing image, the perturbed image can be still accurately predicted by the classification model.
In other words, a testing image with a larger/smaller certified radius indicates  
it is \textit{theoretically} less/more vulnerable to adversarial perturbations. 
{
Though certified radius is derived mainly for doing the good, we realize that {attackers}, on the other hand, can also leverage it to do the bad in the context of image segmentation.
Particularly,  attackers can first attempt to obtain the certified radius 
of pixels, and use them to reversely reveal the inherently vulnerable pixels in the image. Then they can design better attacks using these vulnerable pixels. 
} 

\noindent {\bf Our work:}  
We use the 
property of pixel's certified radius and  
design the first certified radius-guided attack framework to image segmentation models.   
To thoroughly understand the vulnerabilities, we study both white-box and black-box attacks. 
{However, there are several technical challenges: i) How can we obtain pixel-wise certified radius for modern image segmentation models? ii) How can we design an attack framework that is applicable in both white-box and black-box settings? iii) Furthermore, for the black-box attack, can we have guaranteed attack performance to make it more practical?}

{\bf   Obtaining pixel-wise certified radius via randomized smoothing:}
{Directly calculating the pixel-wise certified radius 
for segmentation models 
is challenging. First, there exists no certification method for segmentation models; Second, though we may be able to adjust existing certification methods for image classification models to segmentation models (e.g., \cite{scheibler2015towards,carlini2017provably,ehlers2017formal,katz2017reluplex,wong2017provable,wong2018scaling,raghunathan2018certified,raghunathan2018semidefinite,cheng2017maximum,fischetti2018deep,bunel2018unified,dvijotham2018dual,gehr2018ai2,mirman2018differentiable,singh2018fast,weng2018towards,zhang2018efficient}), the computational overhead can be extremely high. Note that all the existing certification methods for \textit{base} image classification models are not scalable to large models. To address the challenges, we propose to adopt  randomized smoothing~\cite{cohen2019certified,salman2019provably}, which is the state-of-the-art certification method for \textit{smoothed} image classification models, and the only method that is scalable to large models. We generalize randomize smoothing to derive pixel-wise certified radius for image segmentation models (See Theorem~\ref{thm:pwcr}). 
}

{\bf Designing a certified radius-guided attack framework:}
 A larger certified radius of a pixel indicates  
 this pixel is theoretically more robust to  adversarial perturbations. In other words, an attacker needs a larger perturbation to make  the target image segmentation model misclassify this pixel. 
 This observation motivates us to focus more on perturbing pixels with relatively smaller certified radii under a given perturbation budget. 
To achieve the goal, we design a certified radius-guide loss function, where we modify the conventional (pixel-wise) loss function in the target  model by assigning each pixel a weight based on its certified radius. 
Specifically, a pixel with a large/smaller certified radius
will be assigned a smaller/larger weight. 
By doing so, losses for pixels with smaller certified radii will be enlarged, 
and thus more pixels with be wrongly predicted with the given perturbation budget.

\textit{Certified radius-guided white-box attacks:}
In white-box attacks, an attacker 
has full knowledge of the target model, which makes \textit{gradient-based} attacks possible. 
Our aim is then to increase our certified radius-guided loss and generate adversarial perturbations via a gradient-based white-box attack algorithm, e.g., the PGD attack~\cite{arnab2018robustness,madry2018towards}. 
We emphasize that, as our pixel-wise certified radius is plugged 
into the loss function of the target model, any existing gradient-based 
white-box
attack can be used as the base attack in our framework.

\textit{Certified radius-guided black-box attacks:}
In black-box attacks, 
an attacker cannot access the internal configurations of the target  model.
Hence, performing black-box attacks is much more challenging than white-box attacks, 
as no gradient information is available to determine the  perturbation direction. 
Following the existing black-box attacks to image classification models~\cite{ilyas2018black,li2019nattack,lin2020nesterov},
we assume the attacker knows pixels' confidence scores by querying the target segmentation model\footnote{ 
 We note that many real-world systems provide confidence scores, e.g., image classification systems such as Google Cloud Vision~\cite{googlecloud} and Clarifai~\cite{clarifai} return confidence scores when querying the model API.}. 
To design effective black-box attacks, one key step is to 
estimate the gradient of the attack loss with respect to the perturbation.
Generally speaking, there are two types of approaches to estimate the gradients ~\cite{larson2019derivative}—deterministic methods and stochastic methods. The well-known deterministic method is 
the zeroth-order method (i.e., ZOO~\cite{spall2005introduction,chen2017zoo}), and stochastic methods include natural evolutionary strategies (NES)~\cite{ilyas2018black}, SimBA~\cite{guo2019simple} and bandit~\cite{ilyas2018prior,wang2022bandits}. 
When performing real-world black-box attacks, query efficiency and gradient estimation accuracy 
are two critical factors an attacker should consider. 
However, ZOO is very query inefficient, while NES and SimBA are neither query efficient nor have accurate gradient estimation (More detailed analysis are in Section~\ref{sec:bbattack}). 
Bandit methods, when appropriately designed, can achieve the best tradeoff. Moreover, as far as we know, 
bandit is the only framework, under which, we can derive theoretical 
bounds 
when the exact gradient is unknown. Based on these good properties, we thus use bandit as our black-box attack methodology.

Specifically, bandit is  a family of optimization framework with 
partial information (also called bandit feedback)~\cite{hazan2019introduction,hazan2016optimal,bubeck2017kernel,saha2011improved,suggala2021efficient}. 
We notice that black-box attacks with only knowing 
pixels' predictions naturally fit this framework.
With it, we formulate the black-box attacks to image segmentation models as a bandit optimization problem. Our goal is to design a gradient estimator based on the model query and bandit feedback such that the \textit{regret} (i.e., the difference between the expected observed loss through the queries and the optimal loss) is minimized.
We first design a novel gradient estimator, which is query-efficient (2 queries per round) and accurate. 
Then, we design a projected bandit gradient descent (PBGD) attack algorithm based on our gradient estimator.  
As calculating pixels' certified radii only needs to know pixels' predictions, 
our derived pixel-wise certified radius can be seamlessly incorporated into the PBGD attack as well. With it, we further propose a certified
radius-guided PBGD 
attack algorithm to enhance the black-box attack performance.

{{\bf Theoretically guaranteed black-box attack performance:}
We prove that our novel gradient estimator is unbiased and stable. 
We further prove that our 
designed PBGD attack
algorithm achieves a \textit{tight sublinear regret}, which means
the regret tends to be 0 with an optimal rate as the number
of queries increases. 
Finally, our certified radius-guided PBGD attack also obtains
a tight sublinear regret. Detailed theoretical results are seen in Section~\ref{sec:theorems}.}

{\bf Evaluations:} We evaluate our certified radius-guided 
white-box and black-box attacks on modern image segmentation models (i.e., PSPNet~\cite{zhao2017pyramid}, PSANet~\cite{zhao2018psanet}, and HRNet~\cite{sun2019high}) and benchmark datasets (i.e., Pascal VOC~\cite{pascal-voc-2012}, Cityscapes~\cite{cordts2016cityscapes}, and ADE20K~\cite{zhou2019semantic}). In white-box attacks, we choose the state-of-the-art PGD~\cite{arnab2018robustness,madry2018towards} attack as a base attack. The results show that our certified radius-guided PGD attack can substantially outperform the PGD attack.
For instance,  
our attack can have a 
50\% relative gain over the PGD attack in reducing the pixel accuracy on testing images from the datasets.   
In black-box attacks, we show that our gradient estimator achieves 
the best trade-off among query-efficiency, accuracy, and stability, compared with the existing ones. The results also demonstrate the effectiveness of our PBGD attack and that incorporating pixel-wise certified radius can further enhance the attack performance.    

 We also evaluate the state-of-the-art empirical defense FastADT~\cite{wong2020fast} and provable defense SEGCERTIFY~\cite{fischer2021scalable} against our attacks. To avoid the sense of false security~\cite{carlini2019evaluating}, we mainly defend against our white-box CR-PGD attack. Our finding is 
these defenses can mitigate our attack to some extent, but are still not effective enough. 
For example, with an $l_2$ perturbation as 10 on Pascal VOC, the pixel accuracy with SEGCERTIFY and FastADT are 22\% and 41\%, respectively, while the clean pixel accuracy is 95\%. 
Our defense results thus show the necessity of designing stronger defenses in the future.

Our key contributions are summarized as follows:
\begin{itemize}[leftmargin=*]

    \item 
    We propose a certified radius-guide attack framework to study both white-box and black-box attacks to image segmentation 
    models. This is the first work to use certified radius for an attack purpose. Our framework can be seamlessly incorporated into any existing and future loss-based attacks.
    
    \item 
    We are the first to study black-box attacks to image segmentation models 
    based on bandits.  We design a novel gradient estimator 
    for black-box attacks that is query-efficient, 
     and provably 
     unbiased and stable. Our black-box attacks also achieve a \textit{tight} sublinear regret.

    \item Evaluations on modern image segmentation models and datasets 
    validate the effectiveness of our 
    certified radius-guided 
    attacks and their advantages over the compared ones mainly for image classification methods. 
\end{itemize}

\section{Background and Problem Setup}

\subsection{Image Segmentation}

Image segmentation is the task of labeling pixels of an image, where pixels belonging to the same object (e.g., human, tree, car) aim to be classified as the same label.
Formally, given an input image $\bm{x}  = \{x_n\}_{n=1}^N \subset \mathcal{X} $ 
with $N$ pixels 
and groundtruth pixel labels $y=\{y_n\}_{n=1}^N$, where each pixel $x_n$ has a label $y_n$ 
from a label set 
$\mathcal{Y}$,
an image segmentation model learns a mapping $F_\theta: \mathcal{X} \rightarrow \mathbb{P}^{N\times |\mathcal{Y}|}$, parameterized by $\theta$, 
where each row in $\mathbb{P}$ is the set of probability distributions over $\mathcal{Y}$, i.e., the sum of each row in $\mathbb{P}$  equals to 1.  
Different image segmentation methods design different loss functions to learn $F_\theta$. 
Suppose we have a set of training images $\mathbb{D}_{tr} = \{(x,y)\} $, a common way to learn $F_\theta$
is by minimizing 
a  pixel-wise 
loss function $L$ 
defined on the training set as follows:
\begin{equation}
\small
    {\min _\theta }  \sum_{\small (x,y) \in \mathbb{D}_{tr}} L(F_\theta(x), y) = - \sum_{(x,y) \in \mathbb{D}_{tr}} \sum_{n=1}^N 1_{y_n} \odot \log F_\theta(x)_{n},
    \label{equ:seg}
\end{equation}
where we use the cross entropy as the loss function. 
$1_{y_n}$ is an $|\mathcal{Y}|$-dimensional indicator vector whose $y_n$-th entry is 1, and 0 otherwise. $\odot$ is the element-wise product. 
After learning $F_\theta$, giving a testing image $\underline{x}$, each pixel $\underline{x}_n$ is predicted a label $\hat{y}_n  = \arg\max_j F_\theta(\underline{x})_{n,j}$.

\subsection{Certified Radius}
\label{bg_cr}

We introduce the certified radius achieved via state-of-the-art randomized smoothing   methods~\cite{cohen2019certified,salman2019provably}. Certified radius was originally 
derived to measure the certified robustness of an image classifier against adversarial perturbations. Generally speaking, for a testing image, if it has a larger certified radius under the classifier, then it is provably more robust to 
adversarial 
perturbations. 

Suppose we are given a testing image $x$ with a label $y$, and a 
(base) soft classifier $f$, which maps ${x}$ to confidence scores. 
Randomized smoothing first builds a smoothed soft classifier $g$ from the base 
$f$ and then calculates the certified radius for $x$ on the smoothed 
soft classifier 
$g$.  Specifically, given a noise distribution $\mathcal{D}$, $g$ is defined as: 
\begin{align}
g(x) = \mathbb{E}_{\beta \sim \mathcal{D}} [f(x + \beta)],
\end{align}
where $g(x)_c$ is the probability of the noisy $x + \beta$ predicts to be the label $c$, with the noise $\beta$ sampled from $\mathcal{D}$.

Assuming that $g(x)$ assigns to $x$ the true label $y$ with probability $p_A = g(x)_y$, and assigns to $x$ the ``runner-up" label $y'$ with probability $p_B = \max_{y'\neq y} g(x)_{y'}$. 
Suppose $\mathcal{D}$ is a {Gaussian}
distribution with mean 0 and variance $\sigma^2$. 
Then, authors in \cite{cohen2019certified,salman2019provably} derive the following \textit{tight} certified radius of the smoothed soft classifier $g$ for the testing image $x$ against an $l_2$ perturbation:
\begin{align}
\label{cr_raw}
cr(x) = \frac{\sigma}{2} [\Phi^{-1}(p_A) - \Phi^{-1}(p_B)],
\end{align}
where $\Phi^{-1}$ is the inverse of the standard Gaussian cumulative distribution function (CDF). 
That is, $g$ provably has the correct prediction $y$ for $x$ over all 
adversarial 
perturbations $\delta$, i.e., $\arg\max_{c} g(x + \delta)_c = y$, when  $||\delta||_2 \leq cr(x)$.
Note that calculating the exact probabilities $p_A$ and $p_B$ is challenging. Authors in~\cite{cohen2019certified, salman2019provably} use the Monte Carlo sampling algorithm to estimate a lower bound $\underline{p_A}$ of $p_A$ and an upper bound $\overline{p_B}$ of $p_B$ 
with arbitrarily high probability over the samples. 
They further set $\overline{p_B} = 1 - \underline{p_A}$ for simplicity. Then, $cr(x) = \sigma \Phi^{-1}(\underline{p_A})$.

\subsection{Bandit}
\label{bg:bandits}

In the continuous optimization setting, 
the bandit method optimizes a black-box function over an infinite domain with feedback. The black-box function means the specific form of the function is not revealed but its function value can be observed. Due to this property, bandit can be a natural tool to design black-box algorithms.
Next, we describe the three components:
action, (bandit) 
feedback, and goal,
in a bandit optimization problem.

\begin{itemize}[leftmargin=*]
\item \textbf{Action}: A learner plans to maximize a time-varying reward function $r_t(\cdot)$ with $T$ rounds' evaluations. In each round $t$, the learner selects an \textit{action} $x_t$ from a action space, $\mathcal{S}$, 
which is often defined as  
a convex set.  

\item \textbf{Feedback}: When the learner performs an action $x_t$ and submits the decision to the environment in round $t$, he will observe a reward $r_t(x_t)$ at $x_t$. 
As the observed information about the reward function $r_t$ is partial (i.e., only the function value 
instead of the function itself) and incomplete (the function value may be noisy), the observed information is often called \textit{bandit feedback}.

\item \textbf{Goal}: 
As no full information in advance, the learner uses \textit{regret} to measure the performance of his policy $\mathcal{P}$. The goal of the learner is to design a policy $\mathcal{P}$ to minimize the {regret}, which is defined as the gap between the expected cumulative rewards achieved by the selected actions and the maximum cumulative rewards achieved by the optimal action in hindsight, i.e.,
\begin{equation}
R_\mathcal{P}(T) = \mathbb{E}[\sum_{t=1}^Tr_t(x_t) - \max_{x\in\mathcal{S}}r_t(x)],
\end{equation}
where the expectation is taken over the randomness in 
the policy 
$\mathcal{P}$. 
When the policy $\mathcal{P}$ achieves a \textit{sublinear} regret (i.e., $R_{\mathcal{P}}(T) = o(T)$), we say it is \textit{asymptotically optimal} as the incurred regret disappears when $T$ is large enough, i.e., $\lim_{T\to\infty}R_{\mathcal{P}}(T)/T = 0$. 

\end{itemize}

\subsection{Problem Setup}

Suppose we have a target image segmentation model $F_\theta$, a testing image $x = \{x_n\}_{n=1}^N$ with true pixel labels $y = \{y_n\}_{n=1}^N$.
We consider that an attacker can add an adversarial perturbation $\delta = \{\delta_n\}_{n=1}^N$ with a bounded $l_p$-norm $\epsilon$ to
$x$, i.e., $\delta \in \Delta = \{\delta: ||\delta||_p \le \epsilon\}$. The attacker's goal is to maximally mislead 
$F_\theta$ on the perturbed testing image $x +\delta$, i.e., making as many pixels as possible wrongly predicted by $F_\theta$. 
Formally, 
\begin{align}
\label{attack_prob_raw}
& \max_{\delta} \sum_{n=1}^N 1[\arg \max_{c \in \mathcal{Y}} F_\theta(x+\delta)_{n,c} \neq y_n], \, \textrm{s.t., }
\delta \in \Delta.
\end{align}
The above problem is challenging to solve in that the indicator function $1[\cdot]$ is hard to optimize. In practice, the attacker will solve an alternative optimization problem that maximizes an \textit{attack loss} to find the perturbation $\delta$. 
S/He can use any attack loss in the existing works~\cite{xie2017adversarial,fischer2017adversarial,cisse2017houdini,hendrik2017universal,arnab2018robustness}.   
For instance, s/he can simply maximize the loss function $L$:  
{
\begin{align}
\max_{\delta} L(F_\theta(x+\delta),y) = \sum_{n=1}^N L(F_\theta(x+\delta)_n ,y_n), \textrm{s.t., } \delta \in \Delta. \label{attack_loss} 
\end{align}
}
In this paper, we consider both \textit{white-box attacks} and \textit{black-box attacks} to image segmentation models.  

\begin{itemize}[leftmargin=*]

\item {\bf White-box attacks:} An attacker knows the full knowledge about 
$F_\theta$, e.g., model parameters $\theta$, architecture. 

\item {\bf Black-box attacks:} An attacker has no knowledge about the internal configurations of 
$F_\theta$, and s/he only knows the confidence scores $F_\theta(x_q)$ via querying $F_\theta$ with an input $x_q$, 
following the existing black-box attacks to image classification models~\cite{ilyas2018black,li2019nattack,lin2020nesterov}.
\end{itemize}

\section{Certified Radius Guided White-Box Attacks to Image Segmentation 
}
\label{sec:cr_white}

\subsection{Overview}

Existing white-box attacks to image segmentation models 
majorly 
follow the idea of attacking image \textit{classification} models~\cite{kurakin2016adversarial,madry2018towards}. 
However, these attack methods 
are suboptimal.
This is because 
the per-pixel prediction in  segmentation models can provide much richer information for an attacker to be exploited, while classification models only have a {single} prediction for an entire image. 
We propose to exploit the \textit{unique} {pixel-wise certified radius}  information from pixels' predictions. 
We first observe an inverse relationship between a pixel's certified radius and the assigned perturbation to this pixel (See Figure~\ref{fig:observation}) and  
derive pixel-wise certified radius via
randomized smoothing~\cite{cohen2019certified,salman2019provably}. 
Then, we assign each pixel a weight 
based on its certified radius, and 
design a novel certified radius-guided attack loss, where we incorporate the pixel weights into the conventional attack loss.
Finally, we design our certified radius-guided white-box attack framework to image segmentation models based on our new attack loss.

\subsection{Attack Design}

Our attack is inspired by certified radius. 
We first define our \textit{pixel-wise} certified radius that is customized to image segmentation models. 

\begin{definition}[Pixel-wise certified radius]
\label{def:pwcr}
Given a {base} (or smoothed) image segmentation model $F_\theta$ (or $G_\theta$) and a testing image $x$ with pixel labels ${y}$. We define  certified radius of a pixel $x_n$, i.e., $cr(x_n)$, as the maximal value, such that  $F_\theta$ (or $G_\theta$) correctly predicts the pixel $x_n$ against any adversarial perturbation $\delta$ when its (e.g., $l_p$) norm is not larger than this value. Formally, 
\begin{align}
\label{eqn:pwcrdef}
& cr(x_n) = \max r,  \nonumber \\
& \textrm{s.t. } \arg \max_{c \in \mathcal{Y}} G_\theta(x+\delta)_{n,c} = y_n, \forall ||\delta||_p \leq r.
\end{align}
\end{definition}

From Definition~\ref{def:pwcr}, 
the certified radius of a pixel describes the extent to which the image segmentation model can provably has the correct prediction for this pixel against the worst-case adversarial perturbation. Based on this, we have the following 
observation that reveals the \textit{inverse} relationship between the pixel-wise certified radius and the 
perturbation when designing an effective attack.

\vspace{+0.5mm}
\noindent {\bf Observation 1: A pixel with a larger (smaller) certified radius should be disrupted with a smaller (larger) perturbation on the entire image.}
If a pixel has a larger certified radius, it means this pixel is more robust to adversarial perturbations. To wrongly predict this pixel, an attacker should allocate a larger perturbation.  In contrast, if a pixel has a smaller certified radius, this pixel is more vulnerable to adversarial perturbations. 
To wrongly predict this pixel, an attacker just needs to allocate a smaller perturbation. 
Thus, to design more effective attacks with limited perturbation budget, an attack should avoid 
disrupting pixels with 
relatively 
larger certified radii, but focus on 
pixels with relatively smaller certified radii.

With the above observation, our attack needs to solve three closely related 
problems: i) How to obtain the pixel-wise certified radius? ii) How to allocate the perturbation budget in order to perturb the pixels with smaller certified radii? and iii) How to generate adversarial perturbations to better attack image segmentation models? To address  i), we adopt the efficient randomized smoothing method~\cite{cohen2019certified,salman2019provably}. 
To address  ii), we design a certified-radius guided attack loss, by maximizing which an attacker will put more effort on perturbing pixels with smaller certified radii.   
To address iii), we design a certified radius-guide attack framework, where any existing loss-based attack method 
can be adopted as the base attack.

\begin{figure*}[!t]
\centering
\subfigure[Raw image]{\includegraphics[width=0.16\textwidth]{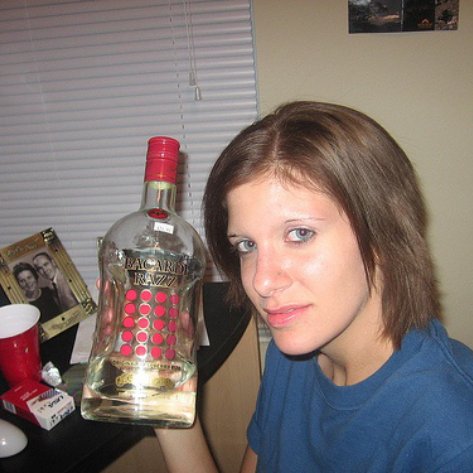}}
\subfigure[True seg. labels]{\includegraphics[width=0.16\textwidth]{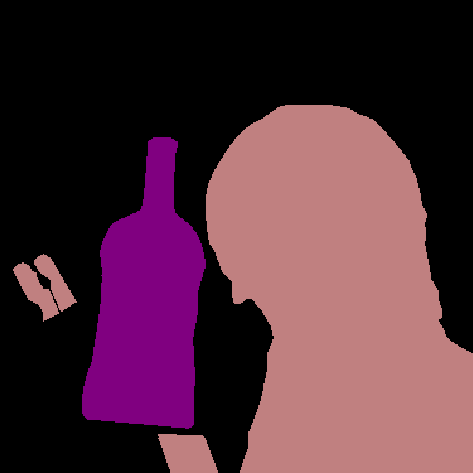}}
\subfigure[PGD pert. image]{\includegraphics[width=0.16\textwidth]{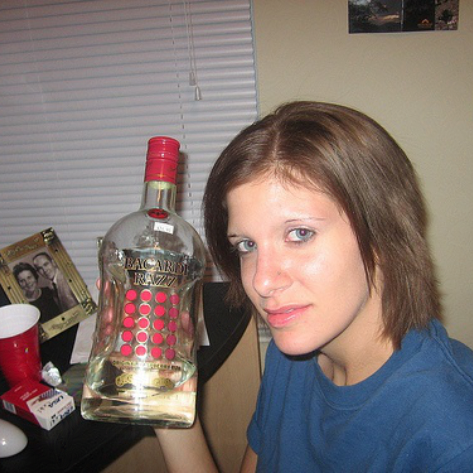}}
\subfigure[PGD perturbation]{\includegraphics[width=0.16\textwidth]{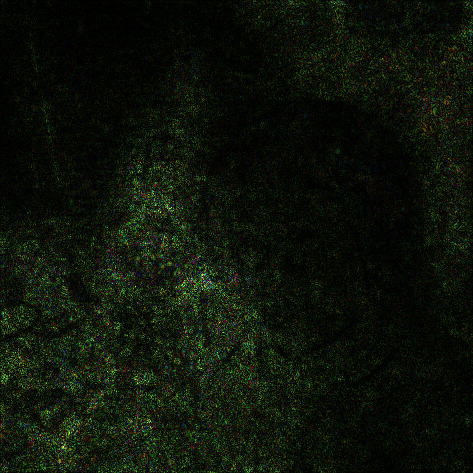}}
\subfigure[{\footnotesize CR-PGD pert. img.}]{\includegraphics[width=0.16\textwidth]{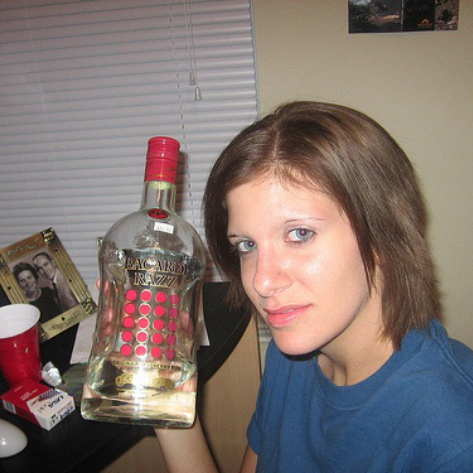}}
\subfigure[CR-PGD perturba.]{\includegraphics[width=0.16\textwidth]{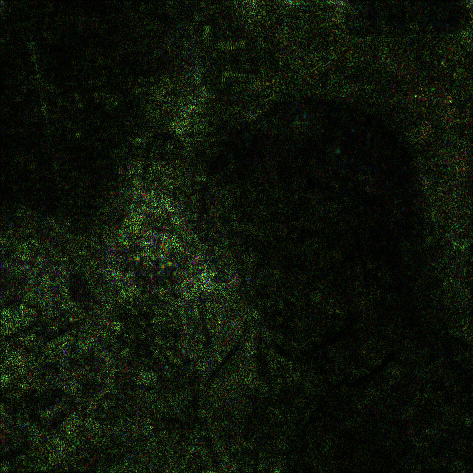}}
\subfigure[PGD predictions]{\includegraphics[width=0.16\textwidth]{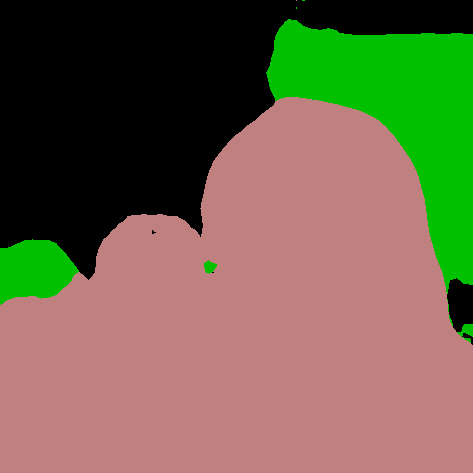}}
\subfigure[CR-PGD preds.]{\includegraphics[width=0.16\textwidth]{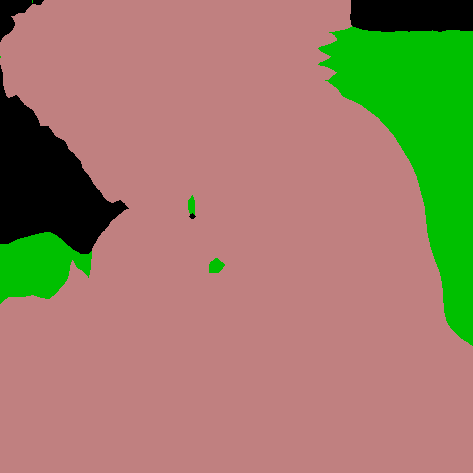}}
\subfigure[Pixel-wise CR]{\includegraphics[width=0.185\textwidth]{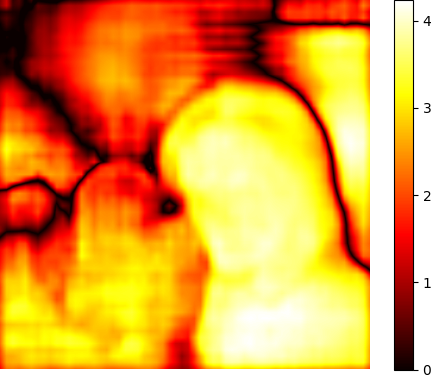}}
\subfigure[Pixel weights]{\includegraphics[width=0.185\textwidth]{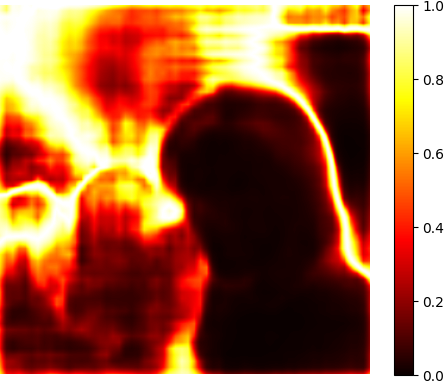}}
\subfigure[Certified radius vs. weight]{\includegraphics[width=0.26\textwidth]{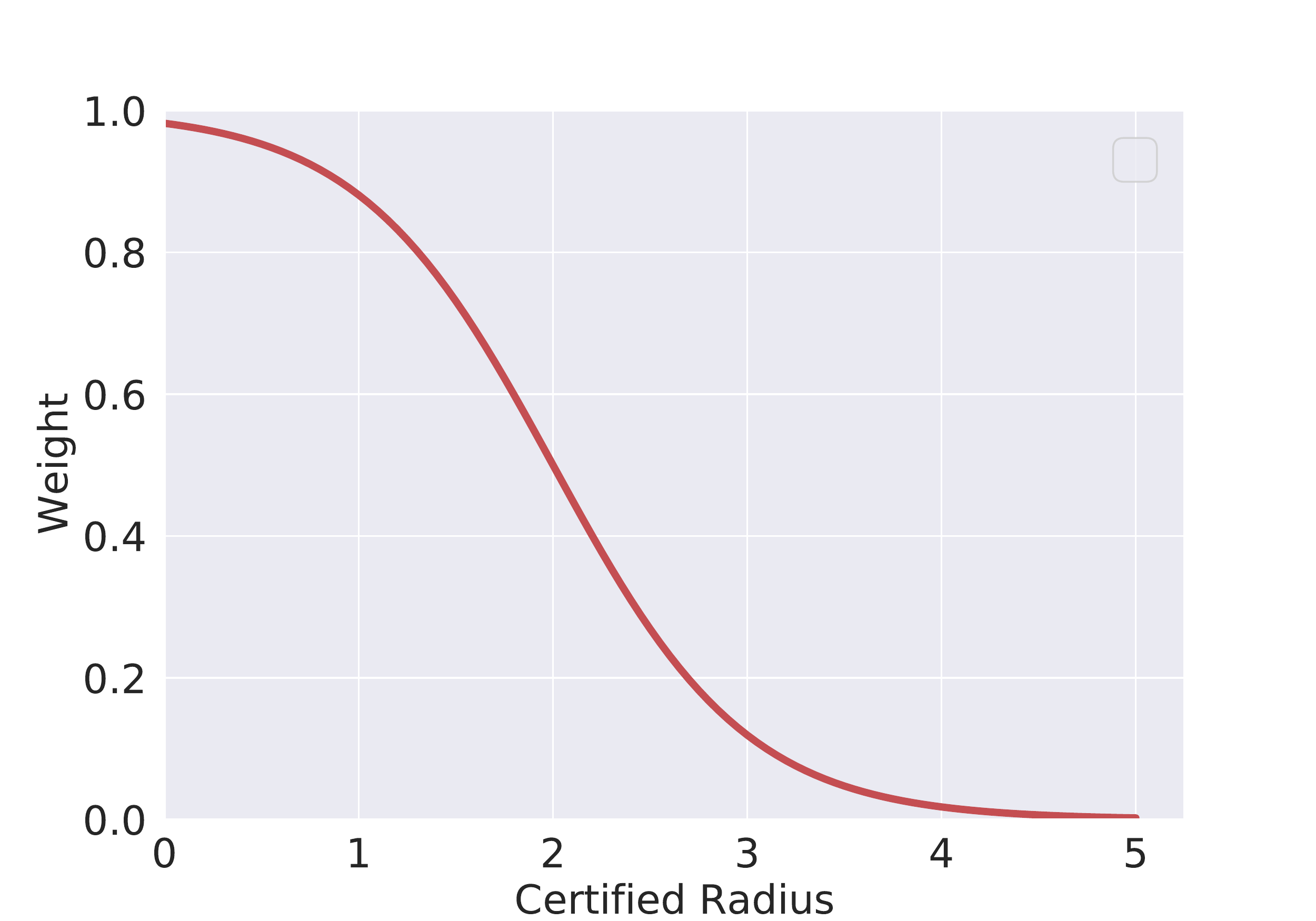}}
\vspace{-2mm}
\caption{(a)-(j) Illustration of our certified radius guided $l_2$ PGD attack on a random image in Pascal VOC. 
We observe that using pixels' certified radii, our CR-PGD attack focuses more on perturbing pixels with relatively smaller certified radii, while PGD does not. Thus, our CR-PGD can misclassify more pixels than PGD. {    
Regarding the inverse relationship in our observations: In (i), the left-bottom pixels have relatively larger certified radii, while the  top-left/middle pixels have relatively smaller certified radii. Then, CR-PGD in (f) assigns less perturbations on the left-bottom region, but more perturbations on the top-left/middle region of the image. In contrast, PGD in (d) assigns more perturbations on the left-bottom region, but less on the top-lef/middle region.  As a result, CR-PGD in (h) causes most pixels at the top-lef/middle to be misclassified, but  PGD in (g) does not.}
(k) Relationship between pixel-wise certified radius and pixel weight in Equation~\ref{eqn:weight}, where $a =2$ and $b =-4$. 
}
\label{fig:observation}
\vspace{-4mm}
\end{figure*}

\subsubsection{Deriving the pixel-wise certified radius via randomized smoothing} 
{Directly calculating the pixel-wise certified radius for segmentation models faces two challenges: no certification method exists; and adjusting existing certification methods for (base) image classification models to segmentation models has extremely high computational overheads. For instance, one can use the approximate local Lipschitz proposed in \cite{wengevaluating}. 
However, as our results shown in Section~\ref{sec:discussion}, it is infeasible to apply \cite{wengevaluating} to calculate certified radius for segmentation models.

To address these challenges, we adapt the state-of-the-art randomized smoothing-based efficient certification method~\cite{cohen2019certified,salman2019provably} for \textit{smoothed} image classification models.}
Specifically, we first  build a smoothed image segmentation model for the target segmentation model and then derive the pixel-wise certified radius on the smoothed model via randomized smoothing as below:

\begin{restatable}[]{theorem}{pwcr}
\label{thm:pwcr}
Given an image segmentation model $F_\theta$ and a testing image $x$, we build a smoothed  segmentation model as $G_\theta(x) = {\mathbb{E}_{\beta  \sim \mathcal{N}(0,{\sigma^2}I)}}F_\theta (x + \beta )$. Then for each pixel $x_n$, its certified radius for $l_2$ perturbation is: 
\begin{align}
    \label{equ:pwcr}
    cr(x_n) = \sigma \Phi^{-1}(\max_c G_\theta(x)_{n,c}).
\end{align}
\end{restatable}
\begin{proof}
See Appendix~\ref{supp:pwcr}.
\vspace{-1mm}
\end{proof}

\noindent {\bf Remark 1:} Theorem 1 generalizes randomized smoothing
to segmentation models.
In practice, however, obtaining the exact value of $ \max_c G_\theta(x)_{n,c}$ is 
{computationally} challenging due to the random noise $\beta$. Here, we use the Monte Carlo sampling algorithm as ~\cite{cohen2019certified,salman2019provably} to estimate its lower bound. Specifically, we first sample a set of, say $M$, noises $\{\beta_1, \beta_2, \cdots, \beta_M \}$ from the {Gaussian} 
distribution $\mathcal{N}(0,\sigma^2 I)$ and then use the empirical mean $\hat{G}_\theta(x) = \frac{1}{M}\sum_{i=1}^M F_\theta(x + \beta_j)$ to estimate the lower bound as $ \max_c \hat{G}_\theta(x)_{n,c}$.  
 In contrast to the tight certified radius obtained in Equation~\ref{cr_raw} for image classification models, the pixel-wise certified radius  in Equation~\ref{equ:pwcr} cannot guarantee to be tight for image segmentation models. {Note that our goal is not to design a  better defense that requires larger certified radii. 
 Instead, we leverage the order of pixels' certified radii
 to identify pixels' relative robustness/vulnerability against adversarial perturbations, whose information is used to design more effective attacks.
 }

\vspace{+0.5mm}
\noindent {{\bf Remark 2:} The pixel-wise certified radius derived for $l_2$ perturbations in Equation~\ref{equ:pwcr} suffices to be used for other common norm-based, e.g., $l_1$ and $l_\infty$, perturbations. This is because a pixel with a larger $l_2$ certified radius also has a larger $l_1$ and $l_\infty$ certified radius, thus more robust against $l_1$ and $1_\infty$ perturbations\footnote{{For any $N$-dimensional vector $x$, its $l_2$, $l_1$, and $l_\infty$ norms have the relation $||x||_1 \leq \sqrt{N} ||x||_2$ and $||x||_\infty  \leq ||x||_2$. Thus, obtaining an $l_2$ certified radius implies an upper bounded $l_1$ or $l_\infty$ certified radius.}}}.

\vspace{+0.5mm}
\noindent {{\bf Remark 3:} The rationale of using randomized smoothing for certification is that, when an image A is intrinsically more robust than an image B on the base model against adversarial perturbation, then adding a small noise to these two images, the noisy A is still more robust than the noisy B on the smoothed model against adversarial perturbation. 
Moreover, the smoothed model has close certified radius under a small noise (i.e., a small $\sigma$) as the base model (on which certified radius is challenging to compute). Hence, a pixel's certified radius on the smoothed model indicates its robustness on the base model as well.
}

\subsubsection{Designing a certified radius-guided attack loss} 
After obtaining the certified radii of all pixels, a naive solution is that the attacker sorts pixels' certified radii in an ascending order, and then perturbs the pixels one-by-one from the beginning until reaching the perturbation budget. However, this solution is both computationally intensive---as it needs to solve an optimization problem for each pixel; and suboptimal---as all pixels collectively make predictions for each pixel and perturbing a single pixel could affect the predictions of all the other pixels.   
 
 Here, we 
 design a certified radius-guided attack loss that assists to \textit{automatically} find the ``ideal" pixels to be perturbed.
 We observe that the attack loss in Equation~\ref{attack_loss} 
 is defined per pixel. 
 Then, we propose to modify the attack loss in Equation~\ref{attack_loss} by associating each pixel with a weight and multiplying the weight with the corresponding pixel loss, where the pixel weight is correlated with 
 the pixel's certified radius. Formally, our certified radius-guided attack loss is defined as follows: 
\begin{equation}
\small
\label{eq:cr_loss}
    L_{cr}(F_\theta(x),y) = \frac{1}{N} \sum_{{n}=1}^N w(x_n) \cdot L(F_\theta(x)_n, y_n),
\end{equation}
where $w(x_n)$ is the weight of the pixel $x_n$. Note that when setting all pixels with a same weight, our certified radius-guided loss reduces to the conventional loss.  

Next, we show the \textit{inverse} relationship between the
pixel-wise certified radius and pixel weight; and define an example form of the pixel weight used in our paper.

\noindent {\bf Observation 2: A pixel with a larger (smaller) certified radius should be
assigned a smaller (larger) weight in the certified radius-guided loss.} As shown in {\bf Observation 1}, we should perturb more pixels with smaller certified radii, as they 
are more vulnerable. 
That is, we should put more weights on pixels with smaller certified radii to enlarge these pixels' losses---making these pixels easier to be misclassified 
with perturbations. By doing so, the image segmentation model will wrongly predict more pixels with a given perturbation budget. 
In contrast, we should put smaller weights on pixels with larger certified radii, in order to save the usage of the 
budget. 

There are many different ways to assign the pixel weight such that $w(x_n) \sim \frac{1}{cr(w_n)}$ based on {\bf Observation 2}. 
In this paper, we propose to use the following form: 
\begin{equation}
\label{eqn:weight}
w(x_n) = \frac{1}{1+\exp(a \cdot cr(x_n) + b)},
\end{equation}
where $a$ and $b$ are two scalar hyperparameters\footnote{We leave designing other forms of pixel weights as future work.}. 
Figure~\ref{fig:observation}(k) illustrates the relationship between the pixel-wise certified radius and pixel weight defined in Equation~\ref{eqn:weight}, where 
$a =2$ and $b =-4$. We can observe that the pixel weight is \textit{exponentially} decreased as the certified radius increases. Such a property can ensure that most of the pixels with smaller radii will be perturbed (See Figure~\ref{fig:distpwcr}) when performing the attack.    

We emphasize that our weight design in Equation~\ref{eqn:weight} is not optimal; actually due to the non-linearity of neural network models, it is challenging to derive optimal weights. 
Moreover, as pointed out \cite{madry2018towards}, introducing randomness can make  the attack more effective. Our smoothed segmentation model based weight design introduces randomness into our attack loss, which makes our attack harder to defend against, as shown in our results in Section~\ref{sec:eval_defense}. 
 
\vspace{-2mm}
\subsubsection{Certified radius-guided white-box attacks to generate adversarial perturbations} 
We use our designed certified radius-guided attack loss to generate adversarial perturbations to image segmentation models.
Note that we can choose any existing white-box attack 
as the base attack.
In particular, given the attack loss from any existing white-box attack, we only need to modify our attack loss by multiplying the pixel weights with the corresponding attack loss.
For instance, when we use the PGD attack~\cite{madry2018towards} as the base attack method, we have our certified radius-guided PGD (CR-PGD) attack that iteratively generates adversarial perturbations as follows:  
\begin{equation}
\label{eqn:crpgd}
\delta = \textrm{Proj}_{\Delta} ( \delta + \alpha \cdot \nabla_\delta L_{cr}(F_\theta(x+\delta),y)),
\end{equation}
where $\alpha$ is the learning rate in PGD, $\Delta = \{ \delta: ||\delta||_p \le \epsilon \}$ is the allowable perturbation set, and $\textrm{Proj}_\Delta$ projects the adversarial perturbation to the allowable set $\Delta$. The final adversarial perturbation is used to perform the attack.

{Figure~\ref{fig:observation} illustrates our certified radius-guided  attack framework to image segmentation models, where we use the $l_2$ PGD attack as the base attack method.} Algorithm~\ref{alg:cr_PGD} in Appendix details 
our CR-PGD attack. 
Comparing with PGD, the computational overhead of our CR-PGD is calculating the pixel-wise certified radius with a small set of $M$ sampled noises (every \textrm{INT} iterations and \textrm{INT} is a predefined parameter in Algorithm 1 in Appendix), {which only involves making predictions on $M$
noisy samples and is very efficient}. Note that the predictions are independent and can be also parallelized. 
Thus, PGD and CR-PGD have the comparable computational complexity.   
We also show the detailed time comparison results in Section~\ref{sec:eval_wb}.

\section{Certified Radius Guided Black-Box Attacks to Image Segmentation  via Bandit}
\label{sec:cr_black}

\subsection{Motivation of Using Bandit}
In black-box attacks, an attacker can only query the target 
image segmentation 
model to obtain pixels' confidence scores. 
A key challenge in this setting 
is that no gradient information is available, and thus an attacker cannot conduct the  (projected) gradient descent like white-box attacks \cite{goodfellow2014explaining,madry2018towards,carlini2017towards} to determine the perturbation directions. 
A common way to address this is converting black-box attacks 
with partial feedback (i.e., confidence scores) to be the 
gradient estimation problem, solving which the standard (projected) gradient descent based attack can then be applied\footnote{{Surrogate models (e.g., \cite{papernot2017practical,liudelving}) are another common way to transfer a black-box attack into a white-box setting, but their performances are worse than gradient estimation methods~\cite{guo2019simple,ilyas2018prior}}.}. 
Existing gradient estimate methods~\cite{larson2019derivative} can be classified as deterministic methods (e.g., zero-order optimization,  ZOO~\cite{spall2005introduction,chen2017zoo}) and stochastic methods (e.g., NES~\cite{ilyas2018black}, SimBA~\cite{guo2019simple}, and bandit \cite{ilyas2018prior,wang2022bandits}). To perform real-world black-box attacks, query efficiency and gradient estimation accuracy (in terms of unbiasedness and stability) are two critical factors.
ZOO (See Equation~\ref{eqn:zoo}) 
is accurate but query inefficient, while 
NES and SimBA are neither accurate nor query efficient. 
In contrast, bandit can achieve the best tradeoff~
\cite{hazan2019introduction,flaxman2004online} when the gradient 
estimator is appropriately designed.

\subsection{Overview}
Inspired by the good properties of bandits, 
we formulate the black-box attacks to image segmentation models as a bandit optimization problem. 
However, designing an effective bandit algorithm 
for solving practicable black-box attacks 
faces several technical challenges: 1) It should be query efficient;
2) it should estimate the gradient accurately; and 3) the most challenging, it should guarantee the attack performance to approach the optimal as the query number increases. We aim to address all these challenges. Specifically, we first design a novel gradient estimator with bandit feedback and show it only needs 2 queries per round.
We also prove it is unbiased and stable. 
Based on our estimator, we then propose projected bandit gradient descent (PBGD) to construct the perturbation. We observe that the pixel-wise certified radius can be also derived in the considered black-box setting and be seamlessly incorporated into our PBGD based attack algorithm. Based on this observation, we further design a certified radius-guided PBGD (CR-PBGD) attack to enhance the black-box attack performance. Finally, we prove that our bandit-based attack algorithm achieves a \textit{tight sublinear} regret, meaning our attack is asymptotically optimal with an optimal rate. \textit{The theoretical contributions of our regret bound can be also seen in Appendix~\ref{sec:regret}.}

\subsection{Attack Design}
\label{sec:bbattack}

\subsubsection {Formulating black-box attacks to image segmentation models as a bandit optimization problem}
In the context of leveraging bandits to design our black-box attack, 
we need to define the attacker's \textit{action}, 
\textit{(bandit) feeback}, and \textit{goal}.
Suppose there are $T$ rounds, which means an attacker will attack the target image segmentation model up to $T$ rounds. In each round $t\in\{1,2,\cdots,T\}$: 
\begin{itemize}[leftmargin=*]
    
    \item {\bf Action:} The attacker determines a \underline{perturbation} $\delta^{(t)}\in \Delta$ via a designed {attack algorithm} $\mathcal{A}$.  
    \item {\bf Feedback:} By querying the target model with the perturbed image $x+\delta^{(t)}$, attacker observes the corresponding \underline{attack loss}\footnote{For notation simplicity, we will use $L(\delta)$ to indicate the attack loss $L(F_\theta(x + \delta),y).$} $L(\delta^{(t)})$ at the selected perturbation $\delta^{(t)}$. 
    
    \item {\bf Goal:} As the attacker only has the bandit feedback $L(\delta^{(t)})$ at the selected perturbation $\delta^{(t)}$, the attack algorithm $\mathcal{A}$ will incur a {regret}, which is defined as the expected difference between the attack loss at 
    $\delta^{(t)}$ brought by 
    $\mathcal{A}$ and the maximum attack loss 
    in hindsight. Let $R_{\mathcal{A}}(T)$ be the cumulative regret after $T$ rounds, then the regret is calculated as 

\begin{equation}\label{eq:regret}
\small
R_{\mathcal{A}}(T) =  \sum_{t=1}^T \mathbb{E}[L(\delta^{(t)})] - \max_{\delta \in \Delta} \sum_{t=1}^T L(\delta),
\end{equation}
where the expectation is taken over the randomness from 
$\mathcal{A}$. The attacker's goal to \underline{minimize the regret}.  

\end{itemize}

\noindent Now our problem becomes:
\textit{how does an attacker design an attack algorithm that utilizes the bandit feedback via querying the target model, and determine the adversarial perturbation to achieve a sublinear regret?}  
There are several problems to be solved: (i) How to
accurately estimate the gradient (i.e., unbiased and stable) in order to determine the perturbation? 
(ii) How to make black-box attacks query-efficient?
(iii) How can we achieve a sublinear regret bound? 
We propose novel gradient estimation methods to solve them.

\subsubsection{Two-point gradient estimation with bandit feedback}
 We first introduce two existing gradient estimators, i.e., the deterministic method-based ZOO~\cite{spall2005introduction,chen2017zoo} and stochastic bandit-based one-point gradient estimator (OPGE)~\cite{granichin1989stochastic,spall1997one,flaxman2004online}, and show their limitations. We do not show the details of NES~\cite{ilyas2018black} and SimBA~\cite{guo2019simple} because they are 
neither accurate nor query efficient. 
Then, we propose our two-point gradient estimator. 

\vspace{+0.5mm}
\noindent {\bf ZOO.} It uses the finite difference method~\cite{lax2014calculus} and determinately estimates the gradient  vector element-by-element. Specifically, given a perturbation $\delta$, ZOO estimates the gradient of the $i$-th element, i.e., $\nabla L(\delta)_i$, as 
\begin{equation}
\label{eqn:zoo} 
\hat{g}^{ZOO}_i = \nabla L(\delta)_i \approx \frac{L(\delta + \gamma e_i) - L(\delta - \gamma e_i)}{2 \gamma},     
\end{equation}
where $\gamma$ is a small positive number and $e_i$ is a standard basis vector with the $i$-th element be 1 and 0 otherwise. 

ZOO is an unbiased gradient estimator. 
However, it faces two challenges: (i) 
It requires a sufficiently large number of queries to perform the 
gradient  estimation, which is often impracticable due to 
a limited query budget. Specifically, ZOO  depends on two losses $L(\delta + \gamma e_i)$ and $L(\delta - \gamma e_i)$, which is realized by querying the target model twice using the two points $\delta + \gamma e_i$ and $\delta - \gamma e_i$, and estimates the gradient of a single element $i$ per round. 
To estimate the gradient vector of $N$ elements, ZOO needs $2N$ queries per round. As the number of pixels $N$ in an image is often large, the total number of queries often exceeds the attacker's query budget. 
(ii) It requires the loss function $L$ to be differentiable everywhere, while some loss functions, e.g., hinge loss, is nondifferentiable.

\vspace{+0.5mm}
\noindent {\bf One-point gradient estimator (OPGE).} It estimates the whole gradient in a random fashion. 
It first defines a smoothed loss $\hat{L}(\delta)$ of the loss $L(\delta)$ at a given perturbation $\delta$ as follows: 
\begin{equation}
\label{eqn:smoothedloss}
\hat{L}(\delta) = \mathbb{E}_{\bm{v}\in\mathcal{B}_p}[L(\delta +\gamma\bm{v})],
\end{equation}
where $\mathcal{B}_p$ is a unit $l_p$-ball, i.e., $\mathcal{B}_p = \{ u: ||u||_p \leq 1 \}$, and  $v$ is a random vector sampling from $\mathcal{B}_p$. 
Then, by observing that $\hat{L}(\delta)\approx L(\delta)$ when $\gamma$ is sufficiently small, OPGE uses the gradient of $\hat{L}(\delta)$ to approximate $L(\delta)$. 
Specifically, the estimated gradient $\hat{L}(\delta)$ has the following form~\cite{granichin1989stochastic,spall1997one,flaxman2004online}:  
\begin{equation}
\label{eq:est1}
\small
 \hat{g}^{OPGE} = \nabla\hat{L}(\delta) = \mathbb{E}_{\bm{u}\in\mathcal{S}_p}[\frac{N}{\gamma}L(\delta+\gamma\bm{u})\bm{u}],
 \end{equation}
where $\mathcal{S}_p$ is a unit $l_p$-sphere, i.e., $\mathcal{S}_p = \{u: ||u||_p = 1\}$. To calculate the expectation in Equation~\ref{eq:est1}, OPGE simply samples a single $\hat{u}$ from $\mathcal{S}_p$ and estimates the expectation as $\frac{N}{\gamma}L(\delta+\gamma\hat{u})\hat{u}$. As OPGE only uses a point $\delta+\gamma\hat{u}$ to obtain the  feedback $L(\delta+\gamma\hat{u})$, it is called one-point gradient estimator. 

OPGE is extremely query-efficient as the whole gradient is estimated with only one query (based on one point $\delta +\gamma\hat{u}$). The gradient estimator $\hat{g}^{OPGE}$ is differentiable everywhere even when the loss function $L$ is non-differentiable. It is also an unbiased gradient estimation method, same as ZOO.  However, OPGE has two key disadvantages: (i) Only when $\gamma$ is extremely small can $\hat{L}(\delta)$ be close to $L(\delta)$. In this case, the coefficient $N/\gamma$ would be very large. Such a phenomenon will easily cause the updated gradient out of the feasible image space $[0,1]^N$. (ii) The estimated gradient norm is unbounded as it depends on $\gamma$, which will make the estimated gradient rather unstable 
when only a single $\hat{u}$ is sampled and used.

\vspace{+0.5mm}
\noindent {\bf The proposed two-point gradient estimator (TPGE).} To address the  challenges in OPGE, we propose a \textit{two-point gradient estimator (TPGD)}. 
TPGD combines the idea of ZOO and OPGE: On one hand, similar to OPGE, we build a smoothed loss function to make the gradient estimator differentiable everywhere and estimate the whole gradient at a time; On the other hand, similar to ZOO, we use two points to estimate the gradient, in order to eliminate the dependency caused by $\gamma$, thus making the estimator stable and correct.
Specifically, based on Equation~\ref{eq:est1}, we first set $u$ as its negative form $-u$, which also belongs to $\mathcal{S}_p$, and have  
$\nabla\hat{L}(\delta) = \mathbb{E}_{\bm{u}\in\mathcal{S}_p}[\frac{N}{\gamma}L(\delta - \gamma\bm{u}) (-\bm{u})]$. 
Combining it with Equation~\ref{eq:est1}, we have 
the estimated gradient as: 
\begin{equation}
\small
\label{eq:est2}
\hat{g}^{TPGE} = \nabla\hat{L}(\delta) = \underset{\bm{u}\in\mathcal{S}_p}{\mathbb{E}}[\frac{N}{2\gamma}\big(L(\delta+\gamma\bm{u})-L(\delta-\gamma\bm{u})\big)\bm{u}].
\end{equation}

\noindent The properties of $\hat{g}^{TPGE}$ are shown in Theorem~\ref{th:grab} (See Section~\ref{sec:theorems}). In summary, $\hat{g}^{TPGE}$ is an unbiased gradient estimator and has a bounded gradient norm independent of $\gamma$. An unbiased gradient estimator is a necessary condition for the estimated gradient to be close to the true gradient; and a bounded gradient norm can make the gradient estimator stable.

\begin{table}[!t]
{\caption{Comparing ZOO, NES, OPGE, and TPGE.}
\centering
\footnotesize
\addtolength{\tabcolsep}{-4pt}
\begin{tabular}{|c|c|c|c|c|c|}
\hline
{\bf Method} & {\bf ZOO} & {\bf NES} & {\bf SimBA} & {\bf OPGE}  &  {\bf TPGE}   \\  \hline \hline
{\bf Type} & Deter.  & Stoc. & Stoc. &   Stoc.  &  Stoc. \\ \hline
{\bf \#Queries per round} & 2\#pixels  & $>=100$ & \#pixels  &   1  &  2 \\   \hline
{\bf Stable} & Yes & No & No & No & Yes  \\  \hline
{\bf Unbiased}  & Yes & Yes &No & Yes & Yes \\ \hline 
{\bf Differentiable Loss} & Yes & No & No & No & No \\ \hline
\end{tabular}
\label{tbl:comparison}
\vspace{-2mm}
}
\end{table}

\vspace{+0.5mm}
\noindent {\bf Comparing gradient estimators.}
Table~\ref{tbl:comparison} compares 
ZOO, NES, SimBA, OPGE, and TPGE in terms of query-efficiency,  
stability, unbiasedness of the estimated gradient, and whether the gradient estimator requires a differentiable loss or not.  
 We observe that our TPGE achieves the best trade-off among these metrics.

\subsubsection{Projected bandit gradient descent attacks to generate adversarial perturbations}
According to Equation \ref{eq:est2}, we need to calculate the expectation to estimate the gradient. In practice, TPGE samples a unit vector $u$ from $\mathcal{S}_p$ and estimates the expectation as $\tilde{g}^{TPGE} = {N}/{2\gamma}\big(L(\delta+\gamma\bm{u})-L(\delta-\gamma\bm{u})\big)\bm{u}$. 
Specifically, TPGE uses two points $\delta+\gamma\bm{u}$ and $\delta-\gamma\bm{u}$ to query the target model and obtains the 
bandit feedback $L(\delta+\gamma\bm{u})$ and $L(\delta-\gamma\bm{u})$. Thus, it is called two-point gradient estimator. 
We call $\tilde{g}^{TPGE}$ as a bandit gradient estimator because it is based on bandit feedback. Then, we use this bandit gradient estimator and propose the projected bandit gradient descent (PBGD) attack to iteratively
generate adversarial perturbations against image segmentation models as follows:
\begin{align}
\label{eqn:bgd}
\delta = \text{Proj}_{\Delta}(\delta + \alpha \cdot \tilde{g}^{TPGE}).
\end{align}
where $\alpha$ is the learning rate in the PBGD.
 $\textrm{Proj}_\Delta$ projects the adversarial perturbation to the allowable set $\Delta$. The final adversarial perturbation is used to perform the attack.

\subsubsection{Certified radius-guided projected bandit gradient descent attacks} 
We further propose to enhance the black-box attacks by leveraging the pixel-wise certified radius information. 
Observing from Equation~\ref{equ:pwcr}, we notice that calculating the pixel-wise certified radius only needs to know the outputs of the smoothed image segmentation model $G_\theta$, which can be realized by first sampling a set of noises offline and adding them to the testing image, and then querying the target model $F_\theta$ with these noisy images 
to build $G_\theta$.   
Therefore, \textit{we can seamlessly incorporate the pixel-wise certified radius into the projected bandit gradient descent black-box attack.} 
Specifically, we only need to replace the attack loss $L$ with the certified radius-guided attack loss $L_{cr}$ defined in Equation~\ref{eq:cr_loss}.  Then, we have the certified radius-guided two point gradient estimator, i.e., $\hat{g}_{cr}^{TPGE}$, as follows:
\begin{equation}\label{eq:est3}
\small
\hat{g}_{cr}^{TPGE} = \nabla\hat{L}_{cr}(\delta) =
\underset{\bm{u}\in\mathcal{S}_p}{\mathbb{E}} [\frac{N}{2\gamma}\big(L_{cr}(\delta+\gamma\bm{u})-L_{cr}(\delta-\gamma\bm{u})\big)\bm{u}]. 
\end{equation}
Similar to TPGE, we sample a $u$ from $\mathcal{S}_p$ and estimate the expectation as $\tilde{g}_{cr}^{TPGE} = \frac{N}{2\gamma}\big(L_{cr}(\delta+\gamma\bm{u})-L_{cr}(\delta-\gamma\bm{u})\big)\bm{u}$. 
Then, we iteratively
generate adversarial perturbations via the certified radius-guided PGBD 
(CR-PGBD) as follows: 
\begin{align}
\label{eqn:cr_bgd}
\delta = \text{Proj}_{\Delta}(\delta + \alpha \cdot \tilde{g}_{cr}^{TPGE}).
\end{align}

Algorithm \ref{alg:BGD_attack} in Appendix details 
our PBGD and {CR-PBGD} black-box attacks to image segmentation models. 
By attacking the target model up to $T$ rounds, the total number of queries of PBGD is $2T$, as in each round we only need to get 2 loss feedback. 
Note that in CR-PBGD, we need to sample $M$ noises and query the  model $M$ times to calculate the certified radius of pixels and get 2 loss feedback. In order to save queries, we only calculate the pixels'  certified radii every $\textrm{INT}$ iterations. Thus, the total number of queries 
of CR-PBGD is  
$(1-\frac{1}{\textrm{INT}}) \cdot 2T + \frac{1}{\textrm{INT}} \cdot (2+M)T = 2T + \frac{MT}{\textrm{INT}}$.

\subsection{Theoretical Results}
\label{sec:theorems}
In this subsection, we theoretically analyze our PBGD and CR-PGBD black-box attack algorithms. 
We first characterize the properties of our proposed gradient estimators and then show the regret bound of the two algorithms. 

Our analysis assumes the loss function to be Lipschitz continuous, 
 which has been validated in recent works~\cite{jordan2020exactly,leino2021globally} that loss functions in deep neural networks are often Lipschitz continuous. 
 Similar to existing regret bound analysis for bandit methods~\cite{zinkevich2003online}, we assume the loss function is {convex} (Please refer to Appendix~\ref{supp:grab} the definitions).
Note that optimizing non-convex functions directly is challenging due to its NP-hardness in general~\cite{jain2017non}. Also, there exist no 
tools to derive the optimal solution when optimizing non-convex functions, and thus existing works relax to convex settings. In addition, as pointed out in~\cite{jain2017non}, convex analysis often plays an important role in non-convex optimization. 
We first characterize the properties of our gradient estimators in the following theorem: 

\begin{restatable}[]{theorem}{grab}
\label{th:grab}
$\hat{g}^{TPGE}$ (or $\hat{g}_{cr}^{TPGE}$) is an unbiased gradient estimator of $\nabla\hat{L}$ (or $\nabla\hat{L}_{cr}$). Assume the loss function $L$ (or the CR-guided loss function $L_{cr}(\cdot)$) is $\hat{C}$ (or $\hat{C}_{cr}$)-Lipschitz continuous with respect to $l_p$-norm, then $\hat{g}^{TPGE}$ (or $\hat{g}_{cr}^{TPGE}$) has a bounded $l_p$-norm, i.e., $||\hat{g}||_p \le N\hat{C}$ (or $||\hat{g}_{cr}||_p \le N\hat{C}_{cr}$).
\end{restatable}
\begin{proof}
See Appendix~\ref{supp:grab}.
\end{proof}

Next, we analyze the regret bound achieved by our PBGD and CR-PBGD black-box  attacks. 

\begin{restatable}[]{theorem}{bgd}
\label{th:bgd}
Assuming $L$ (or $L_{cr}$) is $\hat{C}$ (or $\hat{C}_{cr}$)-Lipschitz continuous and both are convex. 
Suppose we use the PBGD attack to attack the image segmentation model $F_\theta$ up to $T$ rounds by setting a learning rate $\alpha = \frac{\sqrt{N}}{2\hat{C}\sqrt{T}}$ and $\gamma=\frac{N^{3/2}}{6\sqrt{T}}$ in Equation~\ref{eq:est2}. Then, the attack incurs a sublinear regret $R_\mathcal{A}^{PBGD}(T)$ bounded by $\mathcal{O}(\sqrt{T})$, i.e.,
\begin{equation}
\footnotesize
R_\mathcal{A}^{PBGD}(T) = \sum_{t=1}^T\mathbb{E}\{L(\delta^{(t)})\} - TL(\delta_*) \le N^{{\frac{3}{2}}}\hat{C}\sqrt{T}.
\end{equation}
Similarly, if we use the CR-PBGD attack with a learning rate $\alpha = \frac
{\sqrt{N}}{2\hat{C}_{cr}\sqrt{T}}$ and $\gamma=\frac{N^{3/2}}{6\sqrt{T}}$,  the attack incurs a sublinear regret $R_\mathcal{A}^{CR-PBGD}(T)$ bounded by $\mathcal{O}(\sqrt{T})$, 
\begin{small}
\begin{equation*}
\footnotesize
R_\mathcal{A}^{CR-PBGD}(T) = \sum_{t=1}^T\mathbb{E}\{L_{cr}(\delta^{(t)})\} - TL_{cr}(\delta_*) \le N^{\frac{3}{2}}\hat{C}_{cr}\sqrt{T}.
\end{equation*}
\end{small}%
\vspace{-4mm}
\end{restatable}
\begin{proof}
See Appendix~\ref{supp:main}.
\end{proof}

\noindent {\bf Remark.} The sublinear regret bound  establishes our theoretically guaranteed attack performance, and it indicates the worst-case regret of our black-box attacks. With a sublinear regret bound $O(\sqrt{T})$, the time-average regret (i.e., $R_{\mathcal{A}}(T)/T$) of our attacks will diminish as $T$ increases (though the input dimensionality $N$ may be large), which also implies that the generated 
adversarial perturbation is asymptotically optimal.
 Moreover, the $O(\sqrt{T})$ bound is tight, meaning our attack obtains the asymptotically optimal perturbation with an optimal rate.  
More discussions about regret bounds are in Appendix~\ref{sec:regret}.

\section{Evaluation}
\label{sec:eval}

\vspace{-2mm}
\subsection{Experimental Setup}
\label{sec:exp_setup}
\vspace{-2mm}
\noindent {\bf Datasets.} 
We use three widely used segmentation datasets, i.e., Pascal VOC~\cite{pascal-voc-2012}, Cityscapes~\cite{cordts2016cityscapes}, and ADE20K~\cite{zhou2019semantic} for evaluation. More dataset details are in Appendix~\ref{app:datasets}.

\noindent {\bf Image segmentation models.}
We select three modern image segmentation models (i.e., PSPNet, PSANet~\cite{zhao2017pyramid,zhao2018psanet}\footnote{\url{https://github.com/hszhao/semseg}}, and HRNet~\cite{sun2019high}\footnote{\url{https://github.com/HRNet/HRNet-Semantic-Segmentation}}) for evaluation.
We use their public 
pretrained
models to evaluate the attacks. 
By default, we use PSPNet, HRNet, and PSANet to evaluate  Pascal VOC, Cityscapes, and ADE20K, respectively.
Table~\ref{tab:clean} shows the clean pixel accuracy and {MIoU (See the end of Section~\ref{sec:exp_setup})} 
of the three models on the three datasets.  

\noindent {\bf Compared baselines.} 
We implement our attacks in PyTorch. All models are run on a Linux server with 96 core 3.0GHz CPU, 768GB
RAM, and 8 Nvidia  A100 GPUs. 
The source code of our attacks is publicly available at\footnote{\url{https://github.com/randomizedheap/CR\_Attack}}. 

\begin{itemize}[leftmargin=*]
\item {\bf White-box attack algorithms.} 
\cite{arnab2018robustness} performed a systematic study to understand the robustness of modern 
segmentation models against adversarial perturbations. They found that 
the PGD attack~\cite{madry2018towards} performed the best among 
the compared 
attacks (We also have the same conclusion in Table~\ref{tab:comp_exist}). 
Thus, in this paper, we mainly use PGD as the base attack 
and compare it with our CR-PGD attack.
We note that our certified radius can be incorporated into all the existing 
white-box attacks and we show additional results in 
Section~\ref{sec:discussion}. 
Details of the existing attack methods are shown in Appendix~\ref{supp:methods}.

\item {\bf Black-box attack algorithms.} 
We mainly evaluate our  projected bandit gradient descent (PBGD) attack and certified radius-guided PBGD  (CR-PBGD) attack.

\end{itemize}

\begin{table}[!t]
\centering
\footnotesize
    \caption{PixAcc and MIoU of the three segmentation models on the three datasets without attack.}
    \label{tab:clean}
      \begin{tabular}{c|c|cc}
      \hline
      \multirow{1}{*}{\bf Model}
        & {\bf Dataset} & {\bf PixAcc}& {\bf MIoU} \cr
        \hline
        \hline
        \multirow{3}{*}{\bf PSPNet}
        & {\bf Pascal VOC}  & $94.4\%$ & $77.3\%$ \cr
        & {\bf Cityscapes} & $95.0\%$ & $72.7\%$ \cr
        & {\bf ADE20K} & $78.2\%$ & $38.1\%$ \cr
        \cline{1-4}
        \multirow{3}{*}{\bf PSANet}
        & {\bf Pascal VOC} & $94.3\%$ & $76.9\%$\cr
        & {\bf Cityscapes} & 94.8$\%$ & $71.4\%$\cr
        & {\bf ADE20K} & $79.4\%$ & $39.5\%$\cr
        \cline{1-4}
        \multirow{3}{*}{\bf HRNet-OCR}
        & {\bf Pascal VOC} & $94.8\%$ & $79.4\%$\cr
        & {\bf Cityscapes} & $95.2\%$ & $74.3\%$\cr
        & {\bf ADE20K} & $80.5\%$ & $40.6\%$\cr
        \cline{1-4}
        \hline
      \end{tabular}
\vspace{-4mm}
\end{table}

\begin{figure*}[!t]
\centering
\vspace{-3mm}
\subfigure[Pascal VOC]{\includegraphics[width=0.28\textwidth]{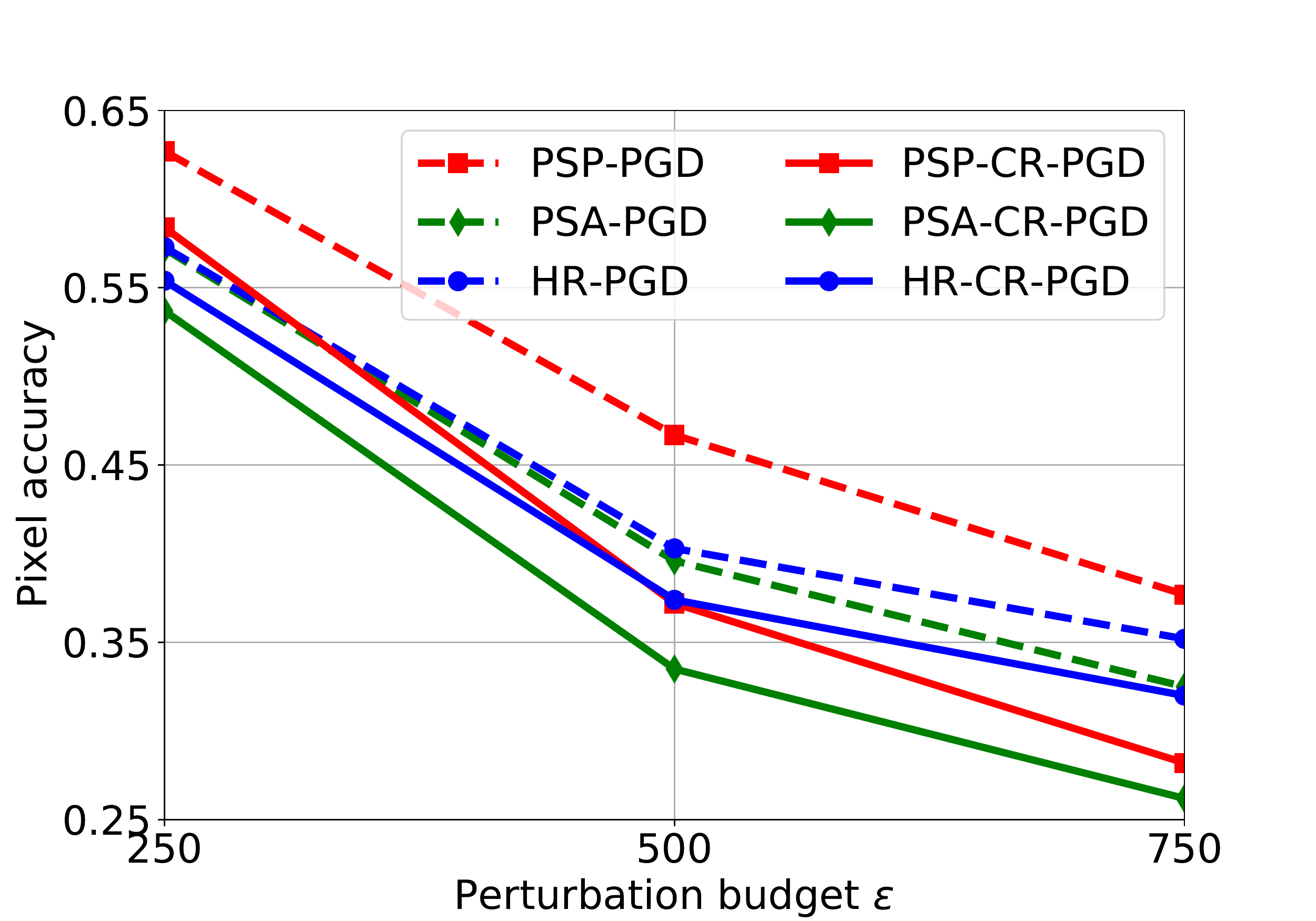}}
\subfigure[Cityscapes]{\includegraphics[width=0.28\textwidth]{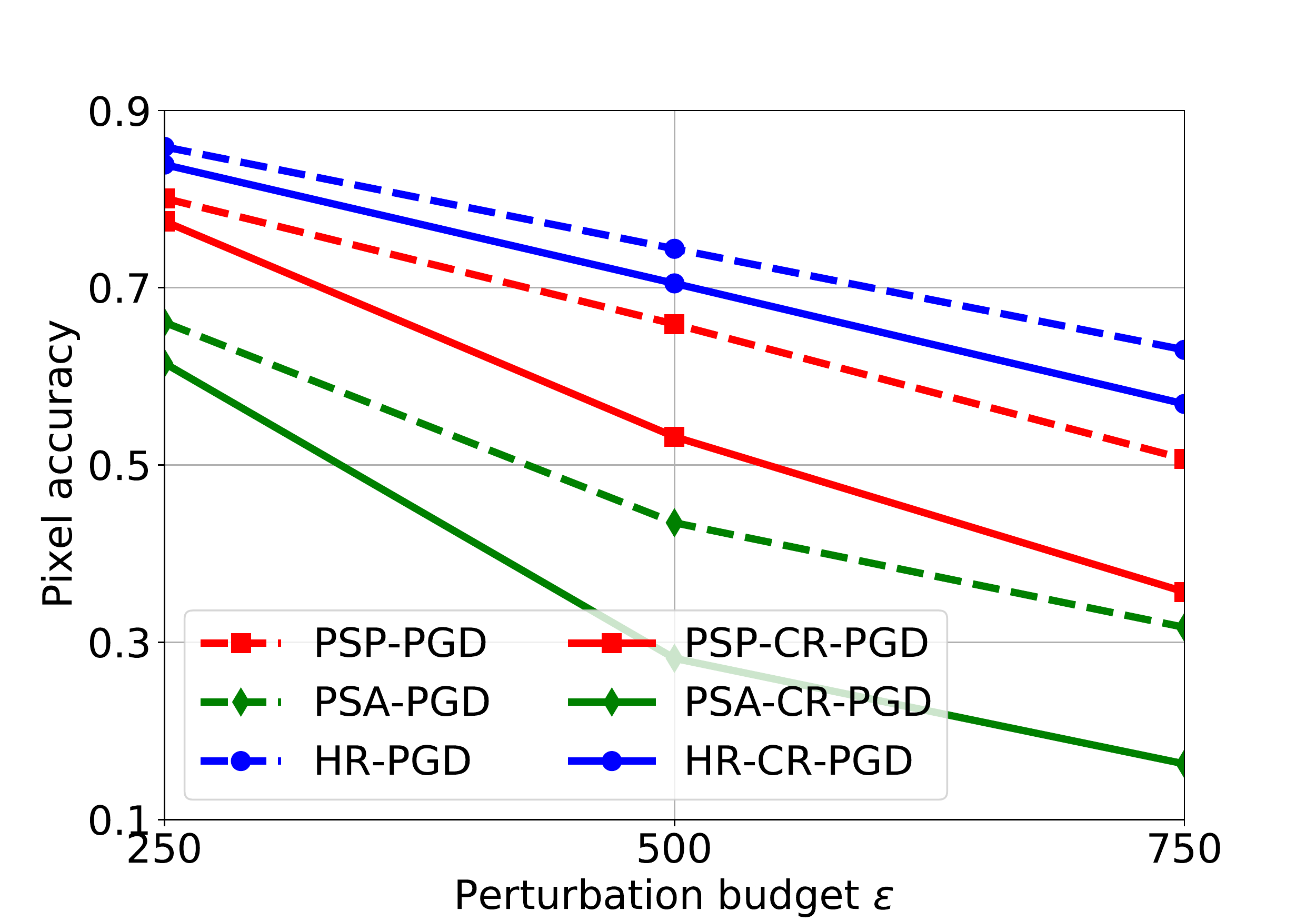}}
\subfigure[ADE20K]{\includegraphics[width=0.28\textwidth]{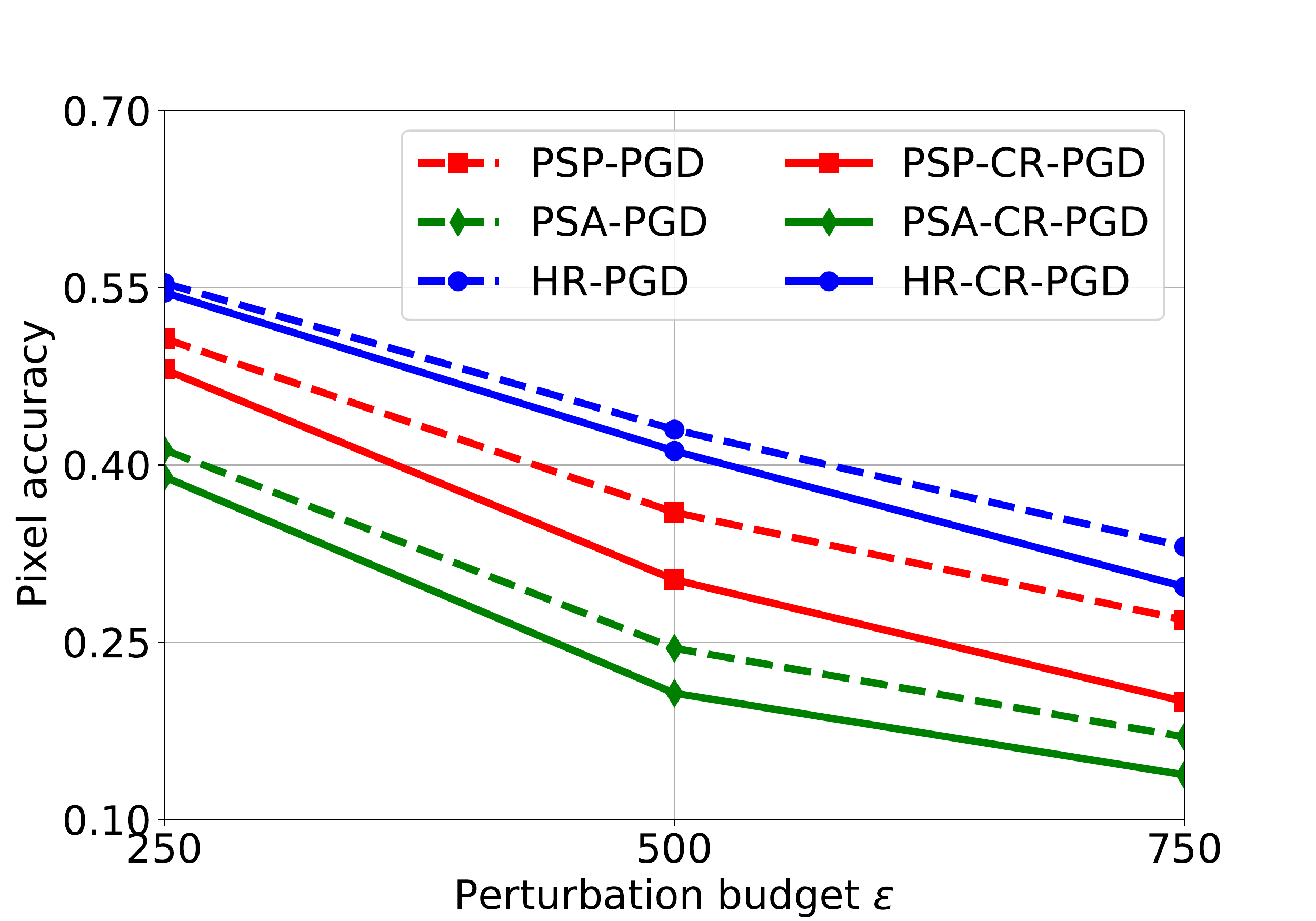}}
\vspace{-2mm}
\caption{PixAcc after PGD and CR-PGD attacks with $l_1$ perturbation vs. perturbation budget $\epsilon$ on the three models and datasets.} 
\label{fig:whitel1_acc}
\vspace{-6mm}
\end{figure*}

\begin{figure*}[!t]
\centering
\subfigure[Pascal VOC]{\includegraphics[width=0.28\textwidth]{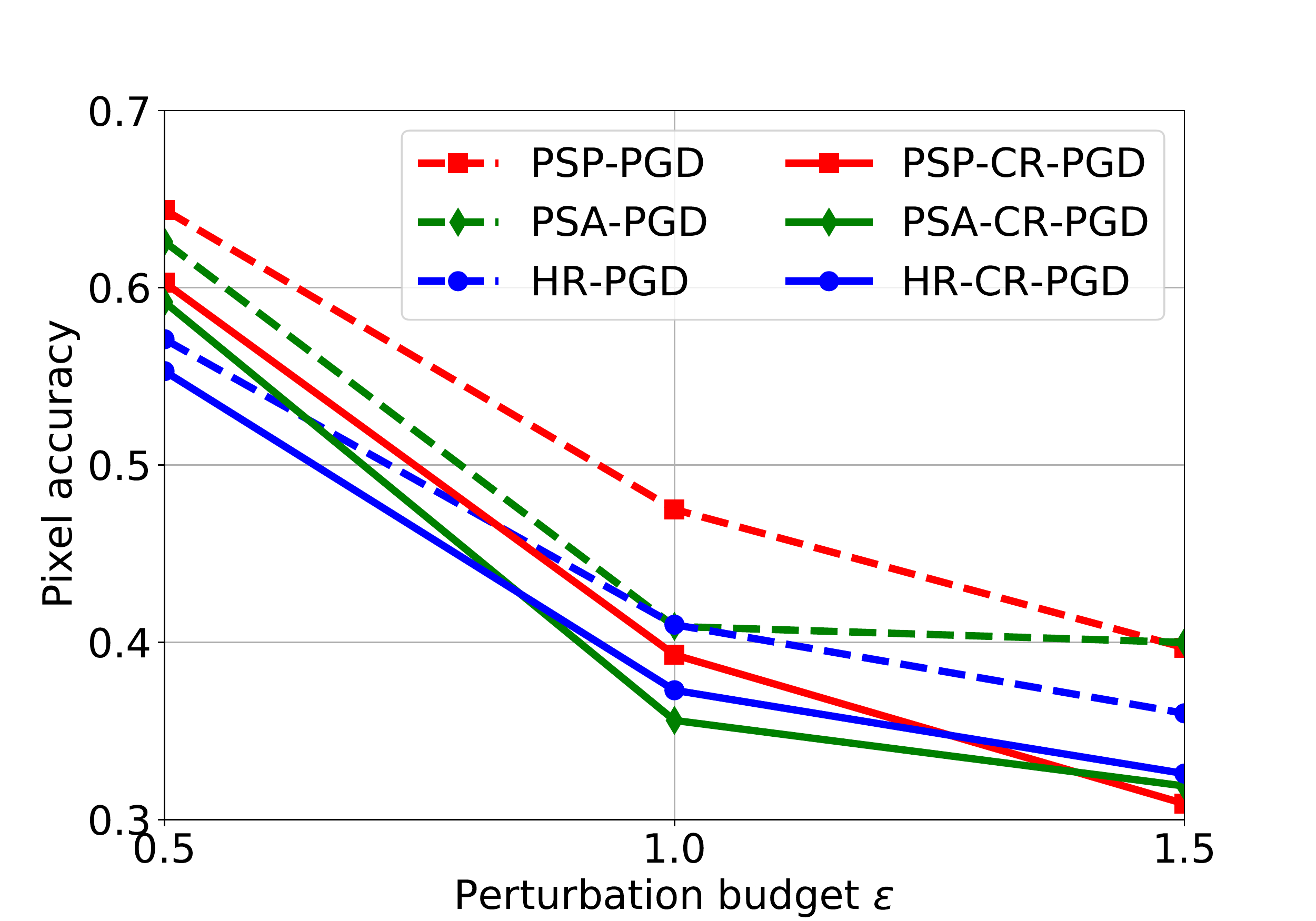}}
\subfigure[Cityscapes]{\includegraphics[width=0.28\textwidth]{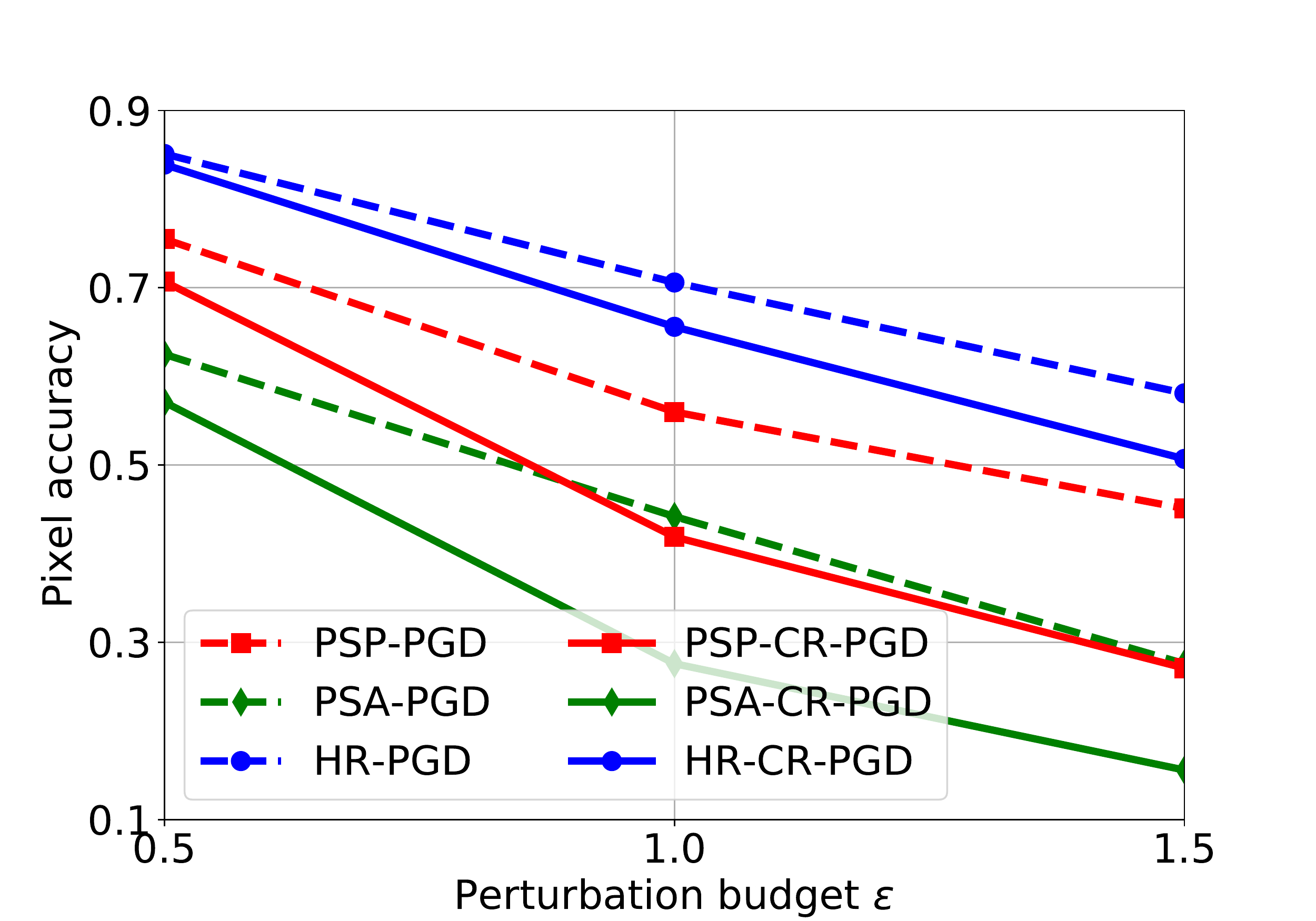}}
\subfigure[ADE20K]{\includegraphics[width=0.28\textwidth]{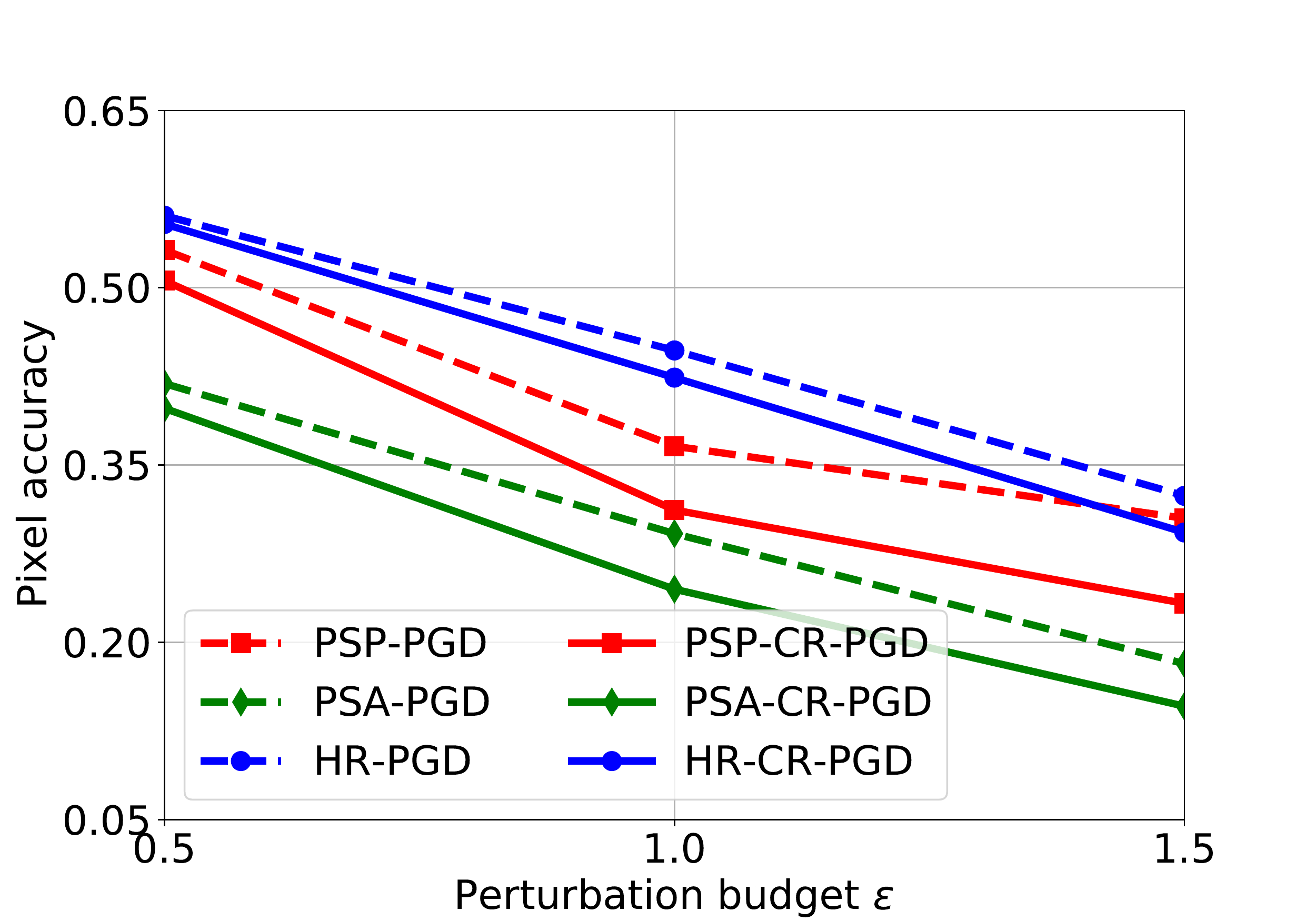}}
\vspace{-2mm}
\caption{PixAcc after PGD and CR-PGD attacks with $l_2$ perturbation vs. perturbation budget $\epsilon$ on the three models and datasets.} 
\label{fig:whitel2_acc}
\vspace{-6mm}
\end{figure*}

\begin{figure*}[!t]
\centering
\subfigure[Pascal VOC]{\includegraphics[width=0.28\textwidth]{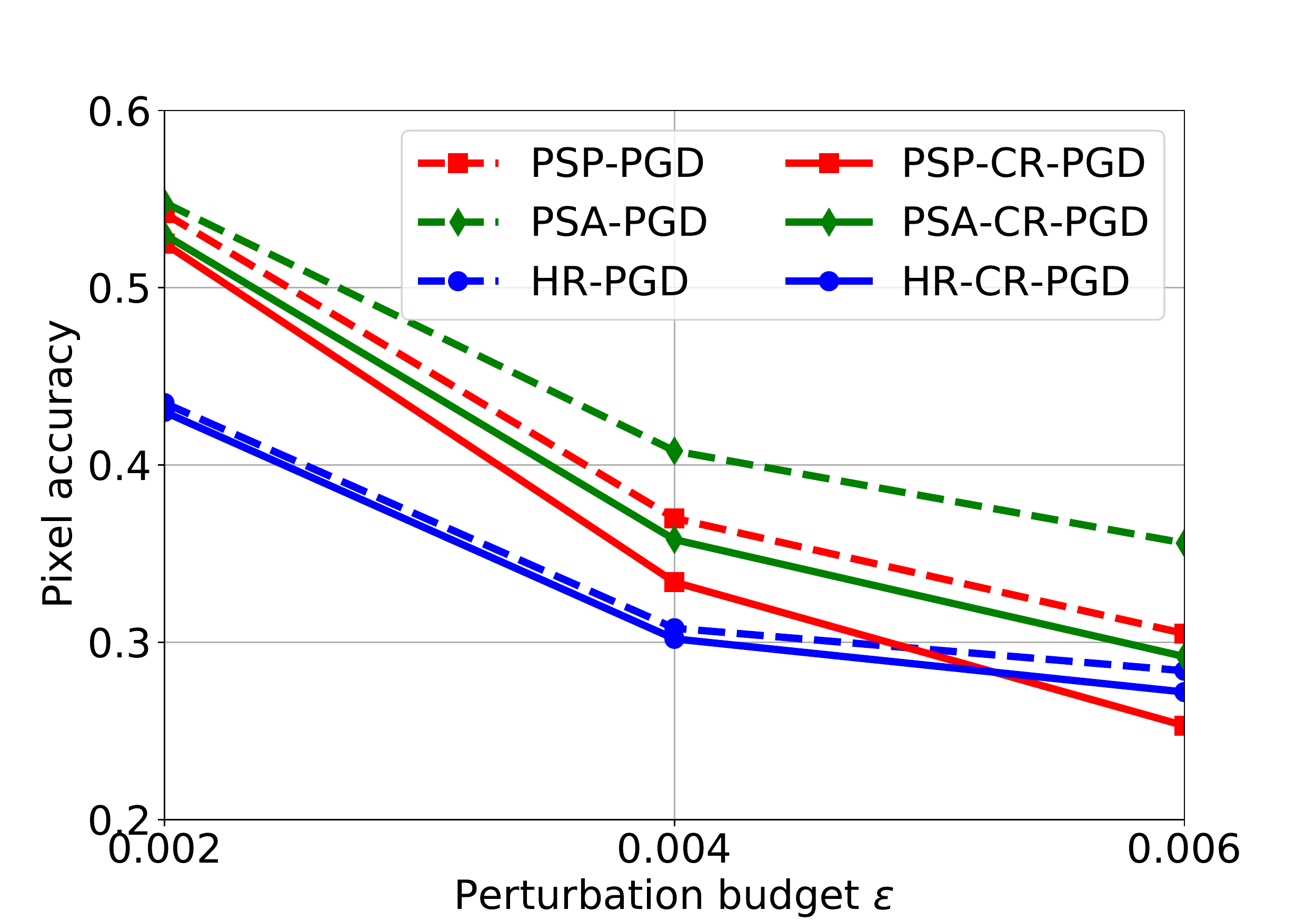}}
\subfigure[Cityscapes]{\includegraphics[width=0.28\textwidth]{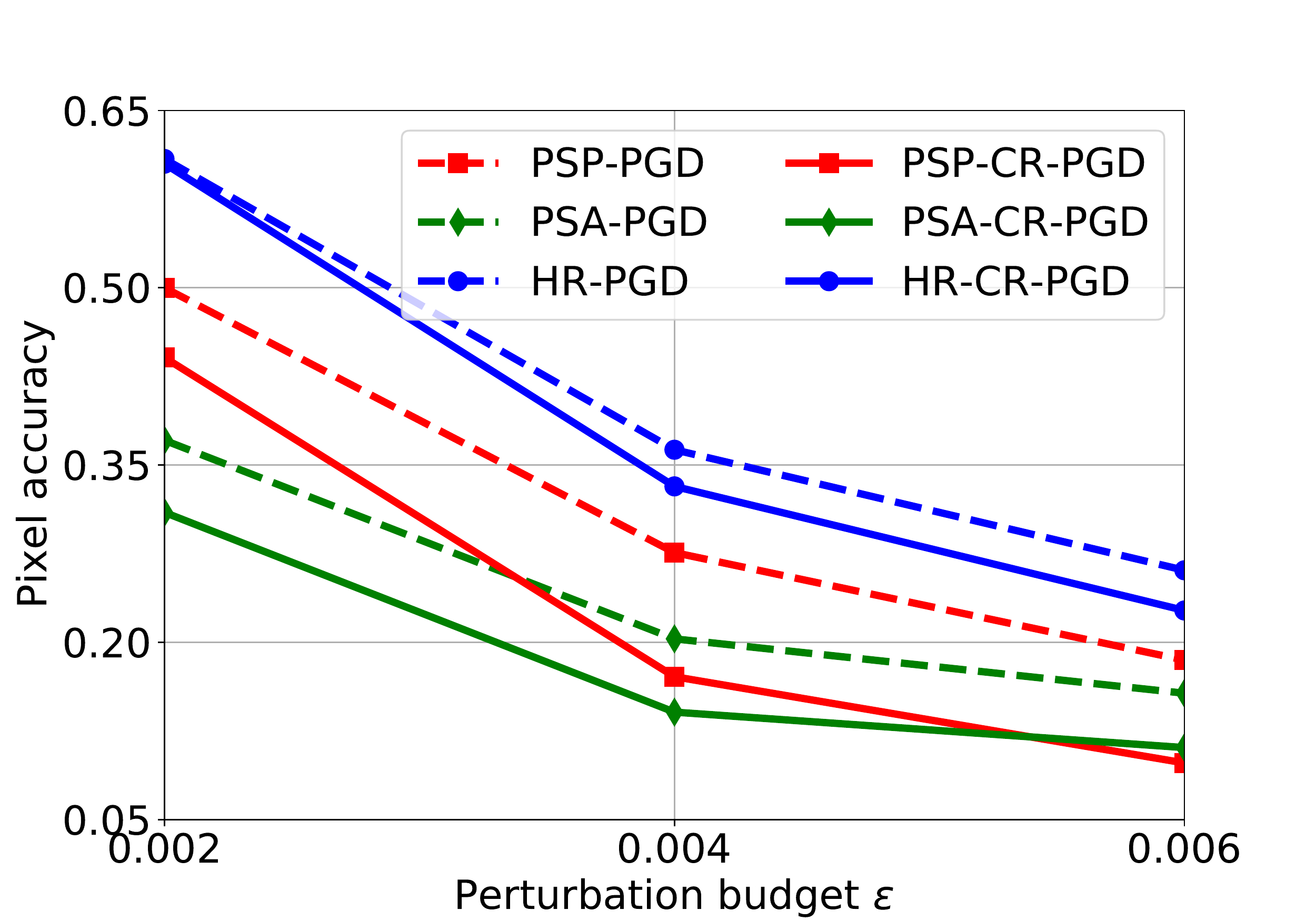}}
\subfigure[ADE20K]{\includegraphics[width=0.28\textwidth]{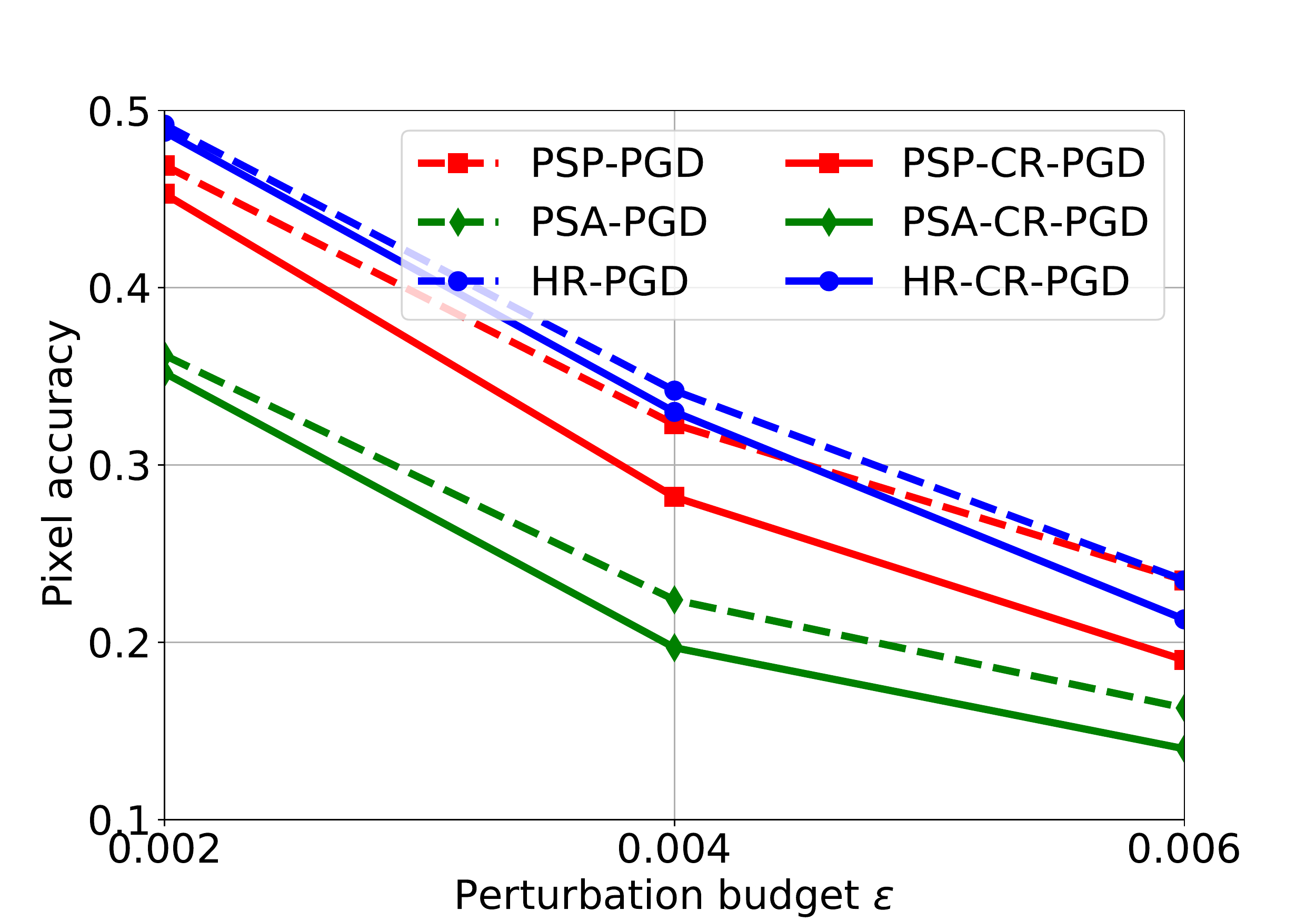}}
\vspace{-2mm}
\caption{PixAcc after PGD and CR-PGD attacks with $l_\infty$ perturbation vs. perturbation budget $\epsilon$ on the three models and  datasets.} 
\label{fig:whitelinf_acc}
\vspace{-4mm}
\end{figure*}

\noindent {\bf Parameter settings.} 
Consider black-box attacks are more challenging than white-box attacks, we  set different values for certain hyperparameters in the two attacks. For instance, we set a larger $l_p$ perturbation budget $\epsilon$ and larger number of iterations $T$ when evaluating black-box attacks.

\begin{itemize}[leftmargin=*]
\item {\bf White-box attack settings.} Our CR-PGD and PGD 
attacks share the same hyperparameters. Specifically, we set the total number of iterations $T=50,50,20$ and the learning rate $\alpha=\frac{2.5 \epsilon}{T}, \frac{2.5 \epsilon}{T}, \frac{ \epsilon}{T}$ to generate $l_1$ $l_2$, and $l_\infty$ perturbations, respectively. 
We set the weight parameters $a=2$ and $b=-4$, and $\textrm{INT} = M$.
We also study the impact of the important hyperparameters in our CR-PGD attack: {Gaussian} 
noise $\sigma$ in certified radius, number of samples $M$ in {Monte} 
Carlo sampling, and $l_p$ perturbation budget $\epsilon$, etc. 
By default, we set $\sigma=0.001$ and 
$M=8$, and $\epsilon=750, 1.5, 0.006$ for $l_1$, $l_2$, and $l_\infty$ perturbations, respectively. When studying the impact of a hyperparameter, we fix the other hyperparameters to be their default values. 

\item {\bf Black-box attack settings.} We set the learning rate $\alpha=5 \cdot 10^{-4}$, Gaussian noise $\sigma=0.001$, number of samples $M=8$, 
weight parameters $a=2$ and $b=-4$, and $\textrm{INT} = 2M$. 
We mainly study the impact of the total number of iterations $T$ 
and the $l_p$ perturbation budget $\epsilon$. 
By default, we set $T=15,000$ and $\epsilon=10000, 10, 0.05$ for $l_1$, $l_2$, and $l_\infty$ perturbations, respectively.

\end{itemize}

\noindent {\bf Evaluation metrics.} We use two common metrics to evaluate attack performance to image segmentation models. 

\begin{itemize}[leftmargin=*]
\item {\bf Pixel accuracy (PixAcc):} 
Fraction of pixels 
predicted correctly by the target model over the testing set.

\item {\bf Mean Intersection over Union ({MIoU}):} For each testing image $x$ and pixel label $c$, it first computes the ratio $IoU_c^x = \frac{|P_c^x \bigcap G_c^x|}{|P_c^x \bigcup G_c^x|}$, where $P_c^x$ denotes the pixels in $x$ predicted as label $c$ and $G_c^x$ denotes the pixels in $x$ with the groudtruth label $c$. Then, \emph{MIoU} is the average of $IoU_c^x$ over all testing images and all pixel labels.
\end{itemize}

\subsection{Results on White-box Attacks}
\label{sec:eval_wb}

\vspace{-2mm}
In this section, we show results on white-box attacks to segmentation models. 
We first compare the existing attacks  and show that the PGD attack achieves the best attack performance.
Next, we compare our certified radius-guided PGD attack with the PGD attack. Finally, we study the impact of the important hyperparameters in our CR-PGD attack. We defer all MIoU results to Appendix~\ref{app:results}.

\begin{table}[!t]\renewcommand{\arraystretch}{0.8}
\centering
\footnotesize
    \caption{Pixel accuracy with the existing attacks.}
    \label{tab:comp_exist}
      \begin{tabular}{c|c|c|c|c}
      \hline
      \multirow{2}{*}{\bf Attack} & \multirow{2}{*}{\bf Norm} & \multicolumn{3}{c}{\bf Dataset}                            \\ \cline{3-5} 
                  &          &    {\bf Pascal VOC}      & \multicolumn{1}{c|}{\bf CityScapes} & \multicolumn{1}{c}{\bf ADE20K}  \\ \hline
        \hline
        \multirow{3}{*}{\bf FGSM}
        & {\bf $l_1$}  & $79.8\%$ & $90.8\%$  & $59.4\%$ \cr
        & {\bf $l_2$} & $80.7\%$ & $91.0\%$  & $59.9\%$ \cr
        & {\bf $l_\infty$} & $74.7\%$ & $85.1\%$  & $54.4\%$ \cr
        \cline{1-5}
        {\bf DAG}
        & {\bf $l_\infty$} & $69.1\%$ & $75.4\%$  & $47.8\%$ \cr
        \cline{1-5}
        \multirow{3}{*}{\bf PGD}
        & {\bf $l_1$} & $37.7\%$ & $63.0\%$  & $17.0\%$ \cr
        & {\bf $l_2$} & $39.7\%$ & $58.1\%$  & $18.2\%$ \cr
        & {\bf $l_\infty$} & $30.5\%$ & $26.1\%$ & $23.5\%$ \cr
        \cline{1-5}
        \hline
      \end{tabular}
\vspace{-3mm}
\end{table}

\vspace{-2mm}
\subsubsection{Verifying that PGD outperforms the existing attacks}
We compare the PGD attack~\cite{arnab2018robustness} with FGSM~\cite{arnab2018robustness} and DAG~\cite{xie2017adversarial}, two other well-known attacks to image segmentation models (Please see their details in Appendix~\ref{supp:methods}). Note that DAG is only for $l_\infty$ perturbation.   
We set the perturbation budget $\epsilon$ to be $750,1.5,0.006$ for $l_1$, $l_2$, and $l_p$ perturbations, respectively.  
The compared results of these attacks are shown in Table~\ref{tab:comp_exist}. 
We observe that PGD consistently and significantly outperforms {FGSM} 
and DAG in all the three datasets. Based on this observation, we will select PGD as the default base attack in our certified radius-guided attack framework.

\vspace{-3mm}
\subsubsection{Comparing CR-PGD with PGD} We compare PGD and CR-PGD attacks with respect to $l_1$, $l_2$, and $l_\infty$ perturbations. The 
pixel accuracy on the three datasets vs. 
perturbation budget $\epsilon$ are shown in Figure~\ref{fig:whitel1_acc}, Figure~\ref{fig:whitel2_acc}, and Figure~\ref{fig:whitelinf_acc}, respectively. 
The MIoU on the three datasets vs. $l_1$, $l_2$, and $l_\infty$ perturbation budget $\epsilon$ are shown in Figure~\ref{fig:whitel1_mio}, Figure~\ref{fig:whitel2_mio}, and Figure~\ref{fig:whitelinf_mio} in Appendix~\ref{app:results}, respectively. 
We have the following observations:

\begin{figure}[!t]
\centering
\vspace{-4mm}
\subfigure[PGD]{\includegraphics[width=0.23\textwidth]{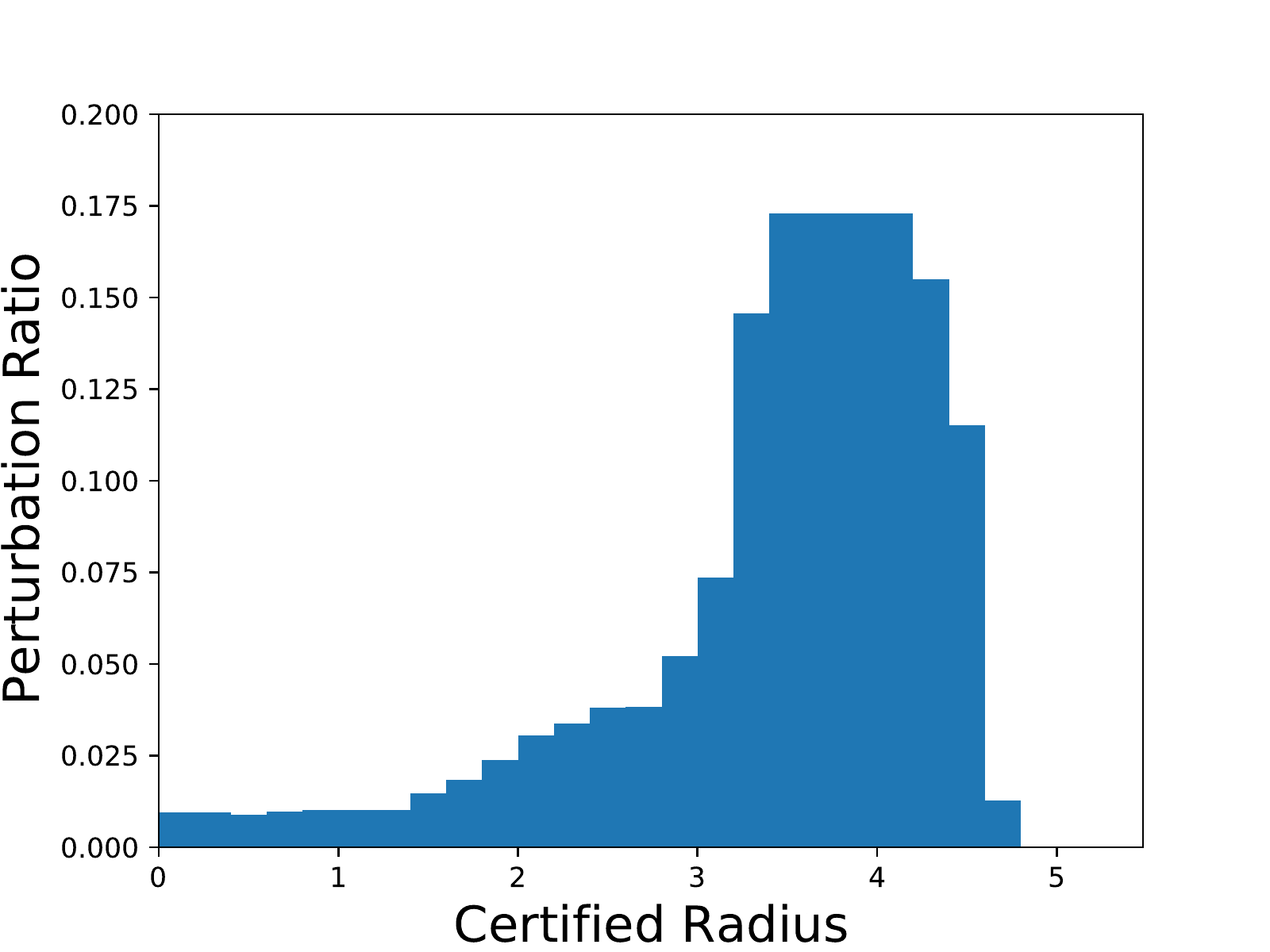}}
\subfigure[CR-PGD]{\includegraphics[width=0.23\textwidth]{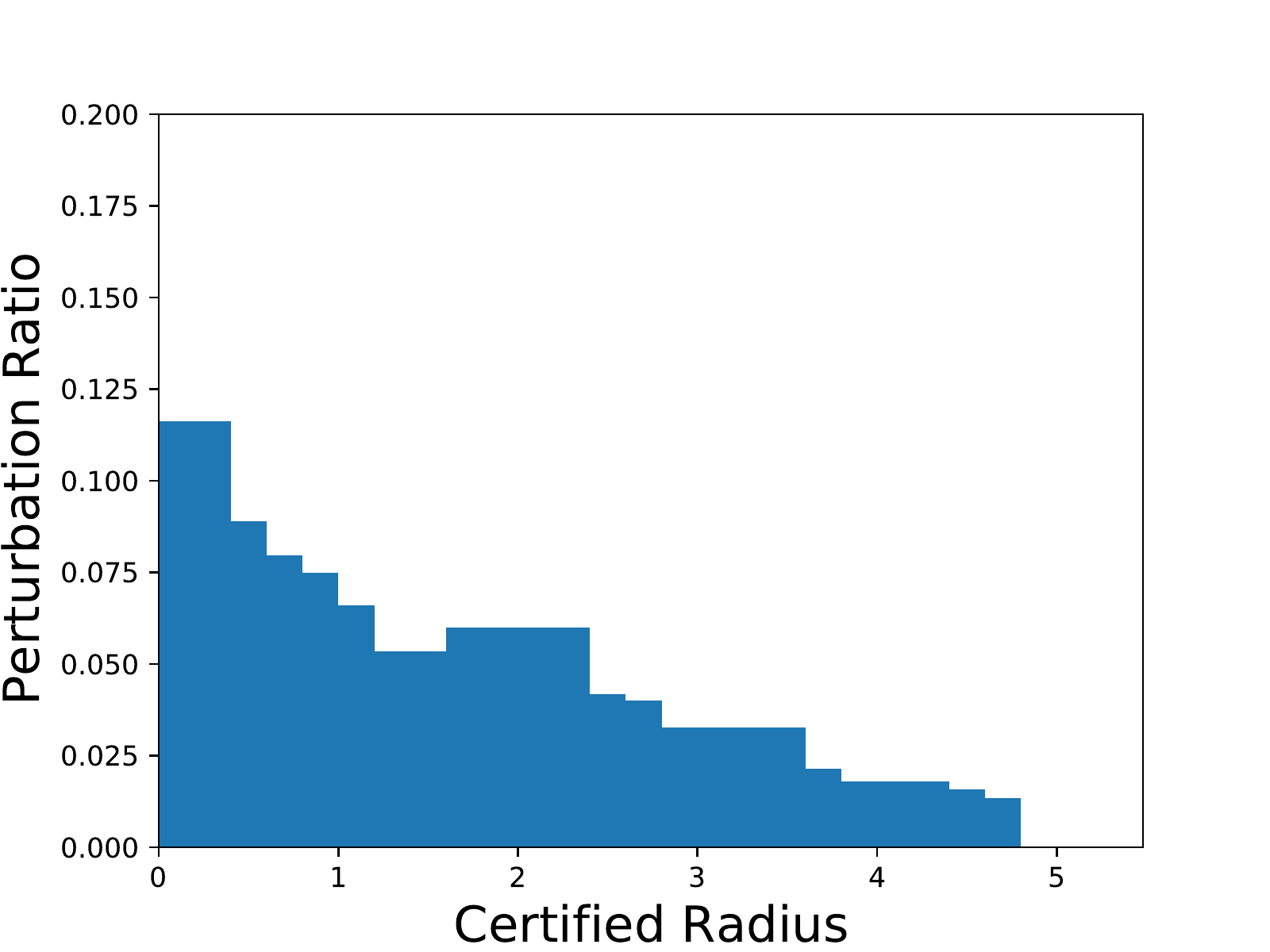}}
\vspace{-2mm}
\caption{{Distribution of pixel perturbations vs. pixel-wise certified radius of PGD and CR-PGD on 10 random images.}} 
\label{fig:distpwcr}
\vspace{-4mm}
\end{figure}

\begin{figure*}[!t]
\centering
\subfigure[$L_1$ perturbation]{\includegraphics[width=0.3\textwidth]{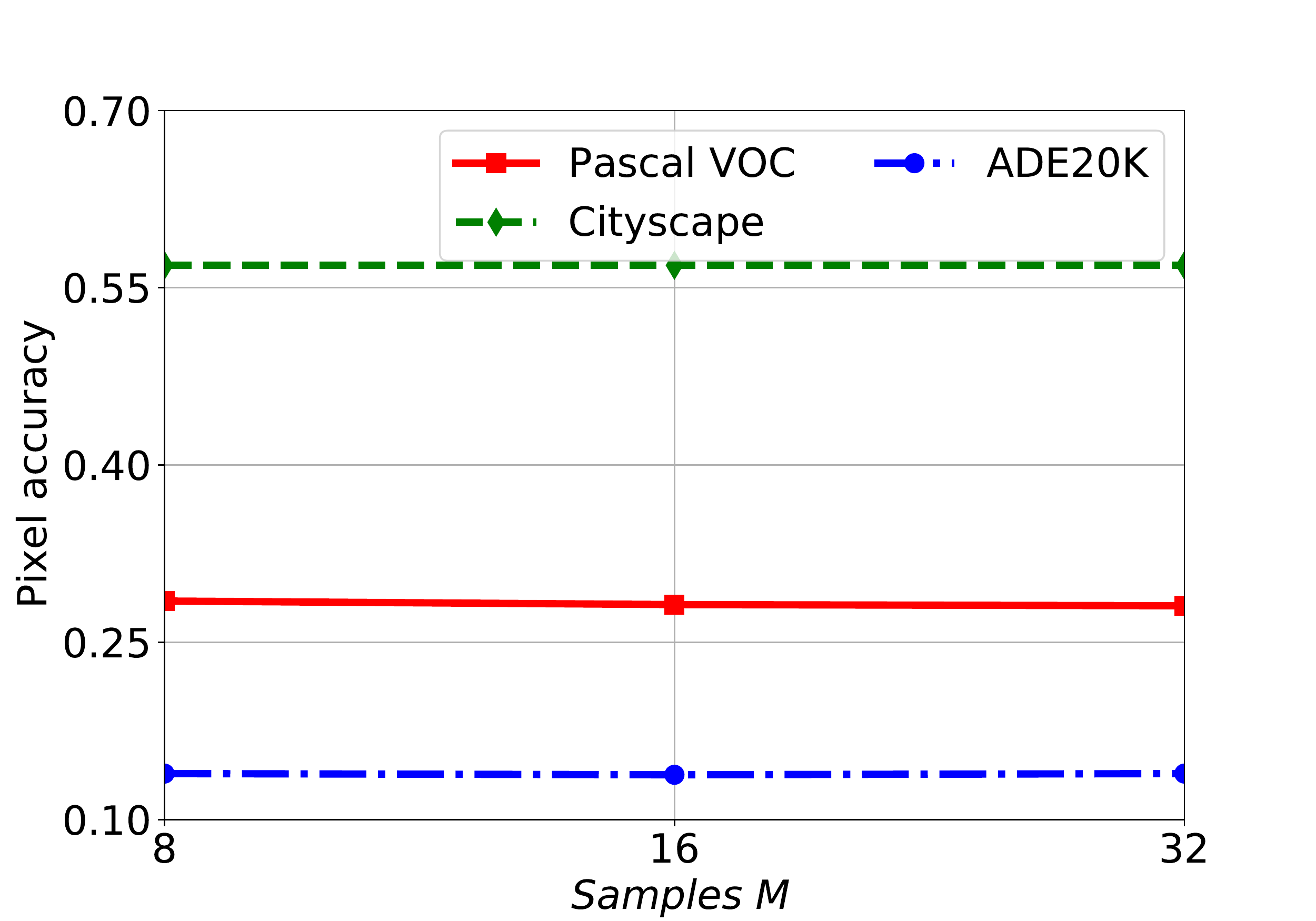}}
\subfigure[$L_2$ perturbation]{\includegraphics[width=0.3\textwidth]{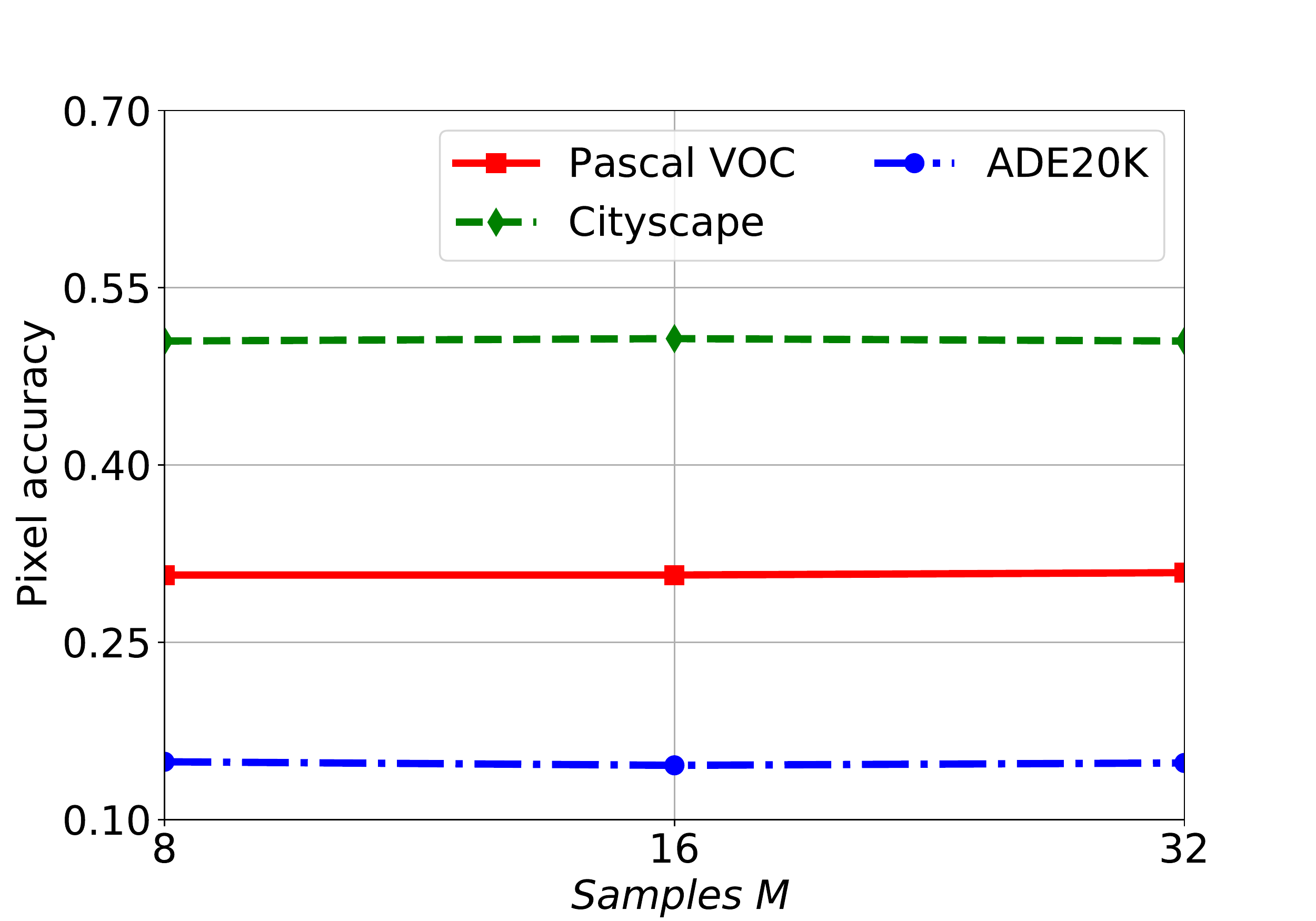}}
\subfigure[$L_\infty$ perturbation]{\includegraphics[width=0.3\textwidth]{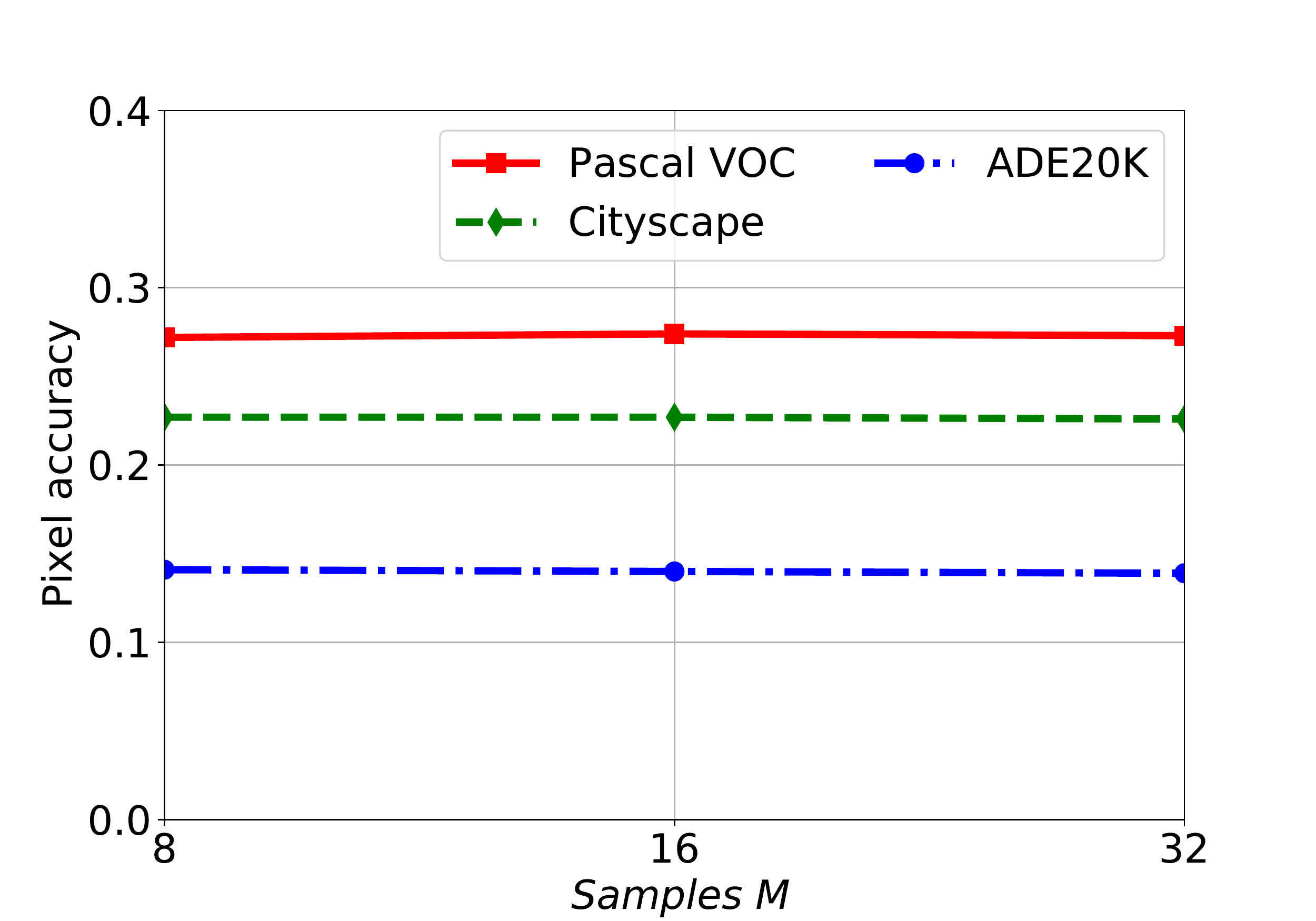}}
\vspace{-2mm}
\caption{Impact of the \#samples $M$ on our CR-PGD attack with different $l_p$ perturbations.} 
\label{fig:white_param_acc_m}
\vspace{-6mm}
\end{figure*}

\begin{figure*}[!t]
\centering
\subfigure[$L_1$ perturbation]{\includegraphics[width=0.3\textwidth]{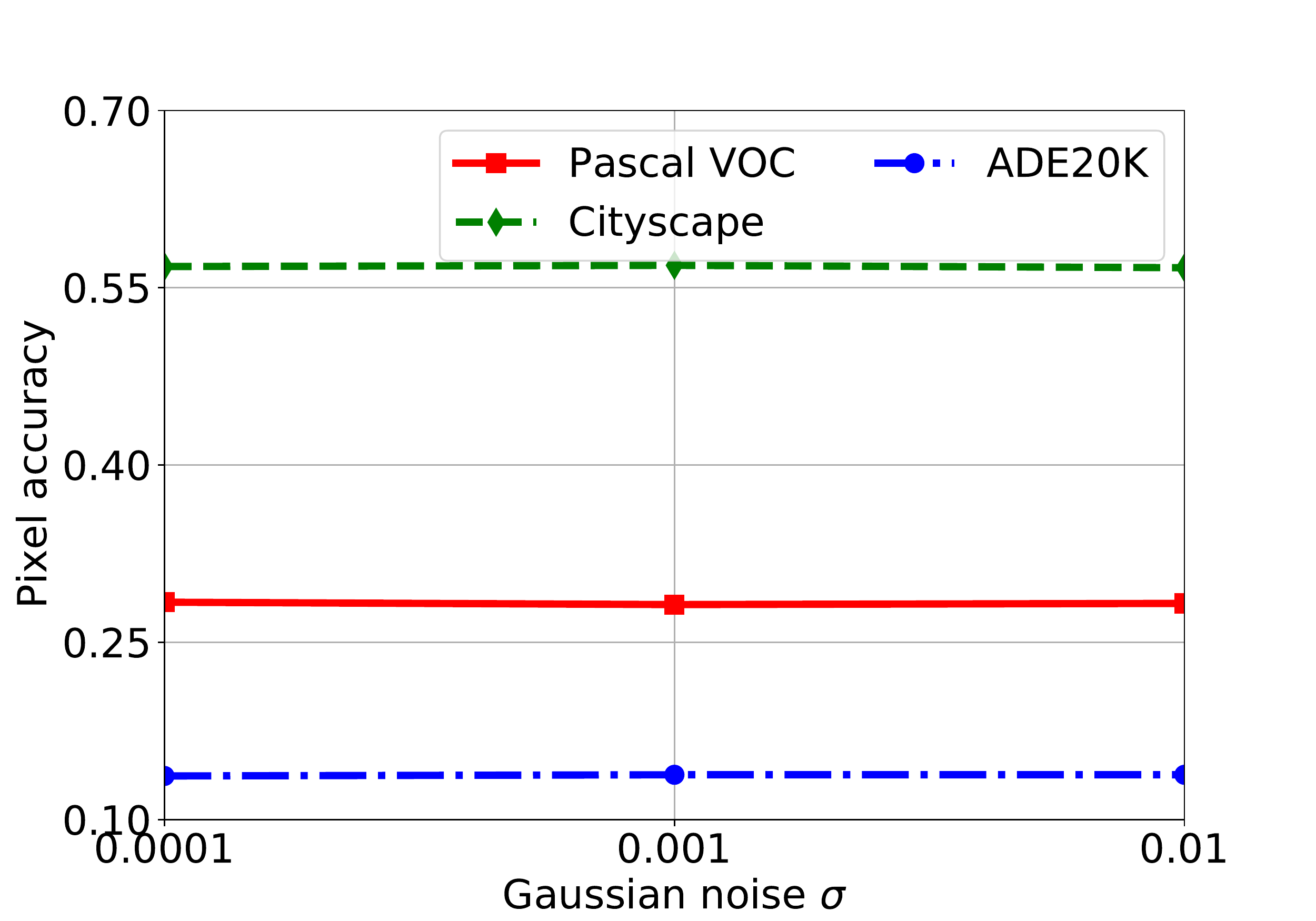}}
\subfigure[$L_2$ perturbation]{\includegraphics[width=0.3\textwidth]{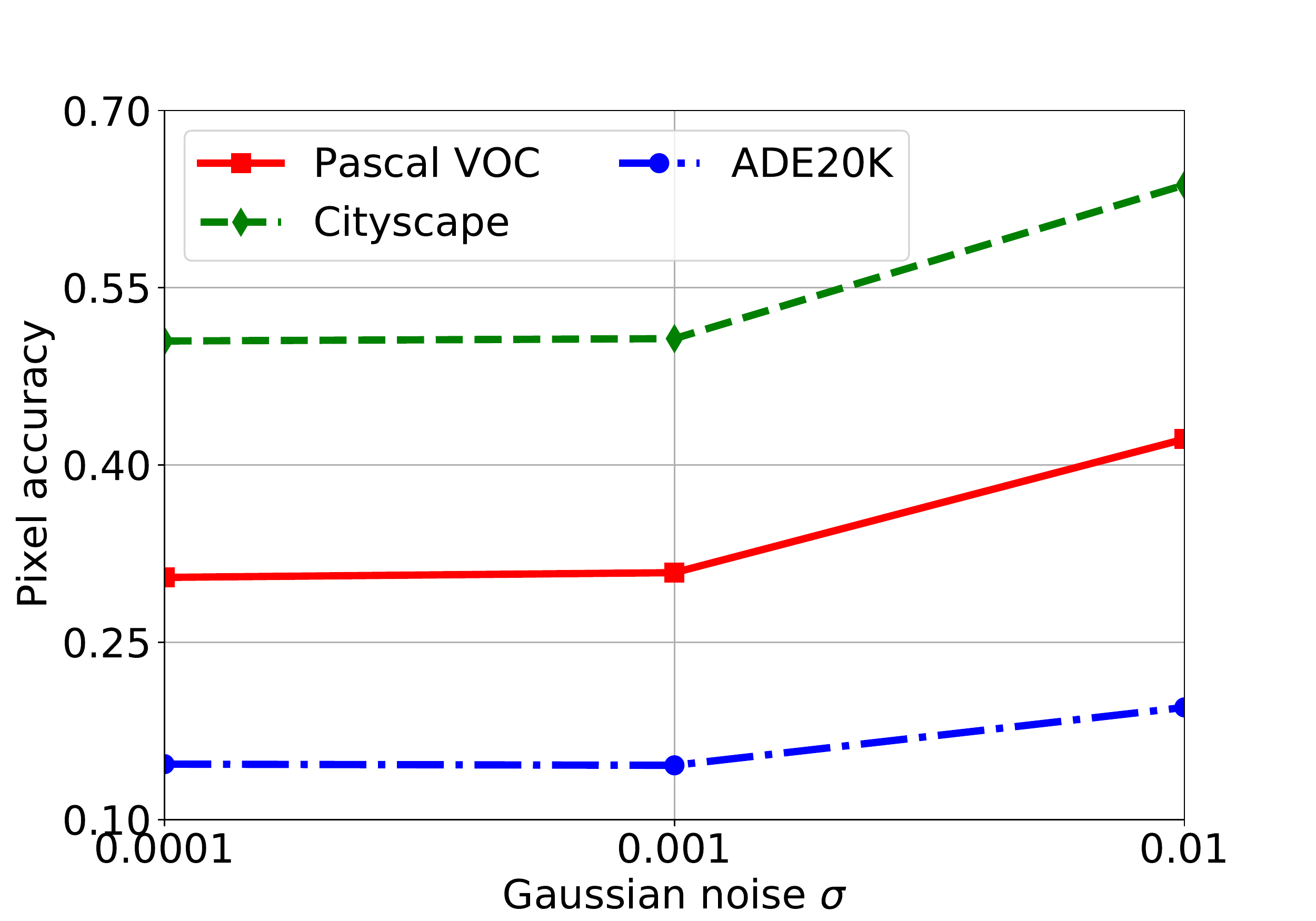}}
\subfigure[$L_\infty$ perturbation]{\includegraphics[width=0.3\textwidth]{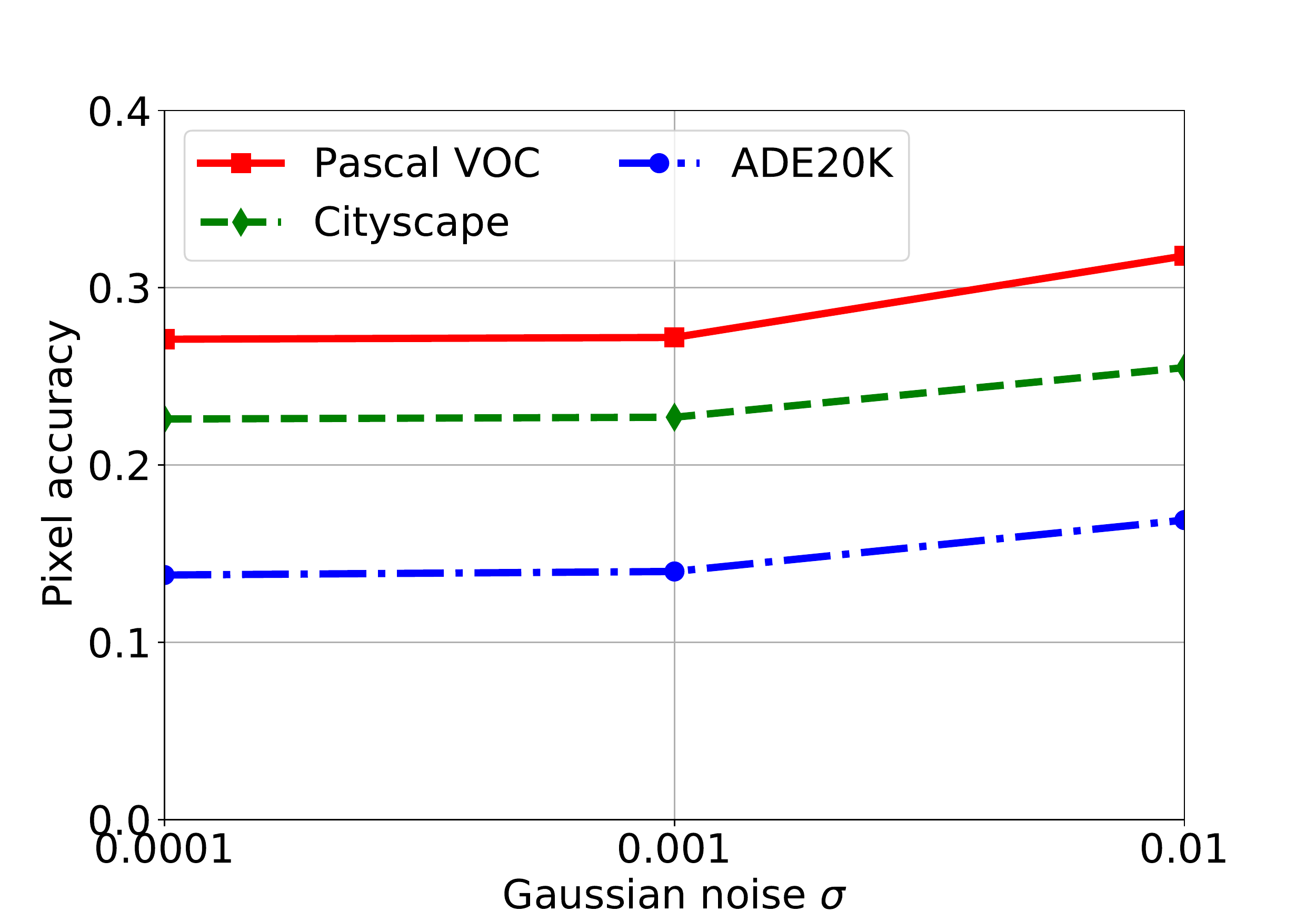}}
\vspace{-2mm}
\caption{Impact of the Gaussian noise $\sigma$ on our CR-PGD attack with different $l_p$ perturbations.} 
\label{fig:white_param_acc_s}
\vspace{-4mm}
\end{figure*}

\noindent {\bf Our CR-PGD attack consistently outperforms the PGD attack in all datasets,  models, and $l_p$ perturbations.} For instance, when attacking PSANet on Cityscapes with $l_1$-perturbation and $\epsilon=500$, our CR-PGD attack has a relative $53.8\%$ gain over the PGD attack in reducing the pixel accuracy;  
When attacking  HRNet  on Pascal VOC with $l_2$-perturbation and $\epsilon=1.5$, CR-PGD has a relative $9.4\%$ gain over PGD;  
When attacking  PSPNet on ADE20K with $l_\infty$-perturbation and $\epsilon=0.004$, CR-PGD has a relative $12.7\%$ gain over PGD. 
Across all settings, the average relative gain of CR-PGD over PGD is 13.9\%.

These results validate that pixel-wise certified radius can indeed guide CR-PGD to find better perturbation directions, and thus help better allocate the pixel perturbations. 
 To further verify it, we aim to show that more perturbations should be added to the pixels that easily affected others, such that more pixels are misclassified.  
We note that directly finding pixels that mostly easily affect others is a challenging combinatorial optimization problem. Hence we use an alternative way, i.e., show the distribution
of pixel perturbations vs. pixel-wise certified radius, 
to approximately show the effect. Our intuition is that: for pixels with small certified radii, perturbing them can also easily affect other neighboring pixels with small certified radii. This is because neighboring pixels with small certified radii can easily affect each other, which naturally forms the groups in the certified radius map (also see Figure 1(i)). Figure 5 verifies this intuition, where we test
10 random testing images in Pascal VOC. We can see that a majority of the perturbations in CR-PGD are assigned to the pixels with relatively smaller certified {radii}, 
in order to wrongly predict more pixels. In contrast, most of the perturbations in PGD are assigned to the pixels with relatively larger certified {radii}. 
As wrongly predicting these pixels requires a larger perturbation,   PGD misclassifies much fewer pixels than our CR-PGD.

\noindent {\bf Different models have different robustness.} 
Overall, HRNet is the most robust against the PGD and CR-PGD attacks, in that it has the smallest PixAcc drop when $\epsilon$ increases. On the other hand, PSPNet is the most vulnerable. We note that \cite{arnab2018robustness} has similar observations. 

\noindent {\bf Running time comparison.} 
Over all testing images in the three datasets, the average time of CR-PGD is 4.0 seconds, while that of PGD is 3.6 seconds. The overhead of CR-PGD over PGD is 11\%.

\begin{table}[!t]\renewcommand{\arraystretch}{0.8}
\footnotesize
\center
\caption{PixAcc with different $a$ and $b$ in $l_2$ CR-PGD.}
\label{tb:ab}
\begin{tabular}{c|c|c|c|c}
\hline
{\bf Dataset} &\diagbox{a}{b} &$b=-2$ & $b=-3$ &$b=-4$\\  \hline 
\multirow{3}{*}{\bf Pascal VOC} & $a=1$ & 33.6\% & 35.6\% & 38.1\% \cr
& $a=2$ & 28.3\% & 31.8\% & 30.9\%\cr
& $a=3$ & 32.4\% & 31.2\% & 30.7\%\cr
\cline{1-5}
\multirow{3}{*}{\bf CityScape} & $a=1$ & 54.1\% & 55.3\% & 56.5\%\cr
& $a=2$ & 52.7\% & 52.5\% & 50.7\%\cr
& $a=3$ & 54.8\% & 53.5\% & 52.6\%\cr
\cline{1-5}
\multirow{3}{*}{\bf ADE} & $a=1$ & 17.7\% & 18.6\% & 20.0\% \cr
& $a=2$ & 15.1\% & 16.2\% & 14.6\%\cr
& $a=3$ & 14.4\% & 15.0\% & 15.3\% \cr
\cline{1-5}
\end{tabular}
\vspace{-4mm}
\end{table}

\begin{figure*}[!t]
\centering
\subfigure[$l_1$ perturbation]{\includegraphics[width=0.3\textwidth]{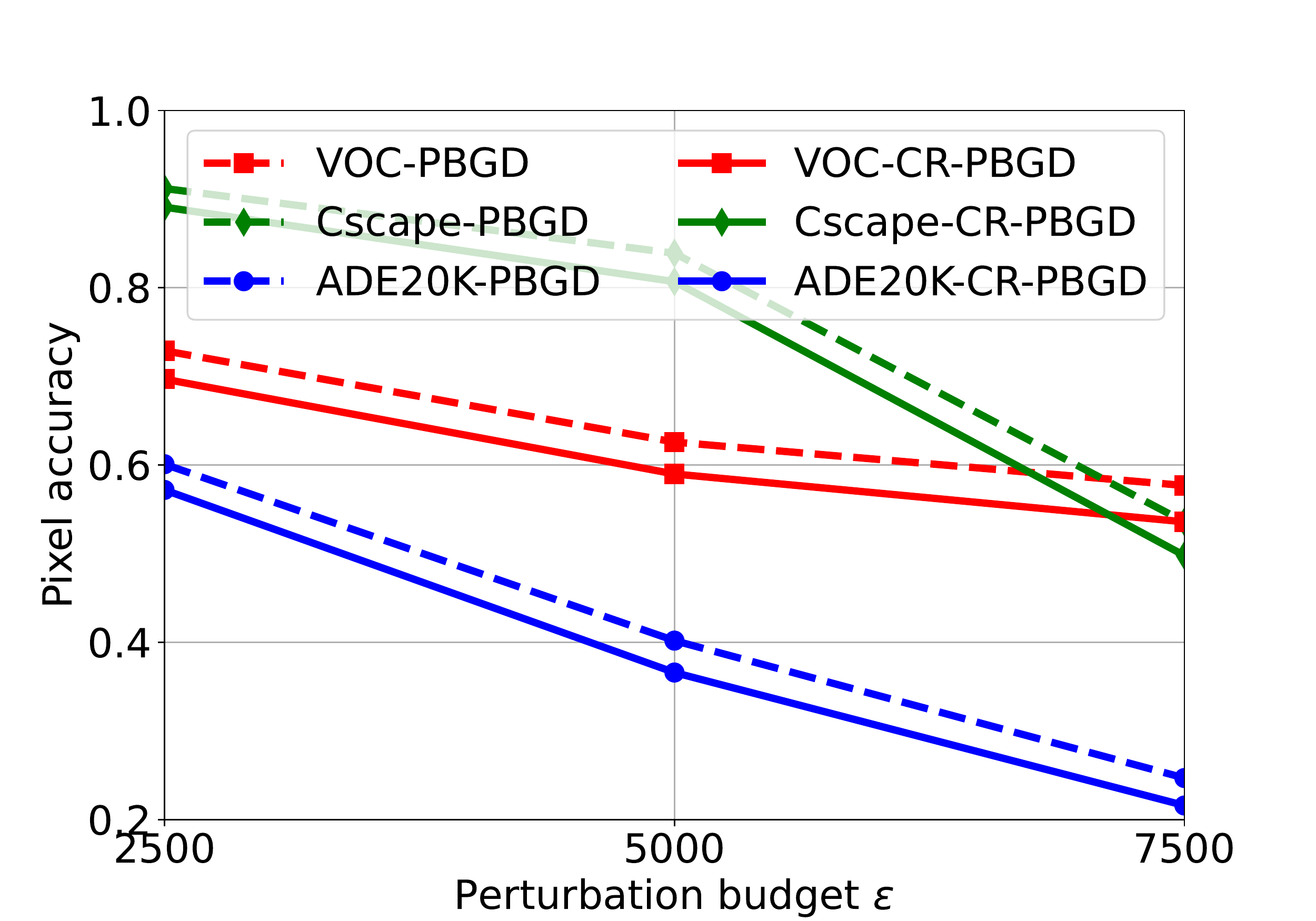}}
\subfigure[$l_2$ perturbation]{\includegraphics[width=0.3\textwidth]{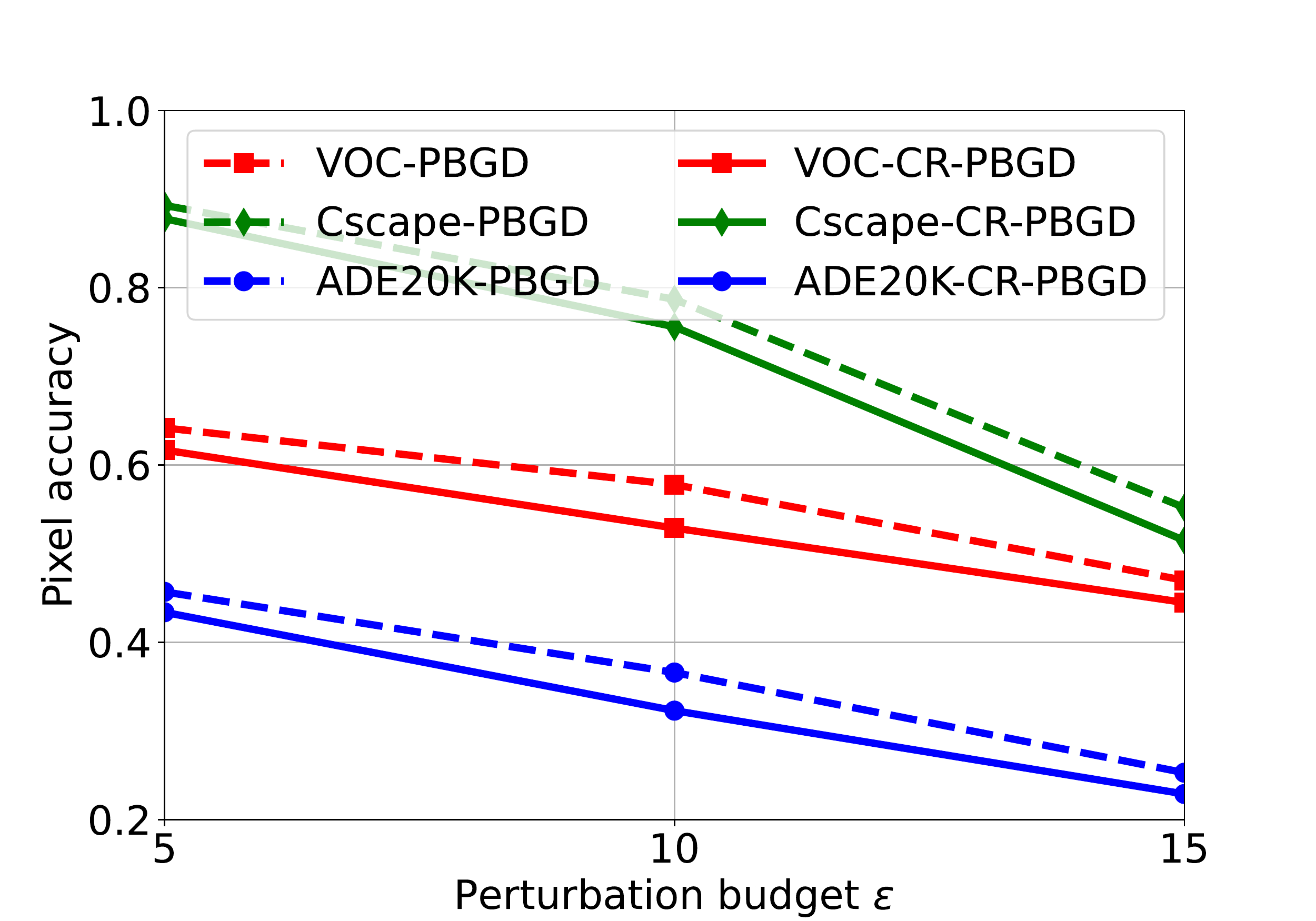}}
\subfigure[$l_\infty$ perturbation]{\includegraphics[width=0.3\textwidth]{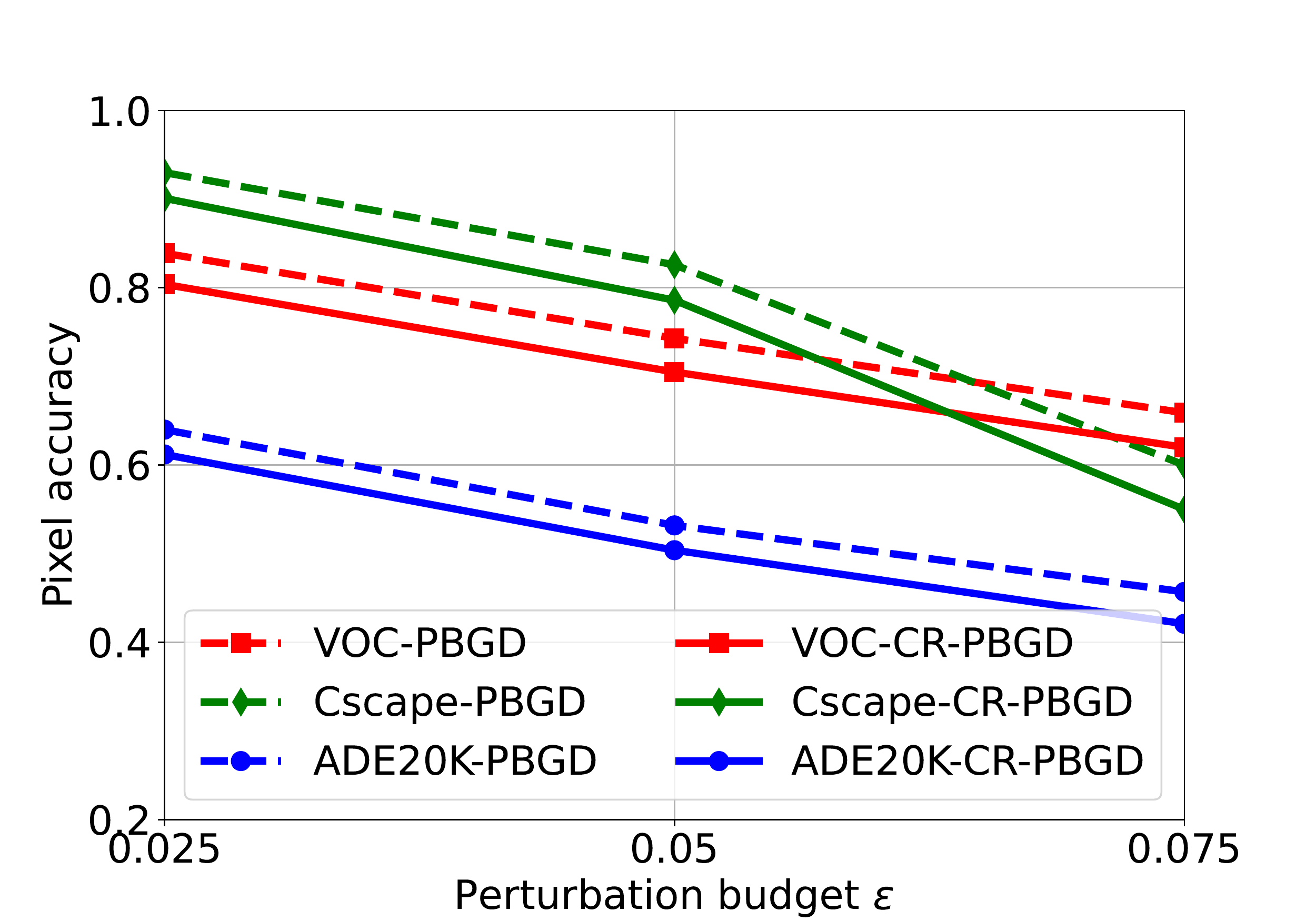}}
\vspace{-2mm}
\caption{{PixAcc after the black-box PBGD and CR-PBGD attacks with different $l_p$ perturbations vs. perturbation budget $\epsilon$.
}}
\label{fig:black_acc}
\vspace{-4mm}
\end{figure*}

\vspace{-2mm}
\subsubsection{Impact of the hyperparameters in CR-PGD} 
Pixel-wise certified radius and pixel weights are key components in our CR-PGD attack. 
Here, we study the impact of important hyperparameters in calculating the pixel-wise certified radius: number of samples $M$ 
in {Monte} Carlo sampling 
and  {Gaussian} noise $\sigma$, and $a$ and $b$ in calculating the pixel weights.  Figure~\ref{fig:white_param_acc_m} (and Figure~\ref{fig:white_param_miou_m} in Appendix) and Figure~\ref{fig:white_param_acc_s} (and Figure~\ref{fig:white_param_miou_s} in Appendix) show the impact of $M$ and $\sigma$ on our CR-PGD attack's pixel accuracy 
with $l_1$, $l_2$, and $l_\infty$ perturbation, respectively. 
Table~\ref{tb:ab} shows our CR-PGD attack's pixel accuracy with different $a$ and $b$. 
We observe that: (i) Our CR-PGD attack is not sensitive to $M$. The benefit of this is that an attacker can use a relatively small $M$ in order to save the attack time.  
{(ii) Our CR-PGD attack has stable performance within a range of small  $\sigma$}. 
Such an observation can guide an attacker to set a relatively small $\sigma$ when performing the CR-PGD attack. 
{More results on studying the impact of $\sigma$ w.r.t. $l_1, l_2, l_\infty$ perturbations on different models and datasets are shown in Figure
~\ref{fig:white_param_acc_sl1} to Figure~\ref{fig:white_param_acc_slinf} 
in Appendix~\ref{app:results}}. 
(iii)  Our CR-PGD attack is stable across different $a$ and $b$, and $a=2$ and $b=-4$ achieves the best tradeoff.

\begin{figure}[!t]
\centering
\vspace{-3mm}
\subfigure[Attack loss]{\includegraphics[width=0.235\textwidth]{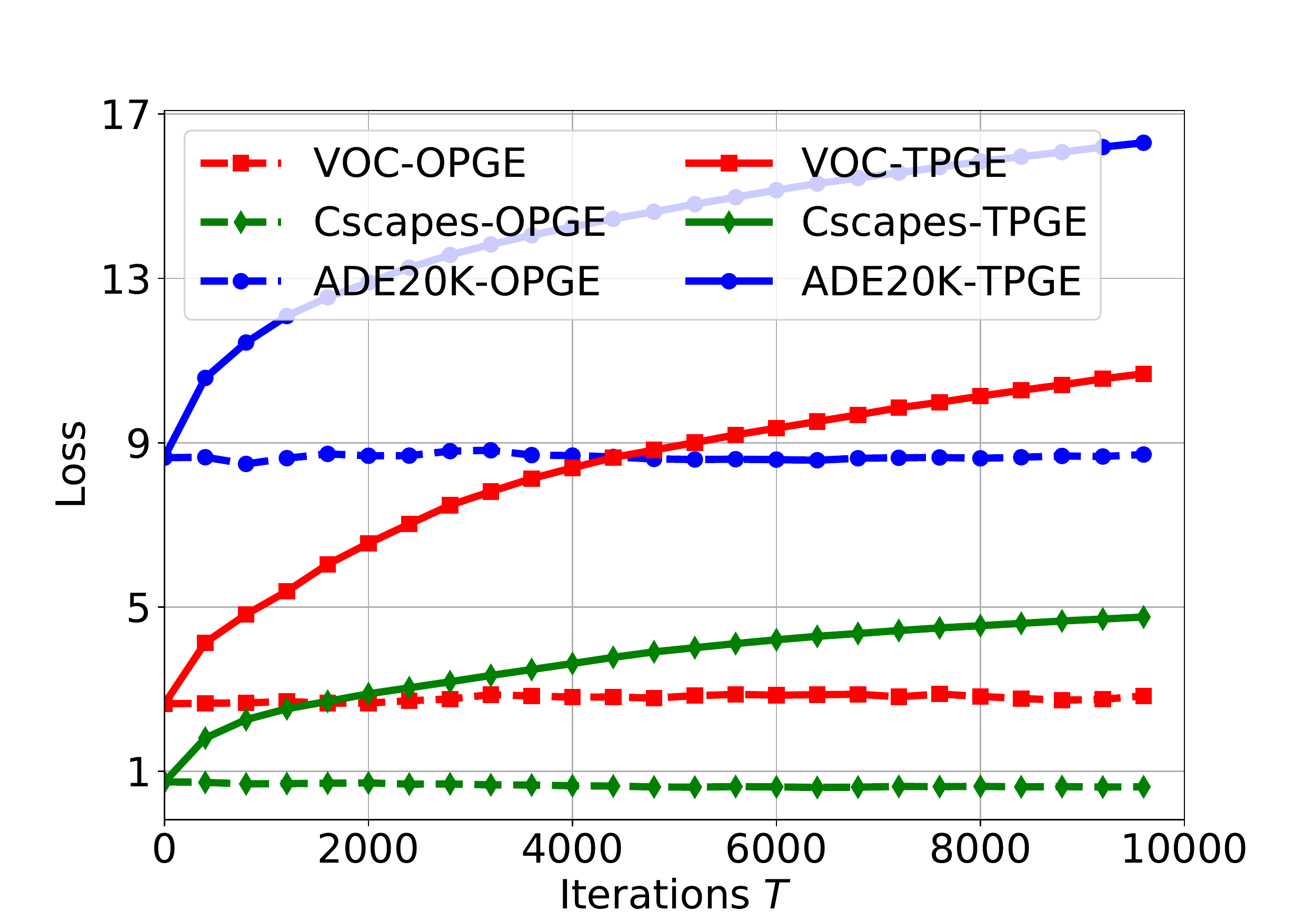}}
\subfigure[Attack performance]{\includegraphics[width=0.235\textwidth]{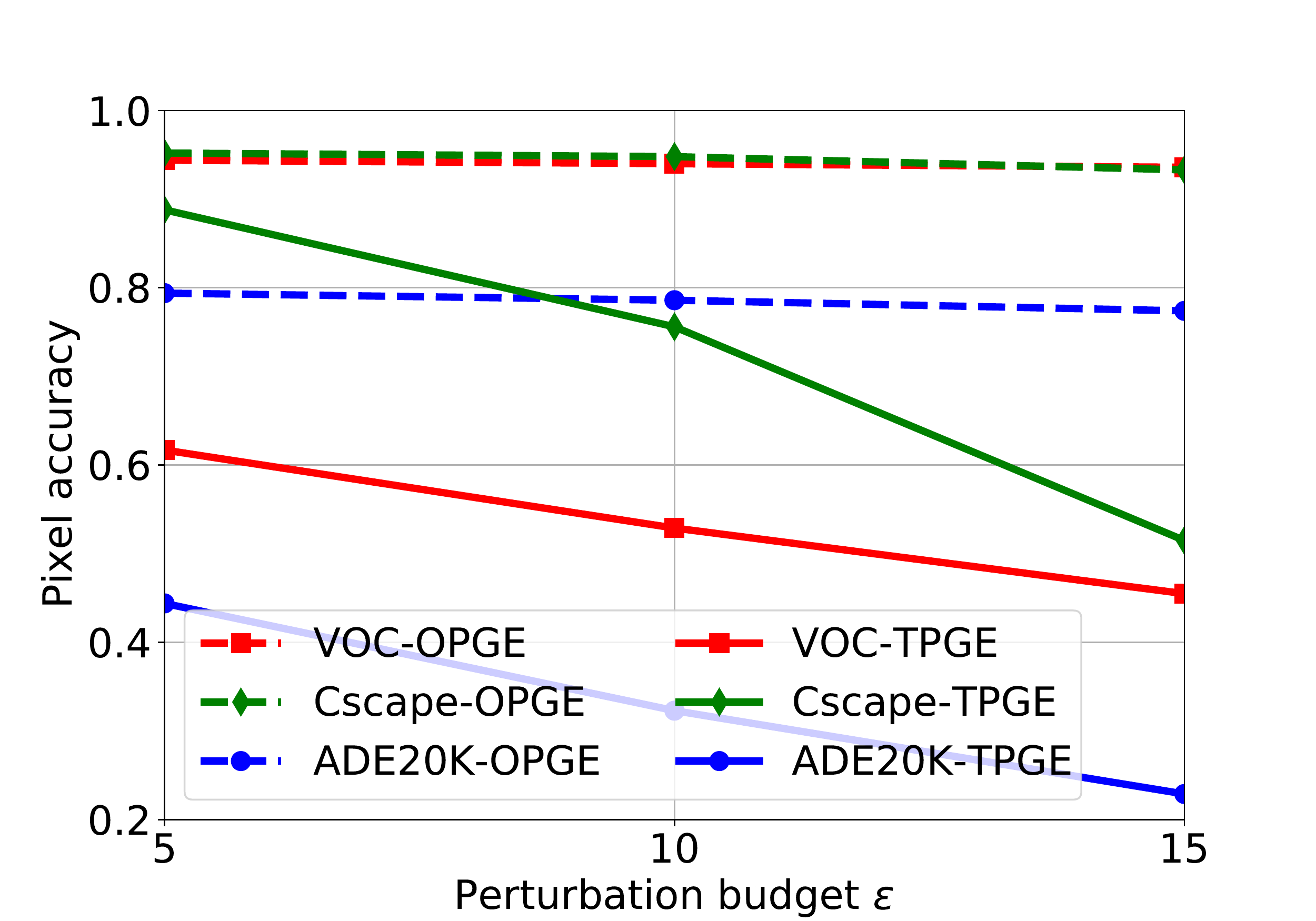}}
\vspace{-4mm}
\caption{{TPGE vs. OPGE. (a) Attack loss; (b) Attack performance.}} 
\label{fig:opge_vs_tpge}
\vspace{-4mm}
\end{figure}

\subsection{Results on Black-Box Attacks}
\label{sec:eval_bb}
\vspace{-2mm}
In this subsection, we show results on black-box attacks. We first
compare the gradient estimators proposed in Section~\ref{sec:bbattack}. Then, we compare our proposed PBGD and CR-PBGD attacks\footnote{
\cite{xie2017adversarial} studied the transferability between segmentation models. However, transferability-based attacks only have limited performance---the pixel accuracy only dropped 3-5\% after the attack.}. Finally, we study the impact of the number of queries, which is specific to black-box attacks. 
Note that 
we do not study the impact of other hyperparameters 
as our attacks are not sensitive to them, shown in the white-box attacks.

\vspace{-2mm}
\subsubsection{Comparing different gradient estimators}
In this experiment, we compare our two-point gradient estimator (TPGE) with the OPGE for simplicity.  We do not compare with the deterministic ZOO as it 
requires very huge number of queries, which is computationally intensive and impractical. 
We also note that stochastic NES~\cite{ilyas2018black} has very close performance as OPGE. For example, with an $l_2$ perturbation to be 10, we test on Pascal VOC with the PSPNet model and set \#populations to be 100. The PixAcc with NES is 77.4\%, which is close to OPGE's 76.9\%. 
For conciseness, 
we thus do not show NES's results.
Figure~\ref{fig:opge_vs_tpge} shows the results on attack loss and attack performance.
We observe in Figure~\ref{fig:opge_vs_tpge}(a) that 
the attack loss obtained by our TPGE 
stably increases, while obtained by the OPGE is unstable and sometimes decreases.
Moreover, as shown in Figure~\ref{fig:opge_vs_tpge}(b),  black-box attacks with TPGE achieve much better attack performance than with the OPGE---almost fails to work.  
The two observations validate that our TPGE outperforms the OPGE in estimating gradients and thus is more useful for attacking image segmentation models.

\subsubsection{Comparing PBGD with CR-PBGD}
Figure~\ref{fig:black_acc} shows the PixAcc with our PBGD and CR-PBGD attacks and different $l_p$ perturbations vs. perturbation budget $\epsilon$ on the three models and datasets. The MIoU results are shown in Appendix~\ref{app:results}. 
We have two observations. First, as 
$\epsilon$ increases, the PixAcc
decreases in all models and datasets with both CR-PBGD and PBGD attacks. Note that the PixAccs are larger that those achieved by white-box attacks (See Figures~\ref{fig:whitel1_acc}-\ref{fig:whitelinf_acc}). This is because white-box attacks use the exact gradients, while black-box attacks use the estimated ones. 
Second, CR-PBGD shows better attack performance than PBGD from two aspects: (i) All models have a smaller pixel accuracy with the CR-PBGD attack than that with the PBGD attack. (ii) The CR-PBGD attack can decrease the pixel accuracy more than the PBGD attack as $\epsilon$ increases.  
These results again demonstrate that the certified radius is beneficial to find more ``vulnerable" pixels to be perturbed.

\begin{figure}[!t]
\centering
\vspace{-3mm}
{\includegraphics[width=0.35\textwidth]{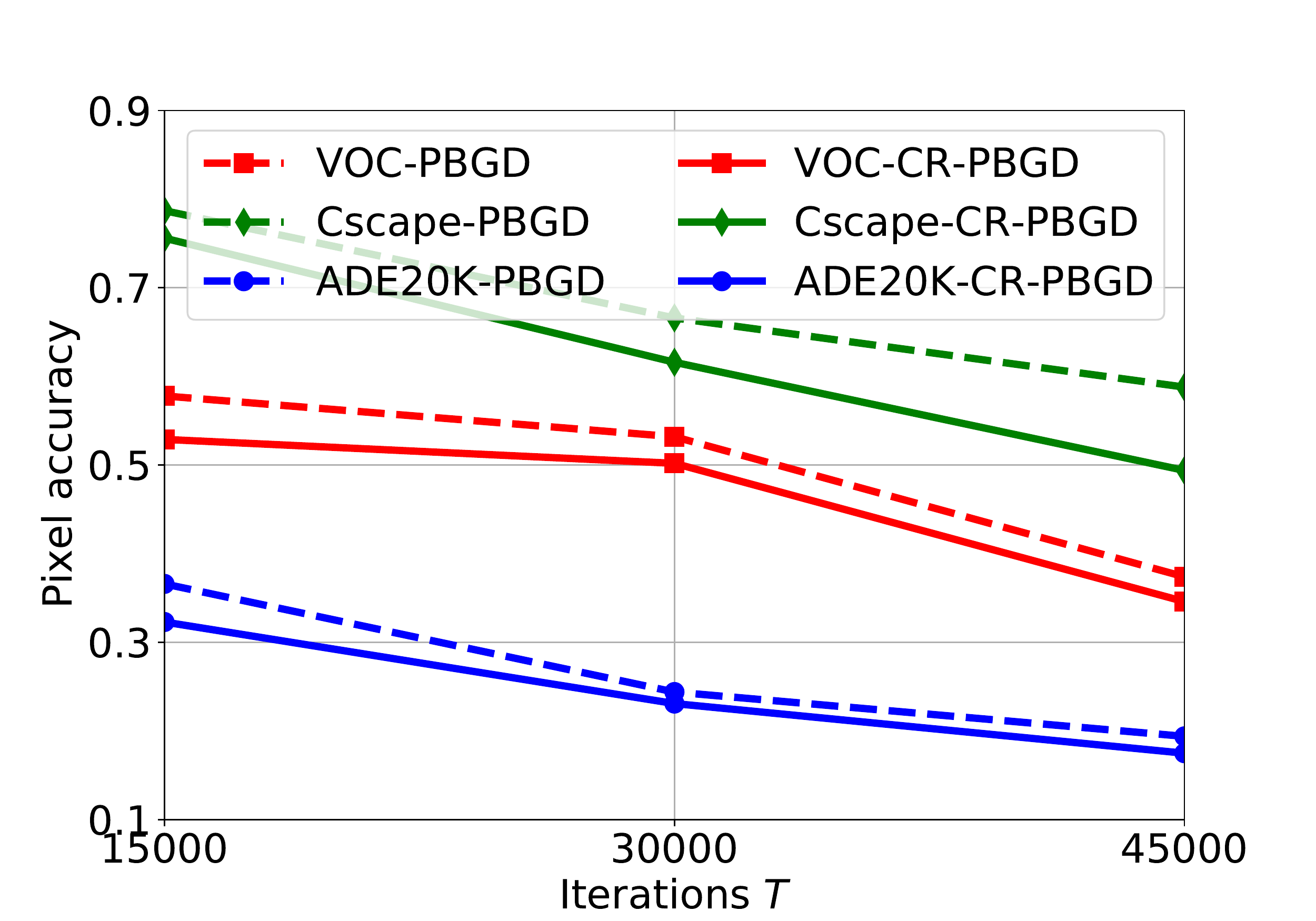}}
\vspace{-2mm}
\caption{Impact of the iterations $T$ on black-box PBGD and CR-PBGD attacks with $l_2$ perturbation on the three datasets.}
\label{fig:impact_query}
\vspace{-4mm}
\end{figure}

\begin{figure*}[!t]
\centering
\vspace{-2mm}
\subfigure[$l_1$ perturbation]{\includegraphics[width=0.28\textwidth]{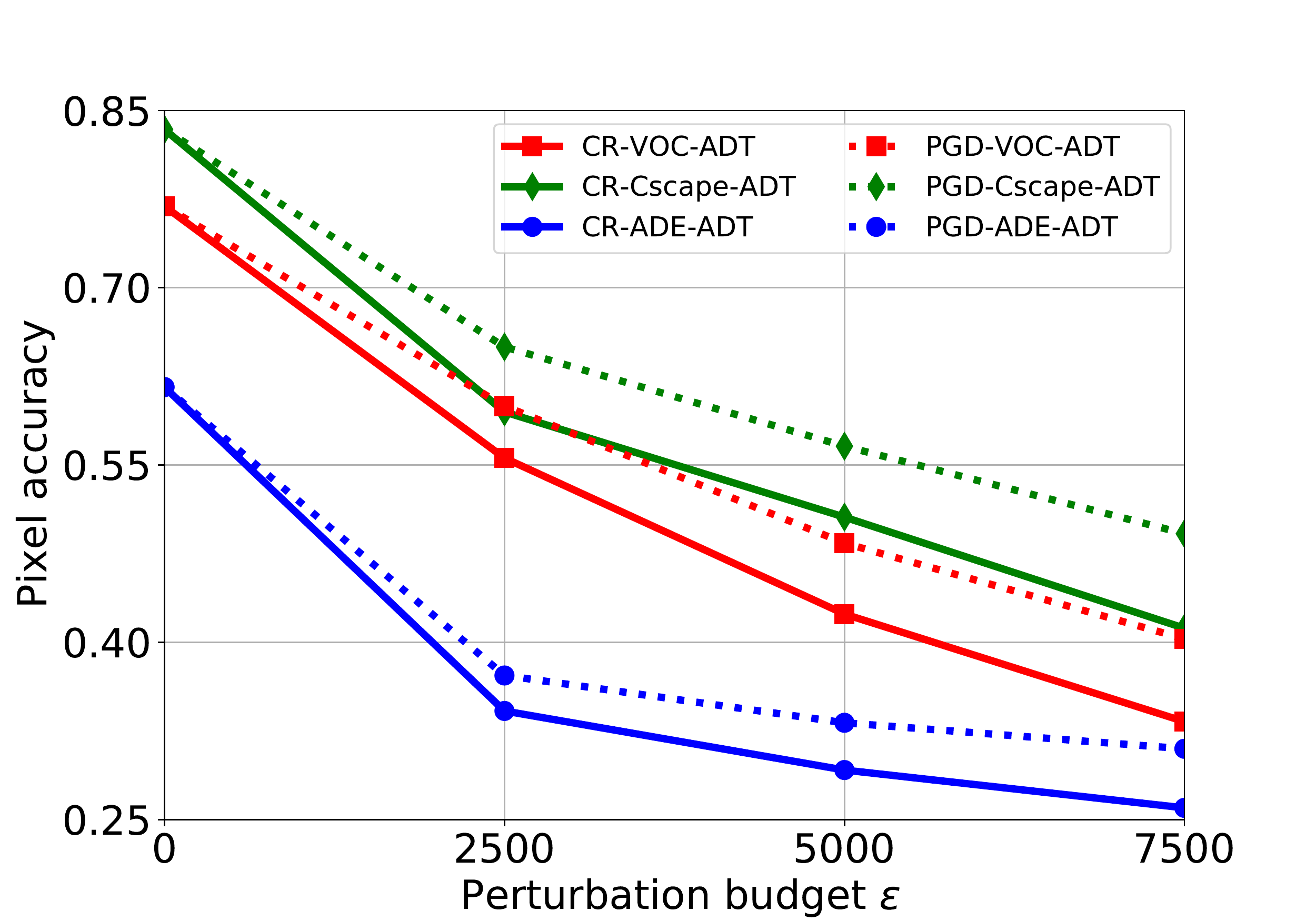}}
\subfigure[$l_2$ perturbation]{\includegraphics[width=0.28\textwidth]{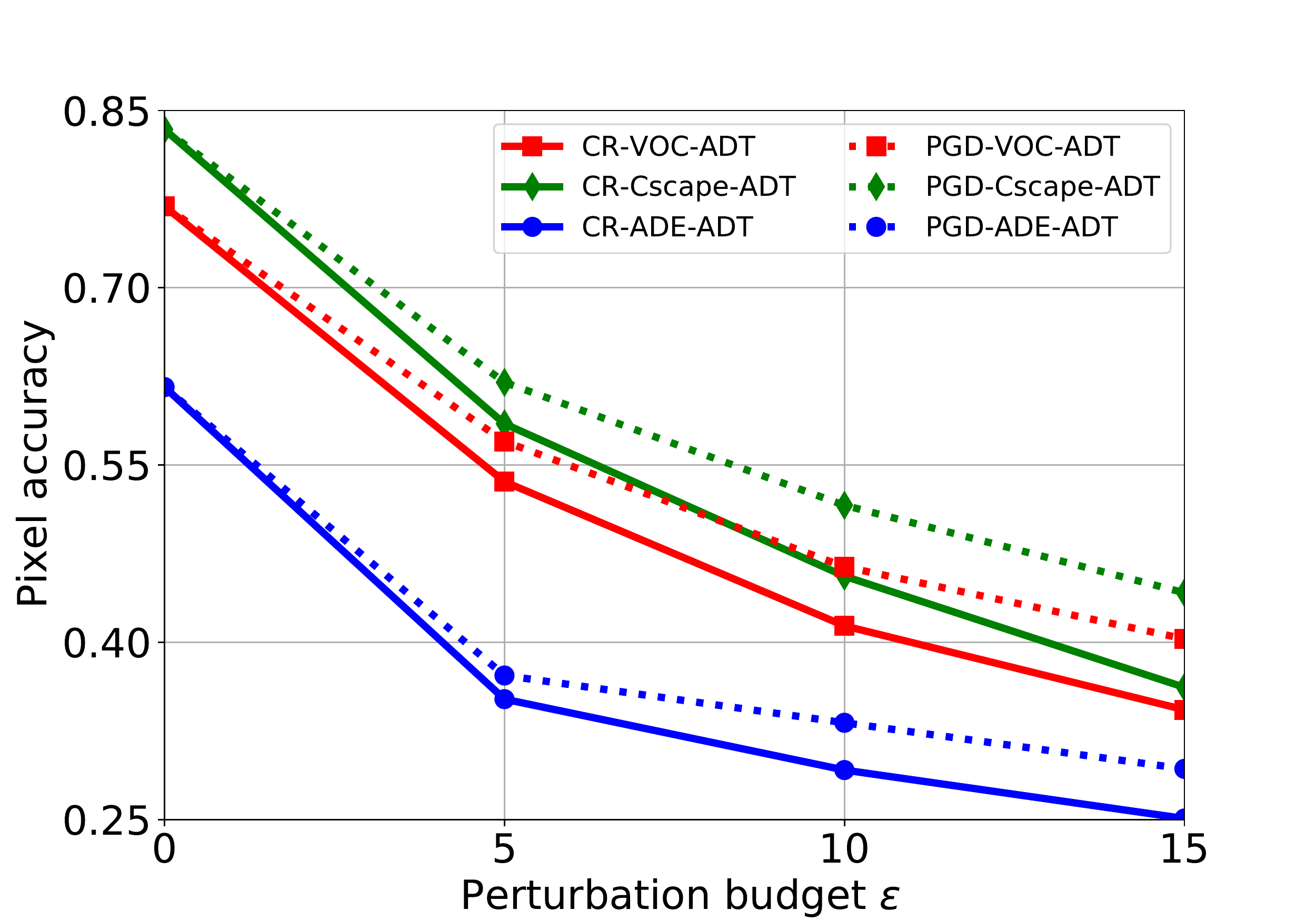}}
\subfigure[$l_\infty$ perturbation]{\includegraphics[width=0.28\textwidth]{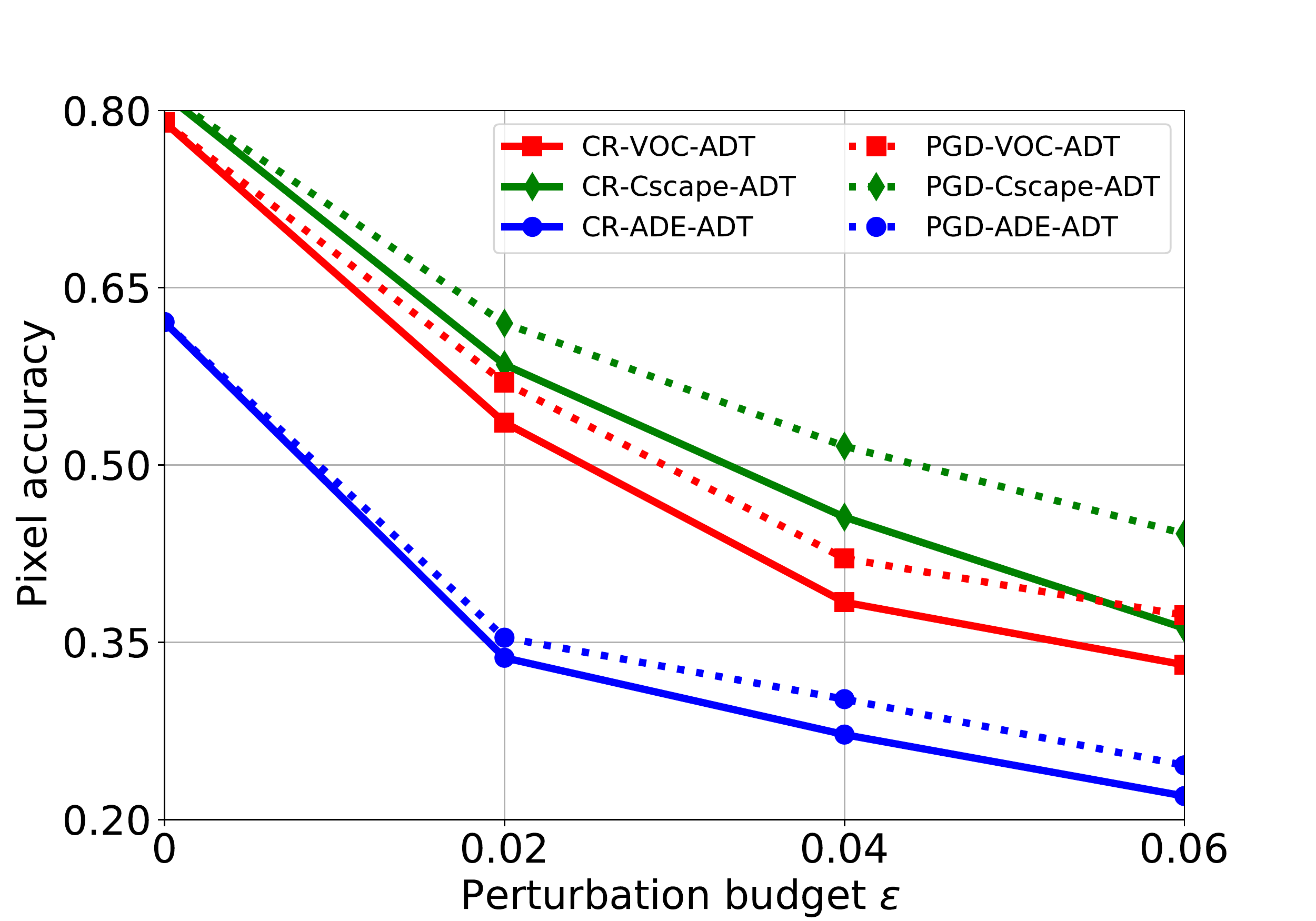}}
\vspace{-2mm}
\caption{{Defending against  white-box PGD and CR-PGD attack via fast adversarial training on the three models and datasets.}}
\label{fig:defense}
\vspace{-4mm}
\end{figure*}

\vspace{-2mm}
\subsubsection{Impact of the number of iterations/queries} 
Note that with $T$ iterations, our CR-PBGD attack has total queries  $2.5T$ (with $\textrm{INT}=2M$), while the PBGD attack has total queries $2T$. For a fair comparison, we set a same number of queries for PBGD,
i.e., we set its iteration number to be $1.25T$. For brevity, we still use a same notation $T$ when comparing them. 

Figure~\ref{fig:impact_query} shows the PixACC vs. $T$ against $l_2$ perturbation. 
Note that the results on $l_1$ and $l_\infty$ perturbations have similar tendency. 
Figure~\ref{fig:query_mio} in Appendix~\ref{app:results} shows the MIoU vs. $T$. 
We observe that  both the attacks perform better as $T$ increases. This is because a larger number of queries can reveal more information about the segmentation model to the attacker. 
Compared with PBGD, the CR-PBGD attack can decrease the pixel accuracy 
faster.

\subsection{Defenses}
\label{sec:eval_defense}

One defense is attempting to remove the certified radius information of all pixels, 
thus making the certified radius information no longer useful for attackers. However, this method will make the image segmentation model useless. Specifically, remember that a pixel's certified radius indicates the pixel’s \textit{intrinsic} confidence to be correctly predicted. To remove certified radius information, we must require the pixel's outputted confidence scores to be even across all labels, which means the image segmentation models' performance is random guessing. 

A second defense is to design robust image segmentation models 
using the existing defenses. Particularly, we choose the state-of-the-art 
empirical defense fast adversarial training (FastADT)~\cite{wong2020fast}, and 
certified defense SEGCERTIFY~\cite{fischer2021scalable}. 
To avoid the sense of false security~\cite{carlini2019evaluating}, we only 
defend against the white-box CR-PGD attack. 
We first compare FastADT and SEGCERTIFY. 
Note that SEGCERTIFY can only defend against the $l_2$ perturbation. By setting an $l_2$ perturbation budget as 10, the PixAccs with SEGCERTIFY are 21.9\%, 26.3\%, and 17.5\%  on the three datasets, respectively, 
while those with FastADT are 41.4\%, 45.6\%, and 29.2\%, respectively.  
These results show 
that FastADT significantly outperforms SEGCERTIFY. Next, we 
focus on adopting FastADT to defend against CR-PGD. 
Figure~\ref{fig:defense} and Figure~\ref{fig:defense_miou} in Appendix~\ref{app:results} show PixAcc and MIoU with FastADT vs. perturbation budget $\epsilon$ on the three models and datasets against CR-PGD, 
respectively. As a comparison, we also show defense results against the PGD attack. 
We observe that: (i) FastADT achieves an accuracy-robustness tradeoff, i.e., the clean (undefended) pixel accuracy decreases at the cost of maintaining the robust accuracy against the attacks.
(ii) CR-PGD is more effective than PGD against FastADT, which again  verified that pixel-wise certified radius is important to design better attacks.

A third defense is to enhance the existing strongest defenses, i.e., FastADT.  
The current FastADT is trained on adversarial perturbations generated by the state-of-the-art PGD attack and then used to defend against our CR-PGD attack. Here, we propose a variant of FastADT, called CR-FastADT, which assumes the defender knows the details of our CR-PGD attack, and trains on adversarial perturbations generated by our CR-PGD. For instance, we evaluate CR-FastADT on Pascal VOC with PSPNet.  When setting the $L_2$ perturbation  to be 5, 10, 15, the PixAccs with CR-FastADT are 59.8\%, 47.9\%, and 38.7\% respectively, which are 
marginally (1.3\%, 2.3\%, and 2.5\%) higher than those with FastADT, i.e., 58.5\%, 45.6\%, and 36.2\%. This verifies that CR-FastADT is a slightly better defense than FastADT, but  not that much.

\begin{table}[!t]\renewcommand{\arraystretch}{0.85}
\caption{Comparing FGSM and CR-FGSM.}
\footnotesize
\addtolength{\tabcolsep}{-3pt}
\begin{tabular}{c|ccc|ccc}
\hline
\multirow{2}{*}{\bf Model} & \multicolumn{3}{c|}{\bf FGSM}  & \multicolumn{3}{c}{\bf CR-FGSM}                                                 \\ \cline{2-7} 
                  & \multicolumn{1}{c|}{$l_1$} & \multicolumn{1}{c|}{$l_2$} & \multicolumn{1}{c|}{$l_\infty$} & \multicolumn{1}{c|}{$l_1$} & \multicolumn{1}{c|}{$l_2$} & \multicolumn{1}{c}{$l_\infty$} \\ \hline \hline
     {\bf PSPNet-Pascal}             & \multicolumn{1}{c|}{$79.8\%$} & \multicolumn{1}{c|}{$80.7\%$} & $74.7\%$                       & \multicolumn{1}{c|}{$78.7\%$} & \multicolumn{1}{c|}{$79.1\%$} &    $73.2\%$                   \\ \hline
       {\bf HRNet-Cityscape}            & \multicolumn{1}{c|}{$37.7\%$} & \multicolumn{1}{c|}{$63.0\%$} &              $17.0\%$         & \multicolumn{1}{c|}{$36.1\%$} & \multicolumn{1}{c|}{$61.6\%$} & $15.6\%$  \\ \hline
       {\bf PSANet-ADE20K}              & \multicolumn{1}{c|}{$59.4\%$} & \multicolumn{1}{c|}{$59.9\%$} &$54.4\%$                       & \multicolumn{1}{c|}{$57.9\%$} & \multicolumn{1}{c|}{$58.6\%$} &  $53.4\%$                     \\ \hline
\end{tabular}
\label{tbl:FGSM}
\vspace{-2mm}
\end{table}

\section{Discussion}
\label{sec:discussion}
\vspace{-3mm}

\vspace{+0.5mm}
\noindent {\bf Applying our certified radius-guided attack framework to other attacks.} 
In the paper, we mainly apply the pixel-wise certified radius in the 
PGD attack. Actually, it can be also applied in other attacks such as the FGSM attack~\cite{arnab2018robustness}
(Details are shown in the Appendix~\ref{supp:methods}) and
enhance its attack performance as well. 
For instance, Table~\ref{tbl:FGSM} shows the pixel accuracy with FGSM and FGSM with certified radius (CR-FGSM) on the three image segmentation models and datasets.
We set all parameters same as PGD/CR-PGD. For instance,  
the perturbation budget $\epsilon$ 
is $\epsilon=750,1.5,0.006$ for $l_1$, $l_2$, and $l_\infty$  perturbations, respectively. 
We observe that CR-FGSM consistently produces a lower pixel accuracy than FGSM in all cases, showing that certified radius can also guide FGSM to achieve better attack performance.

\vspace{+0.5mm}
\noindent {{\bf Randomized smoothing vs. approximate local Lipschitz based certified radius.}
Weng et al.~\cite{wengevaluating} proposed CLEVER, which uses sampling to estimate the local Lipschitz constant, then derives the certified radius for image classification models. CLEVER has two sampling relevant hyperparameters $N_b$ and $N_s$. In the untargeted attack case (same as our setting), CLEVER requires sampling the model $N_b \times N_s \times|\mathcal{Y}|$ times to derive the certified radius, where $|\mathcal{Y}|$ is the number of classes. The default value of $N_b$ and $N_s$ are 50 and 1024, respectively. We can also adapt CLEVER to derive pixel's certified radius for segmentation models and use it to guide PGD attack.  

With these defaults values, we test that CLEVER is 4 orders of magnitude slower than randomized smoothing. We then reduce $N_b$ and $N_s$. 
When $N_b=10$ and  $N_s=32$, CLEVER is 2 orders of magnitude slower, and its attack effectiveness is less effective than ours. E.g., on Pascal VOC dataset and PSPNet model, with $l_2$ perturbation budget $1.5$, it achieves an attack accuracy $35.4\%$, while our CR-PGD attack achieves $30.9\%$. 
In the extreme case, we set $N_b=10$ and $N_s=1$ and CLEVER is still 1 order of magnitude slower, and it only achieves $41.6\%$ attack accuracy, even worse than vanilla PGD attack which achieves  $39.7\%$. This is because the calculated pixels' certified radii are very inaccurate when $N_s=1$. 
}

\vspace{+0.5mm}
\noindent {\bf Evaluating our attack framework for target attacks.}
Our attack framework is mainly for untargeted attacks, i.e., it aims to misclassify as many pixels as possible, while not requiring what wrong labels to be. This is actually due to the inherent properties of certified radius---if a pixel has a 
small certified radius, 
this pixel is easily misclassified to be \textit{any} wrong label, but not a specific one. 
On the other hand, targeted attacks misclassify pixels to be specific wrong labels. 
Nevertheless, we can still adapt our attack framework to perform the targeted attack, where we replace the existing attack loss for untargeted attacks to be that for target attacks. 
We experiment on Pascal VOC and PSPNet with a random target label and 500 testing images and set $l_2$ perturbation budget as 10. We show that 95.3\% of pixels are misclassified to be the target label using our adapted attack, and 96.5\% of pixels  are misclassified to be the target label using 
the state-of-the-art target attack~\cite{ozbulak2019impact}.
This result means our attack is still effective and achieves comparable results with specially designed targeted attacks.

\vspace{+0.5mm}
\noindent {\bf Comparing white-box CR-PGD vs. black-box CR-BPGD.} 
We directly compare our white-box CR-PGD and black-box CR-BPGD in the same setting, and consider $l_2$ perturbation. Specifically, by setting the perturbation budget $\epsilon=5, 10, 15$, pixel accuracies with CR-PGD are 17.9\%, 13.7\%, and 10.4\%, respectively, while that with CR-BPGD are 61.7\%, 52.9\%, and 45.5\%, respectively. 
The results show that white-box attacks are much more effective than black-box attacks and thus there is still room to improve the performance of black-box attacks.

\vspace{+0.5mm}
\noindent {\bf Adversarial training with CR-PGD samples}. We evaluate CR-PGD for adversarial training and recalculate the CR for the pixels with the (e.g., 20\%) least CR. Using CR-PGD for adversarial training increases the CR of these pixels by 11\%. This shows that CR-PGD can also increase the robustness of ``easily perturbed'' pixels.

\vspace{+0.5mm}
\noindent {\bf Defending against black-box attacks.} 
We use FastADT to defend against our black-box CR-PBGD attack. 
We evaluate FastADT on Pascal VOC with the PSPNet model and set $l_2$ perturbation  to be 10. 
We note that the PixAcc with no defense is 52.9\%, but can be largely increased to 71.5\% when applying FastADT.  As emphasized, 
a defender cares more on defending against the ``strongest" white-box attack, as defending against the ``weakest" black-box attack does not mean the defense is effective enough in practice. This is because an attacker can always leverage better techniques to enhance the attack.

\vspace{+0.5mm}
\noindent {\bf Applying our attacks in the real world and the difficulty.} Let's 
consider tumor detection and traffic sign prediction 
systems. An insider in an insurance company can be a white-box attacker, i.e., s/he knows the tumor detection
algorithm details. S/He can then modify her/his medical images to attack the algorithm offline by using our white-box attack and submitting modified images for fraud insurance claims. An outsider that uses a deployed traffic sign prediction system can be a black-box attacker. For example, s/he can iteratively query the system with perturbed ``STOP'' sign images 
and optimize the perturbation via our black-box attacks. The attack ends when the final perturbed ``STOP'' sign  
is classified, e.g., as a ``SPEED" sign. S/He finally prints this perturbed ``STOP'' sign for physical-world attacks. The difficulties lie in that this attack may involve many queries, but note that there indeed exist physical-world attacks using the perturbed ``STOP" sign~\cite{eykholt2018robust}.

\section{Related Work}
\label{sec:related}

\vspace{-2mm}
\noindent {\bf Attacks.} 
A few white-box attacks~\cite{xie2017adversarial,fischer2017adversarial,hendrik2017universal,arnab2018robustness} have been proposed on image segmentation models.   For instance, Xie et al.~\cite{xie2017adversarial} proposed the first gradient based attack to 
segmentation models. 
Motivated by that pixels are separately classified in image segmentation, they developed the Dense Adversary Generation (DAG) attack 
that considers all pixels together and optimizes the summation of all pixels' losses to generate adversarial perturbations.
Arnab et al.~\cite{arnab2018robustness} presented the first systematic evaluation of adversarial perturbations on modern image segmentation models. They adopted the FGSM~\cite{goodfellow2014explaining,kurakin2016adversarial} and PGD~\cite{madry2018towards} as the baseline attack. 
 Existing black-box attacks, e.g., NES~\cite{ilyas2018black}, SimBA~\cite{guo2019simple} and ZOO~\cite{chen2017zoo}, to image classification models can be adapted to attack image segmentation models. However, they are very query inefficient.

\vspace{+0.5mm}
\noindent {\bf Defenses.}
A few defenses~\cite{xiao2018characterizing,he2019non,bar2019robustness,mao2020multitask,xu2021dynamic,fischer2021scalable} have been proposed recently to improve the robustness of segmentation models against adversarial perturbations. For instance,  Xiao et al.~\cite{xiao2018characterizing} first characterized adversarial perturbations based on spatial context information in segmentation models and then proposed to detect adversarial regions using spatial consistency information. 
Xu et al.~\cite{xu2021dynamic} proposed a dynamic divide-and-conquer adversarial training method. 
Latterly, Fischer et al.~\cite{fischer2021scalable} adopted randomized smoothing to develop the {first} certified segmentation. 
Note that we also use  randomized smoothing in this paper, but  
our goal is not to use it to defend against attacks.

\vspace{+0.5mm}
\noindent {\bf Certified radius.} 
Various methods~\cite{scheibler2015towards,carlini2017provably,ehlers2017formal,katz2017reluplex,wong2017provable,wong2018scaling,raghunathan2018certified,raghunathan2018semidefinite,cheng2017maximum,fischetti2018deep,bunel2018unified,dvijotham2018dual,gehr2018ai2,mirman2018differentiable,singh2018fast,weng2018towards,zhang2018efficient}
have been proposed to derive the certified radius for image classification models against adversarial
perturbations.
However, these methods are not scalable. 
Randomized smoothing~\cite{lecuyer2018certified,li2018second,cohen2019certified,lee2019stratified,zhai2020macer,levine2020robustness,yang2020randomized,kumar2020curse,mohapatra2020higher,wang2020certifying,jia2020certifiedcommunity,wang2021certified,jia2020certified,jia2022almost,hong2022unicr} was the first method to certify the robustness of large 
models and achieved the state-of-the-art certified radius. For instance, Cohen et al.~\cite{cohen2019certified} obtained a tight $l_2$ certified radius 
with Gaussian noise on normally trained image classification models. 
Salman et al.~\cite{salman2019provably} improved 
\cite{cohen2019certified} by combining the design of an adaptive attack against smoothed soft image classifiers and adversarial training on the attacked classifiers. 
Many follow-up works~\cite{jia2020certified,li2021tss,fischer2020certified,wang2021certified,jia2020certifiedcommunity,bojchevski2020efficient,chiang2020detection} extended randomized smoothing in various ways and applications. For instance, 
Fischer et al.~\cite{fischer2020certified} and Li et al.~\cite{li2021tss} proposed to 
certify the robustness against geometric perturbations. 
Chiang et al.~\cite{chiang2020detection} introduced median smoothing and applied it to certify object detectors. 
Jia et al.~\cite{jia2020certifiedcommunity} and Wang et al.~\cite{wang2021certified}
applied randomized smoothing in the graph domain and derived the certified radius 
for community detection, node/graph classifications methods against graph structure perturbation.

In contrast, we propose to leverage certified radius 
derived by randomized smoothing 
to design better attacks against image segmentation models\footnote{Concurrently, we note that \cite{wang2023turning} also proposed more effective attacks to graph neural networks based on certified robustness.}.

\section{Conclusion}
\label{sec:conclusion}

\vspace{-2mm}
We study attacks to image segmentation models. 
Compared with image classification models, image segmentation models have 
richer information (e.g., predictions on each pixel instead of on an whole image) 
that can be exploited by attackers. 
We propose to leverage {certified radius}, the first work that uses it from the attacker perspective, and derive the pixel-wise certified radius based on randomized smoothing. 
Based on it, we design a {certified radius}-guide attack framework  for both white-box and black-box attacks. Our framework can be seamlessly incorporated into any existing attacks to design more effective attacks.
Under black-box attacks, 
we also design a random-free gradient estimator based on bandit:
it is query-efficient, unbiased and stable. 
We use our gradient estimator to instantiate PBGD and certified radius-guided PBGD attacks, both with a tight sublinear regret bound.    
Extensive evaluations verify the effectiveness and generalizability of our certified radius-guided attack framework.

\vspace{+0.5mm}
\noindent{\bf Acknowledgments.} 
We thank the anonymous reviewers for their constructive feedback. 
This work was supported by Wang's startup funding, Cisco Research Award, and National Science Foundation under grant No. 2216926.

\bibliographystyle{IEEEtran}
\bibliography{refs}

\appendices

\section{Appendices}

\begin{algorithm}[h]
\small
\caption{\small Certified radius-guided white-box PGD attacks to image segmentation models}
\label{alg:cr_PGD}
\LinesNumbered
\KwIn{Image segmentation model $F_\theta$, testing image $x$ with pixel labels $y$, perturbation budget $\epsilon$, learning rate $\alpha$, \#iterations $T$, weight parameters $a$ and $b$, \#samples $M$, Gaussian variance $\sigma$, interval INT.}
\KwOut{Adversarial perturbation $\delta^{(T)}$.}
Initialize: $\delta^{(0)} = 0 \in \mathbb{B} = \{ \delta:  \|\delta\|_p \leq \epsilon \}$. \\
\For{$t=0,1, \cdots,T-1$}{
	\eIf {t \textrm{mod} INT != 0}
	{ 
	Reuse the pixel weights:
	}
	{
	Sample $M$ noises: $\beta_j \sim \mathcal{N}(0,\sigma^2 I), j\in\{1,2,...,M\}$; \\
Define the perturbed image: $x^{(t)} = x + \delta^{(t)}$; \\
Compute the smoothed segmentation model: $G(x^{(t)}) = \frac{1}{M} \sum_{j=1}^M F_\theta(x^{(t)} +\beta_j)$; \\
Estimate the lower bound probabilities of the top label for each pixel $x^{(t)}_n$:
$ \underline{p_{n}} = \max_c {G(x^{(t)})_{n,c}}$; \\
	Calculate the certified radius for each pixel $x^{(t)}_n$: 
$cr(x^{(t)}_n) =\sigma \Phi^{-1}(\underline{p_{n}})$; \\
Assign a weight to each pixel $x^{(t)}_n$: 
	}

Define the certified radius-guided loss:
$L_{cr}(F_\theta(x^{(t)}),y) = \frac{1}{N} \sum_{n=1}^N w^{(t)}_n \cdot L(F_\theta(x_n^{(t)}), y_n) $; \\
	Run CR-PGD to update the adversarial perturbation: $\delta^{(t+1)}=\textrm{Proj}_{\mathbb{B}} (\delta^{(t)} + \alpha \cdot \nabla_{\delta^{(t)}} L_{\textrm{cr}}(F_\theta(x^{(t)}),y))$.
}
\Return{$\delta^{(T)}$}
\end{algorithm}

\setlength{\textfloatsep}{1mm}

\begin{algorithm}[!th]
\small
\caption{\small Black-box PBGD and CR-PBGD to image segmentation models}
\label{alg:BGD_attack}
\LinesNumbered
\KwIn{Image segmentation model $F_\theta$, testing image $x$ with pixel labels $y$, perturbation budget $\epsilon$, learning rate $\alpha$, parameter $\gamma$, \#iters $T$, {weight parameters $a$ and $b$, \#samples $M$, Gaussian variance $\sigma$}, CRFlag, interval INT.}
\KwOut{Adversarial perturbation $\delta^{(T)}$.}
Initialize: $\delta^{(0)} = 0 \in \mathbb{B} = \{ \delta:  \|\delta\|_p \leq \epsilon \}$. \\
\For{$t=0,1,\cdots,T-1$}{
Sample a unit vector $\bm{u}^{(t)} \in \mathbb{S}_p$ uniformly  at random\;
    Query 
    $F_\theta$ twice to obtain predictions: $p_1 = F_\theta(x+\delta^{(t)}+\gamma \bm{u}^{(t)})$ and $p_2 = F_\theta(x+\delta^{(t)}-\gamma \bm{u}^{(t)})$\;
    	
	\eIf { (CRFlag == false) }
	{ 
	/*{\bf PBGD: Without pixel-wise certified radius}*/ \\
	Obtain the loss feedback:\\
	$\ell_1 = L(p_1,y)$ and 
	$\ell_2 = L(p_2,y)$;
	
	Estimate the gradient: $\tilde{g} = \frac{N}{2\gamma}\big( \ell_1-\ell_2 \big)\bm{u}^{(t)}$;
	
	Run PBGD to update the adversarial perturbation: 
	$\delta^{(t+1)}=\textrm{Proj}_{\mathbb{B}} (\delta^{(t)} + \alpha \cdot \tilde{g})$. 
	}
	{
	{/*{\bf CR-PBGD: With pixel-wise certified radius}*/} \\

	Calculate the pixel-wise certified radius and define $L_{cr}$ using $F_\theta, a,b, M, \sigma, t$ as in Algorithm~\ref{alg:cr_PGD};
	
	Obtain the certified radius-guided loss feedback:\\
	$\ell_{cr,1} = L_{cr}(p_1,y)$ and  
	$\ell_{cr,2} = L_{cr}(p_2,y)$;
	
    Estimate the gradient: $\tilde{g}_{cr} = \frac{N}{2\gamma}\big( \ell_{cr,1}-\ell_{cr_2} \big)\bm{u}^{(t)}$;
    
    Run CR-PBGD to update the adversarial perturbation: 
	$\delta^{(t+1)}=\textrm{Proj}_{\mathbb{B}} (\delta^{(t)} + \alpha \cdot \tilde{g}_{cr})$.
	}
}
\Return{$\delta^{(T)}$}
\end{algorithm} 
 
\setlength{\textfloatsep}{1mm}

 \begin{algorithm}[!t]
\small
\caption{DAG for $l_\infty$ perturbation}
\label{alg:DAG}
\LinesNumbered
\KwIn{Segmentation model $F_\theta$,  testing image $x$ with label $y$, 
\#epochs $T$, perturbation budget $\epsilon$.}
\KwOut{Adversarial perturbation $\delta$.}
Initialize: $x^{(1)} \leftarrow x$, $\delta^{(0)} = 0$. \\
\For{$t=1,2,\cdots,T$}{

$\delta^{(t)} \leftarrow \nabla_{\delta^{(t)}} F_\theta(x^{(t)} + \delta^{(t-1)},y) - \nabla_{\delta^{(t)}} F_\theta(x^{(t)}, y)$ \\
 
$\delta^{(t)} = \frac{0.5 \delta^{(t)}}{||\delta^{(t)}||_\infty}$ \\

$\delta = \textrm{Clip}(\delta + \delta^{(t)}, \epsilon)$ \\
$x^{(t+1)} = x^{(t)} + \delta^{(t)}$ \\
}
\Return{$\delta$}
\end{algorithm}

\subsection{Dataset Details} 
\label{app:datasets}
Detailed of the used datasets are as below:

\begin{itemize}[leftmargin=*]
    \item {\bf Pascal VOC~\cite{pascal-voc-2012}.} It 
consists of internet images labeled with 21 different classes. 
The training set has 10582 images when combined with additional annotations from ~\cite{SBD}.
The validation set contains 1449 images.

\item {\bf Cityscapes~\cite{cordts2016cityscapes}.} It consists
of road-scenes captured from car-mounted cameras and has
19 classes. The training set totals 2975 images and the validation set has 500 images. As the dataset contains
high-resolution images (2048×1024 pixels) that require
too much memory for segmentation models, we resize all images to 1024x512 before training, same as~\cite{arnab2018robustness}. 

\item {\bf ADE20K~\cite{zhou2019semantic}.} It is a densely annotated image dataset that covers 365 scenes and 150 object categories with pixel-wise
annotations for scene understanding. 
The dataset contains 27,574 images, where 25,574 images form the training set and the remaining 2,000 images as the testing set. 

\end{itemize}

\begin{figure*}[!t]
\centering
\subfigure[Pascal VOC]{\includegraphics[width=0.32\textwidth]{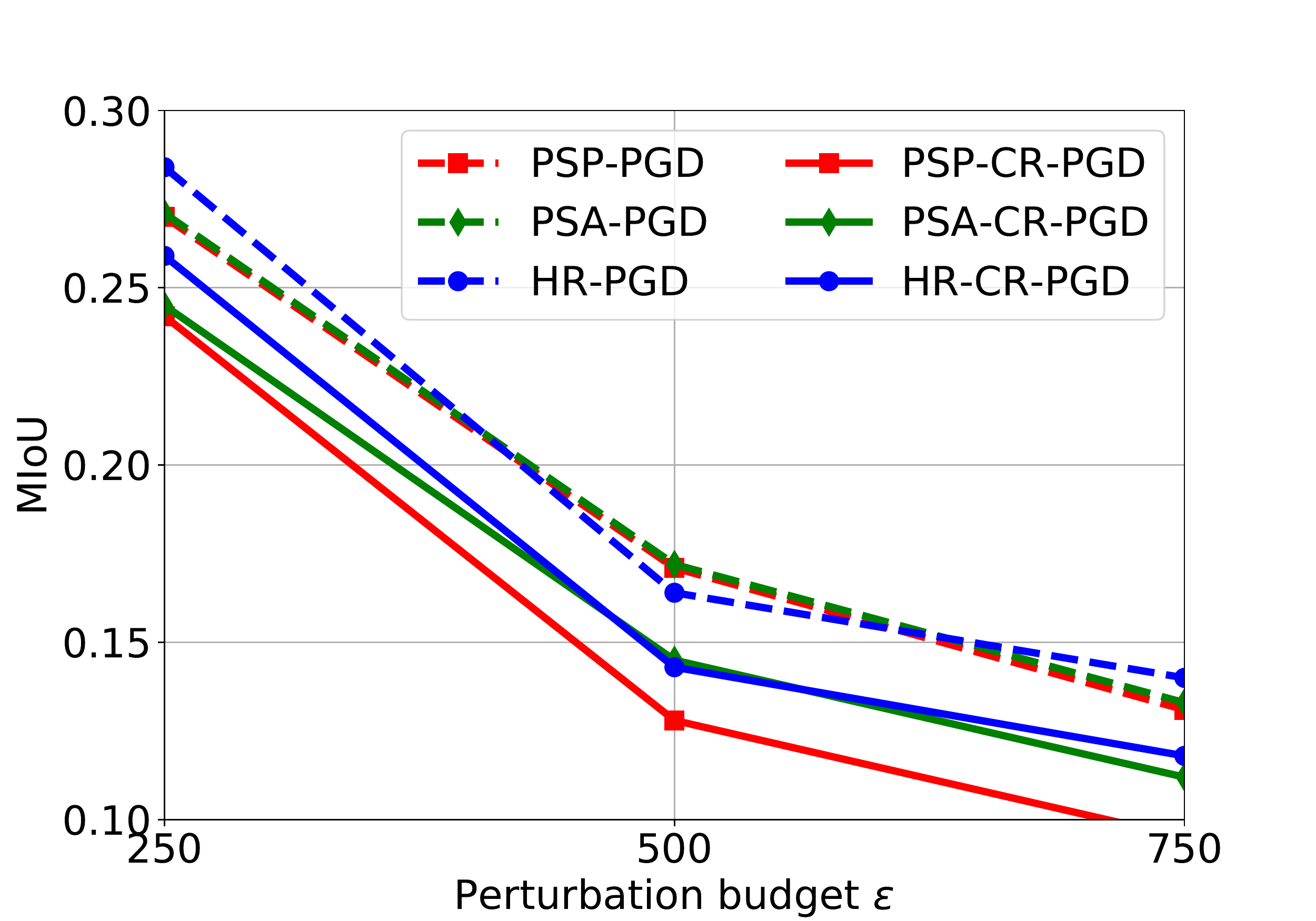}}
\subfigure[Cityscapes]{\includegraphics[width=0.32\textwidth]{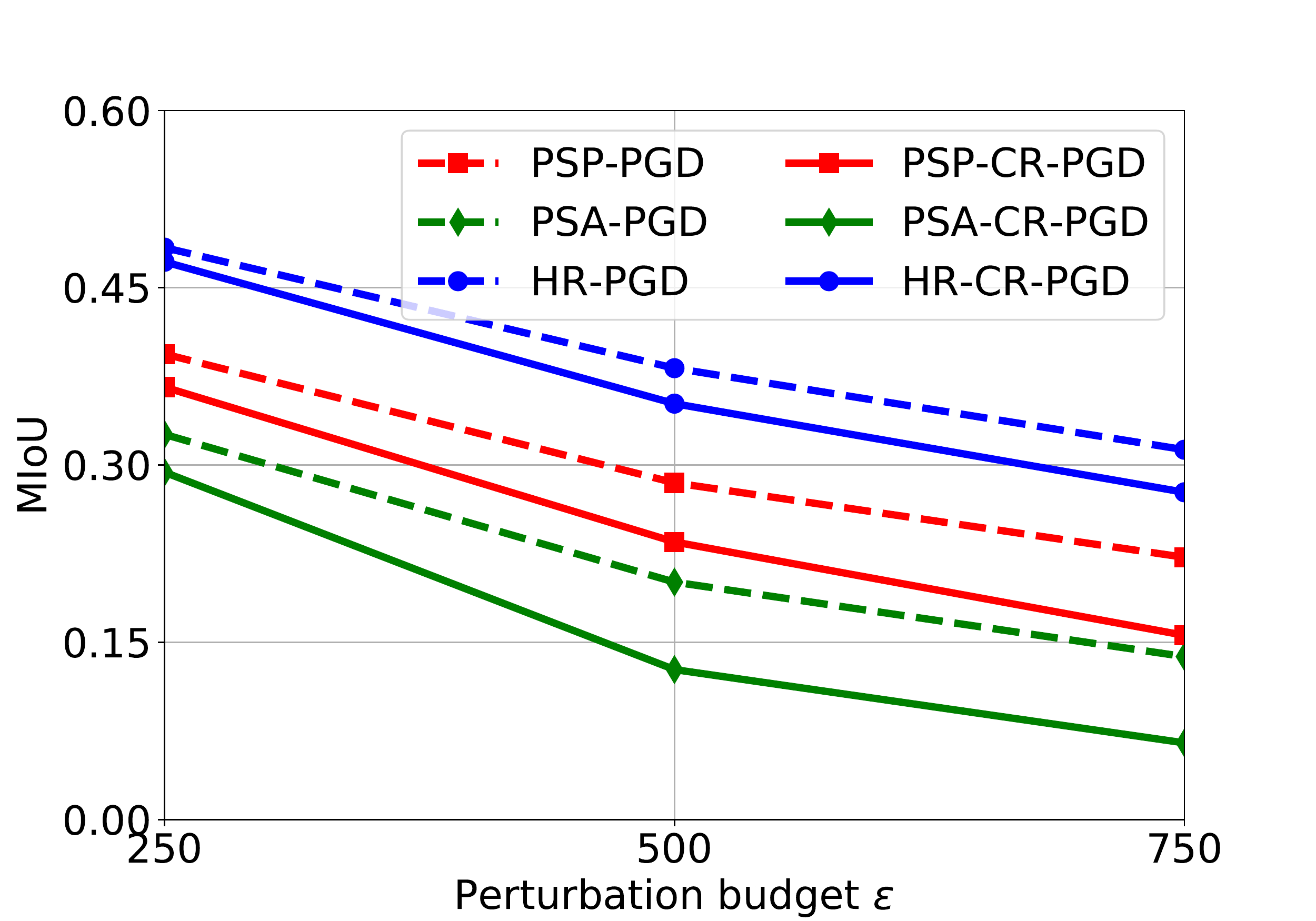}}
\subfigure[ADE20K]{\includegraphics[width=0.32\textwidth]{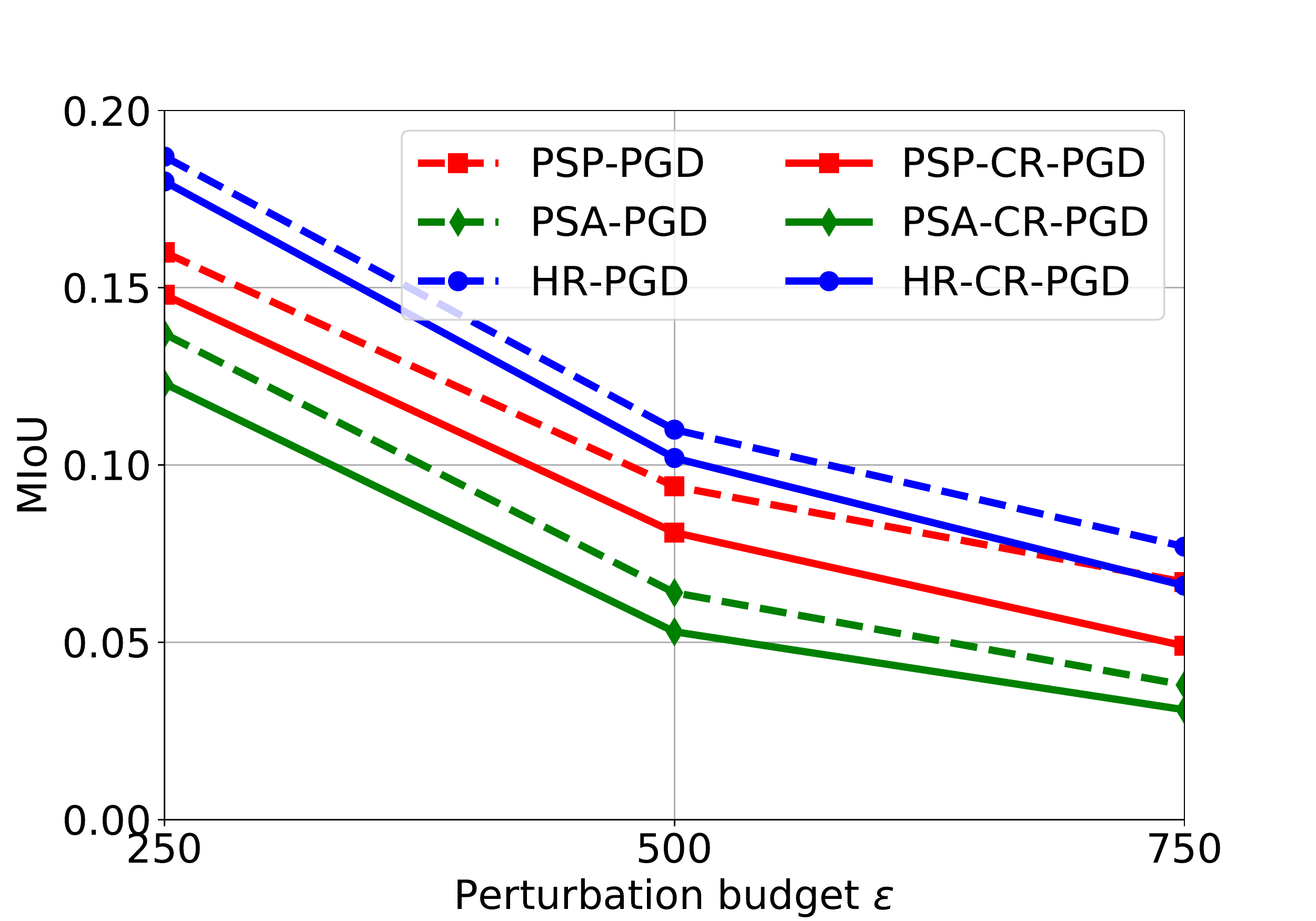}}
\vspace{-2mm}
\caption{MIoU after PGD and CR-PGD attacks with $l_1$ perturbation vs. perturbation budget $\epsilon$ on the three models and datasets} 
\label{fig:whitel1_mio}
\vspace{-4mm}
\end{figure*}

\begin{figure*}[!t]
\centering
\subfigure[Pascal VOC]{\includegraphics[width=0.32\textwidth]{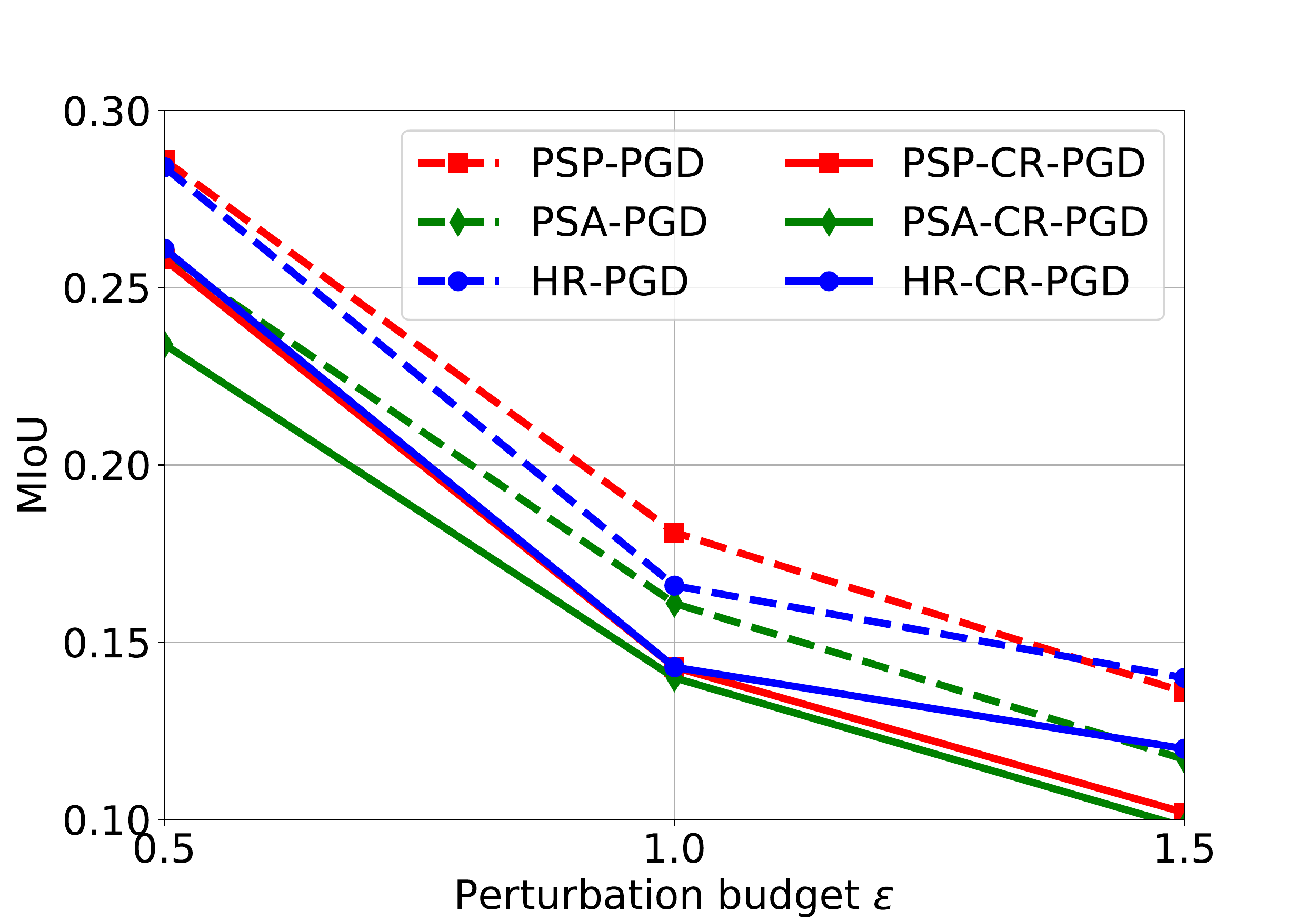}}
\subfigure[Cityscapes]{\includegraphics[width=0.32\textwidth]{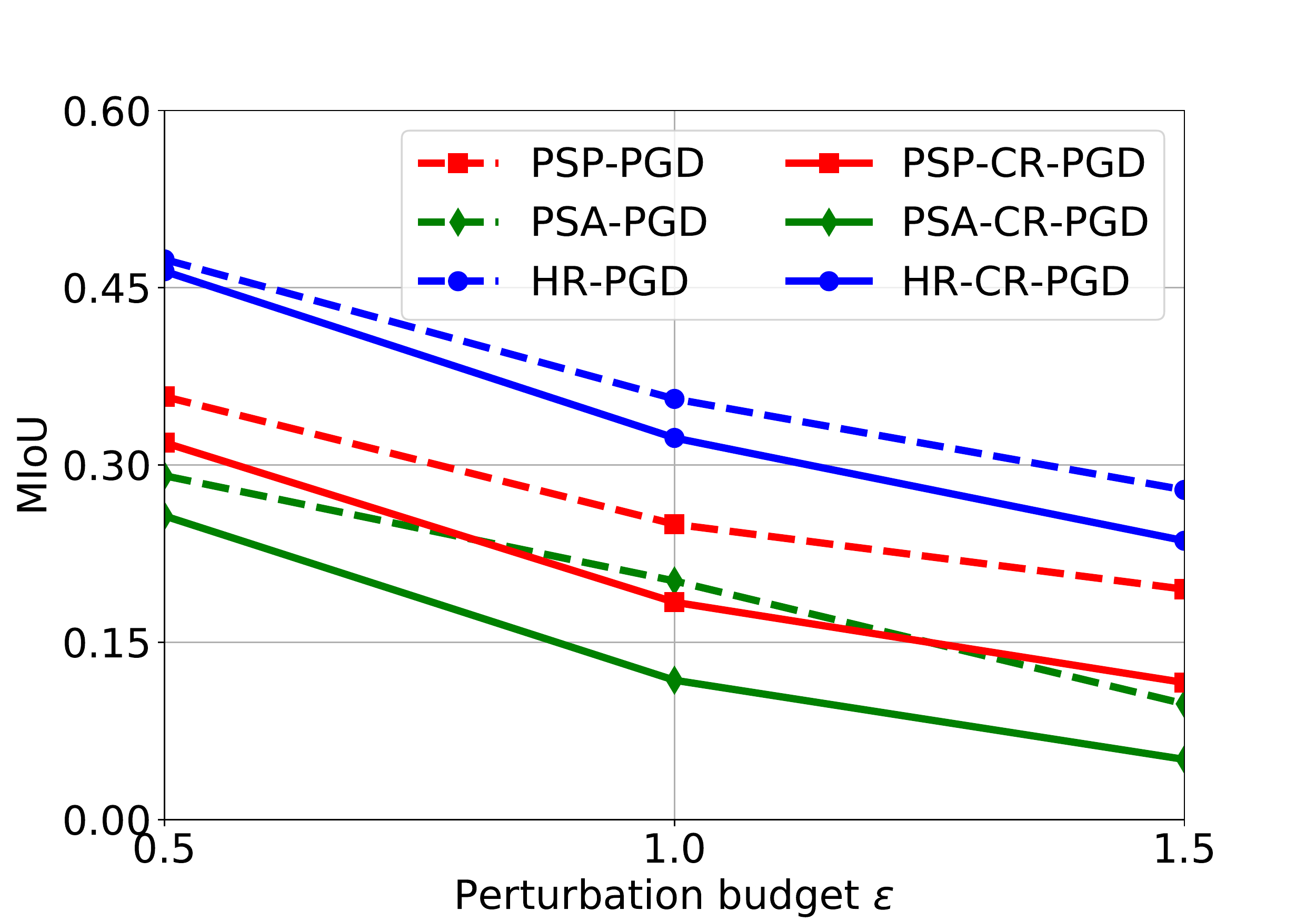}}
\subfigure[ADE20K]{\includegraphics[width=0.32\textwidth]{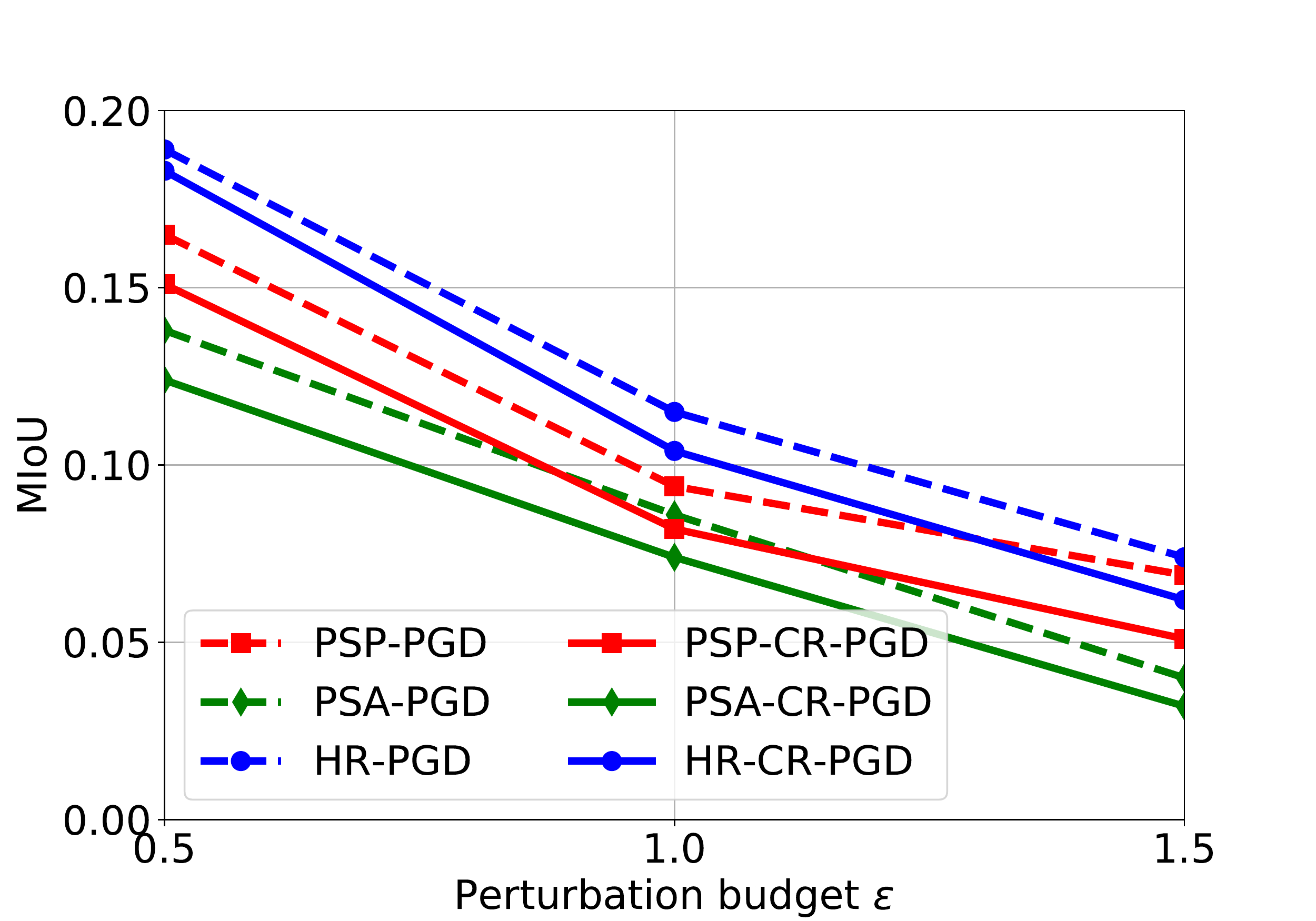}}
\vspace{-2mm}
\caption{MIoU after PGD and CR-PGD attacks with $l_2$ perturbation vs. perturbation budget $\epsilon$ on the three models and datasets} 
\label{fig:whitel2_mio}
\vspace{-4mm}
\end{figure*}

\begin{figure*}[!t]
\centering
\subfigure[Pascal VOC]{\includegraphics[width=0.32\textwidth]{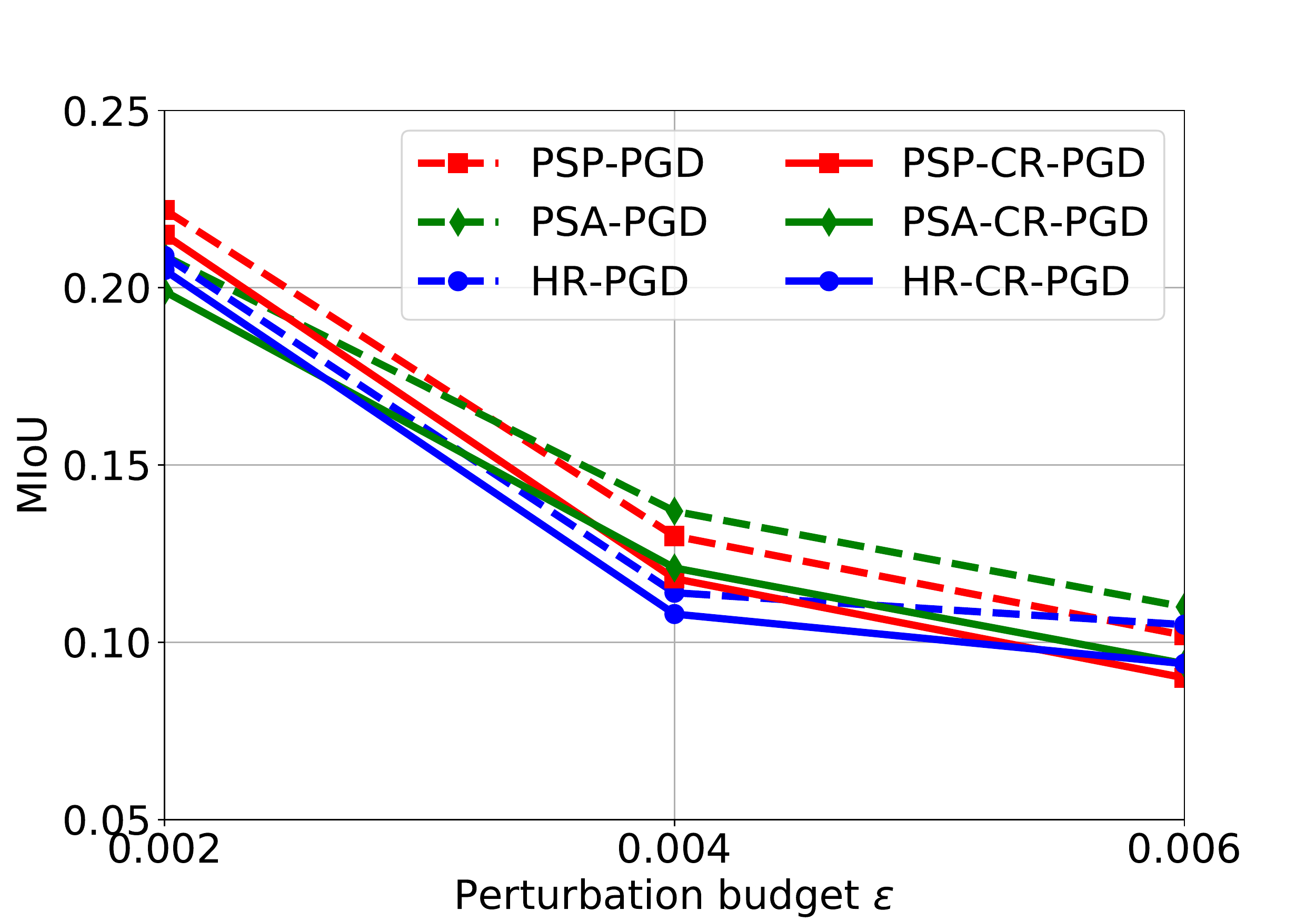}}
\subfigure[Cityscapes]{\includegraphics[width=0.32\textwidth]{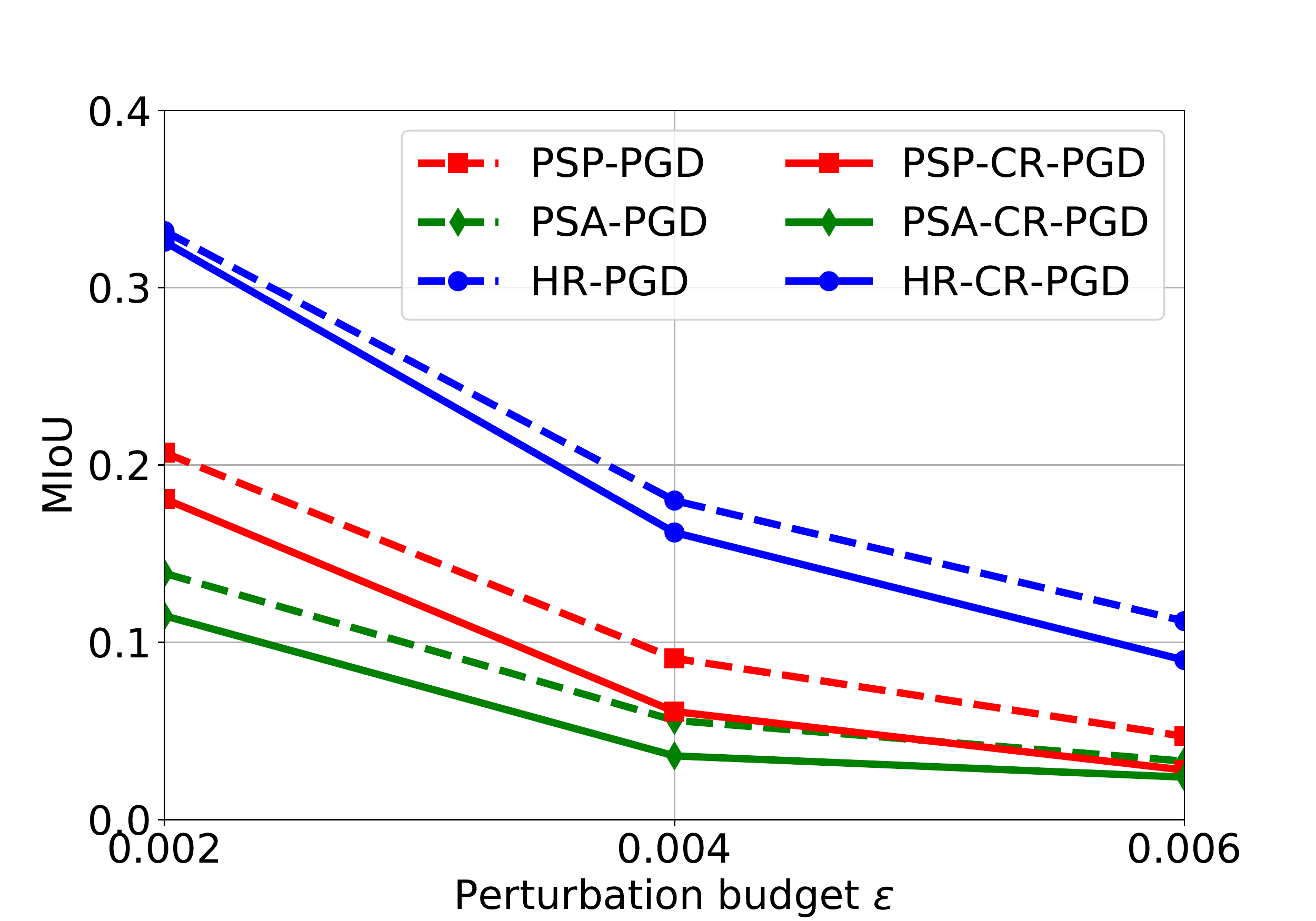}}
\subfigure[ADE20K]{\includegraphics[width=0.32\textwidth]{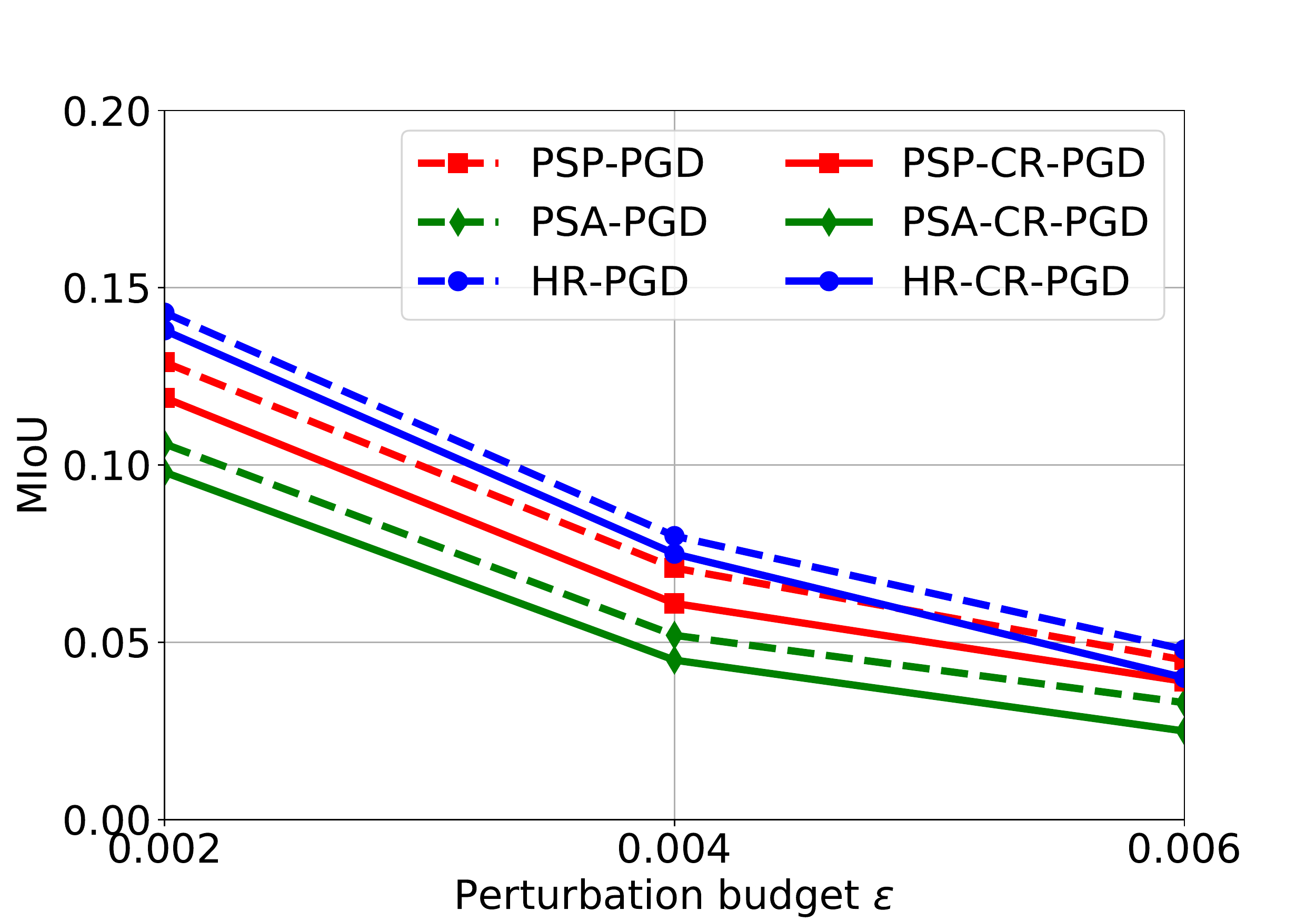}}
\vspace{-2mm}
\caption{MIoU after PGD and CR-PGD attacks with $l_\infty$ perturbation vs. perturbation budget $\epsilon$ on the three models and datasets} 
\label{fig:whitelinf_mio}
\vspace{-2mm}
\end{figure*}

\subsection{More Experimental Results}
\label{app:results}

\noindent {\bf MIoU after PGD and CR-PGD attacks.} See Figures~\ref{fig:whitel1_mio}, \ref{fig:whitel2_mio}, and ~\ref{fig:whitelinf_mio} for $l_1$, $l_2$, and $l_\infty$ perturbations, respectively. We can observe that our CR-PGD outperforms PGD.

\noindent {\bf Impact of $M$ and $\sigma$ on MIoU.} See Figures~\ref{fig:white_param_miou_m} and~\ref{fig:white_param_miou_s}, respectively. We see that CR-PGD is insensitive to $M$ and a relatively larger $\sigma$ will negatively impact our CR-PGD. 

{\noindent {\bf Impact of $\sigma$ on more datasets and models.} See Figures~\ref{fig:white_param_acc_sl1}-\ref{fig:white_param_acc_slinf}. We see that CR-PGD is insensitive to $\sigma$ within a range.} 

\noindent {\bf Defense results on MIoU.} See Figure \ref{fig:defense_miou} the defense results of FastADT on 
MIoU against CR-PGD with different $l_p$ perturbations.

\noindent {\bf Black-box attack results on MIoU.} See Figure~\ref{fig:query_mio} the impact of iterations/queries on MIoU with our PBGD and CR-PBGD black-box attacks and the $l_2$ perturbation on the three datasets.

\begin{figure*}[!t]
\centering
\vspace{-2mm}
\subfigure[$l_1$ perturbation]{\includegraphics[width=0.32\textwidth]{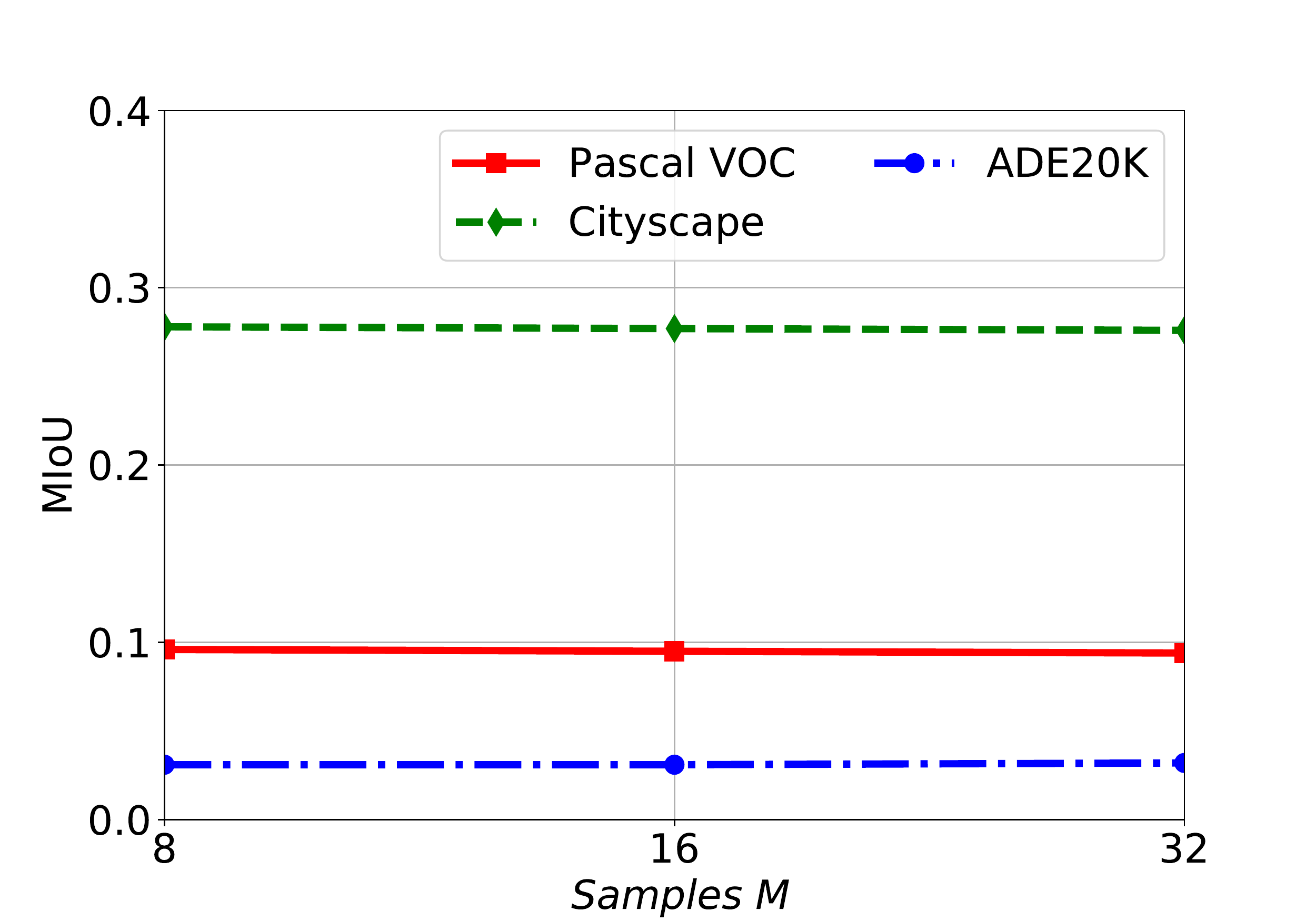}}
\subfigure[$l_2$ perturbation]{\includegraphics[width=0.32\textwidth]{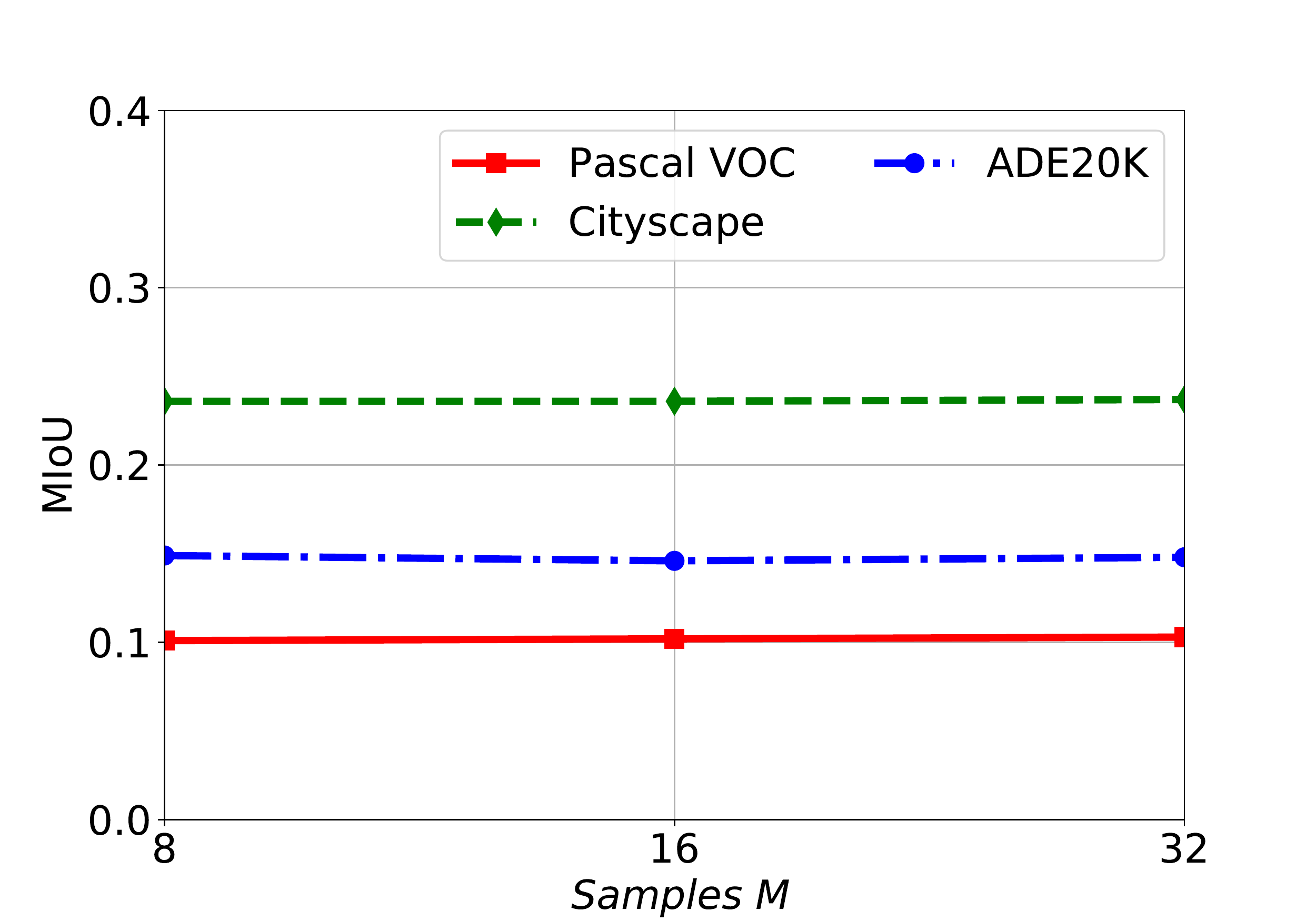}}
\subfigure[$l_\infty$ perturbation]{\includegraphics[width=0.32\textwidth]{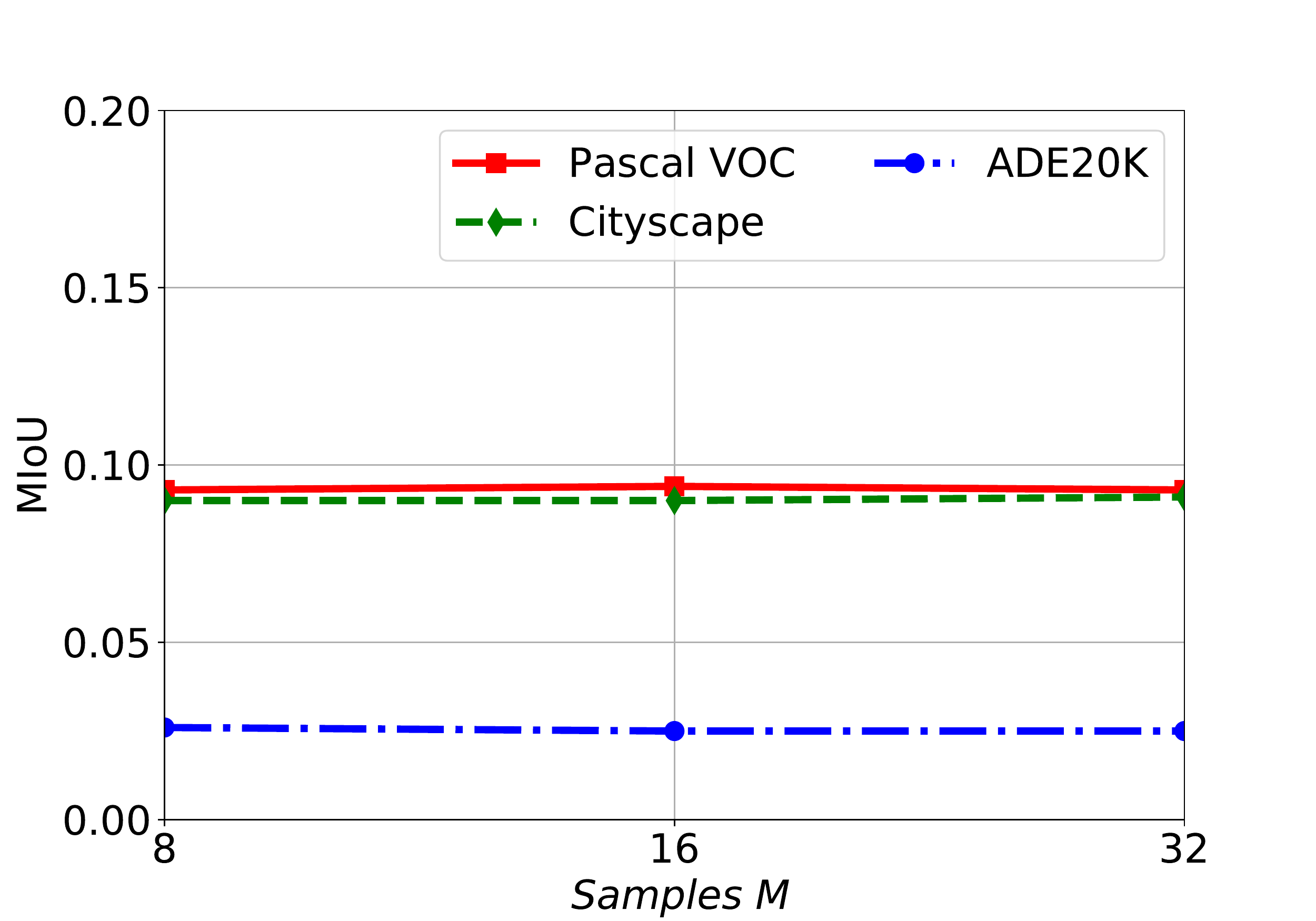}}
\vspace{-2mm}
\caption{MIoU: Impact of the \#samples $M$ on our CR-PGD attack with different $l_p$ perturbations.} 
\label{fig:white_param_miou_m}
\vspace{-4mm}
\end{figure*}

\begin{figure*}[!t]
\centering
\vspace{-2mm}
\subfigure[$l_1$ perturbation]{\includegraphics[width=0.32\textwidth]{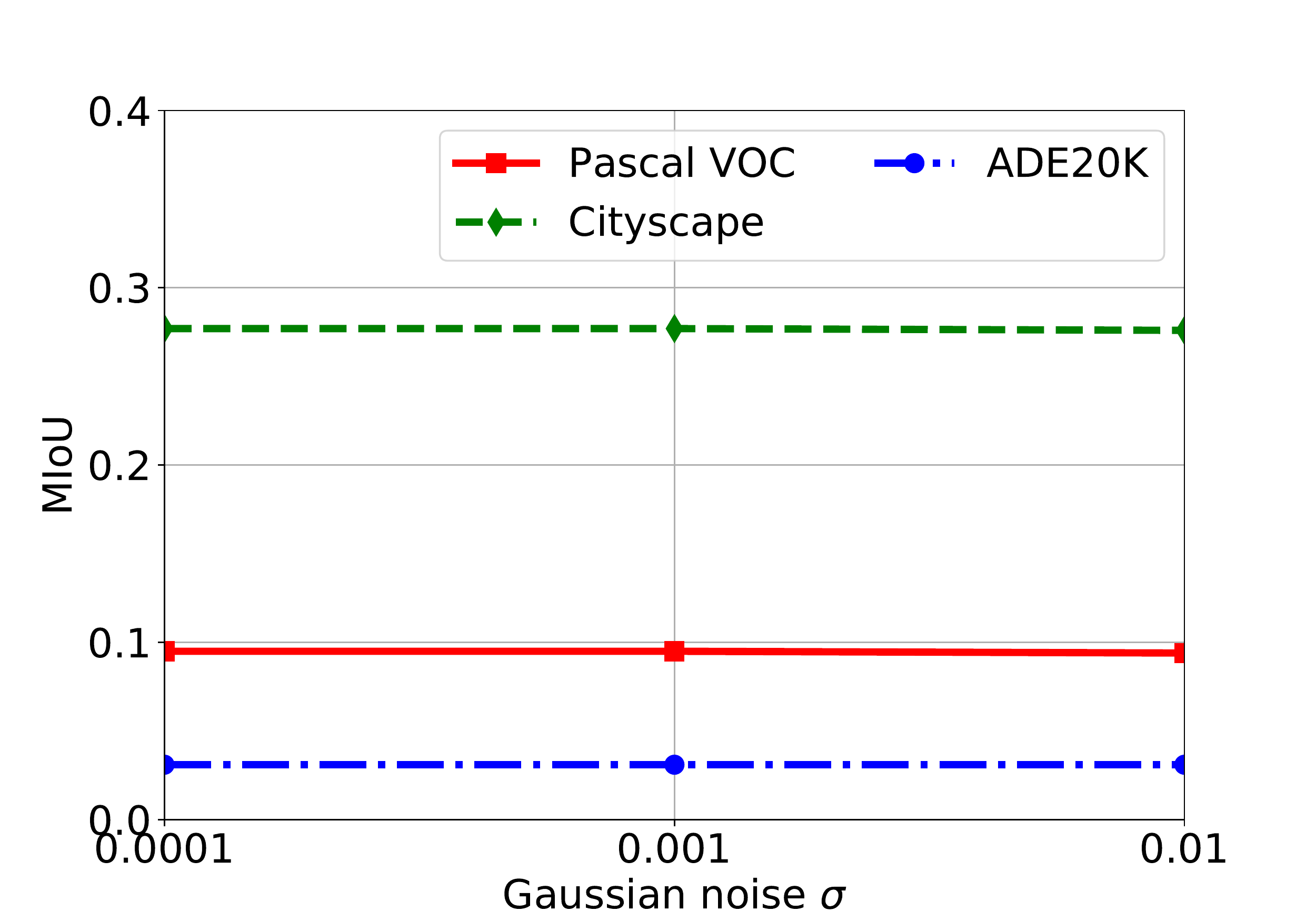}}
\subfigure[$l_2$ perturbation]{\includegraphics[width=0.32\textwidth]{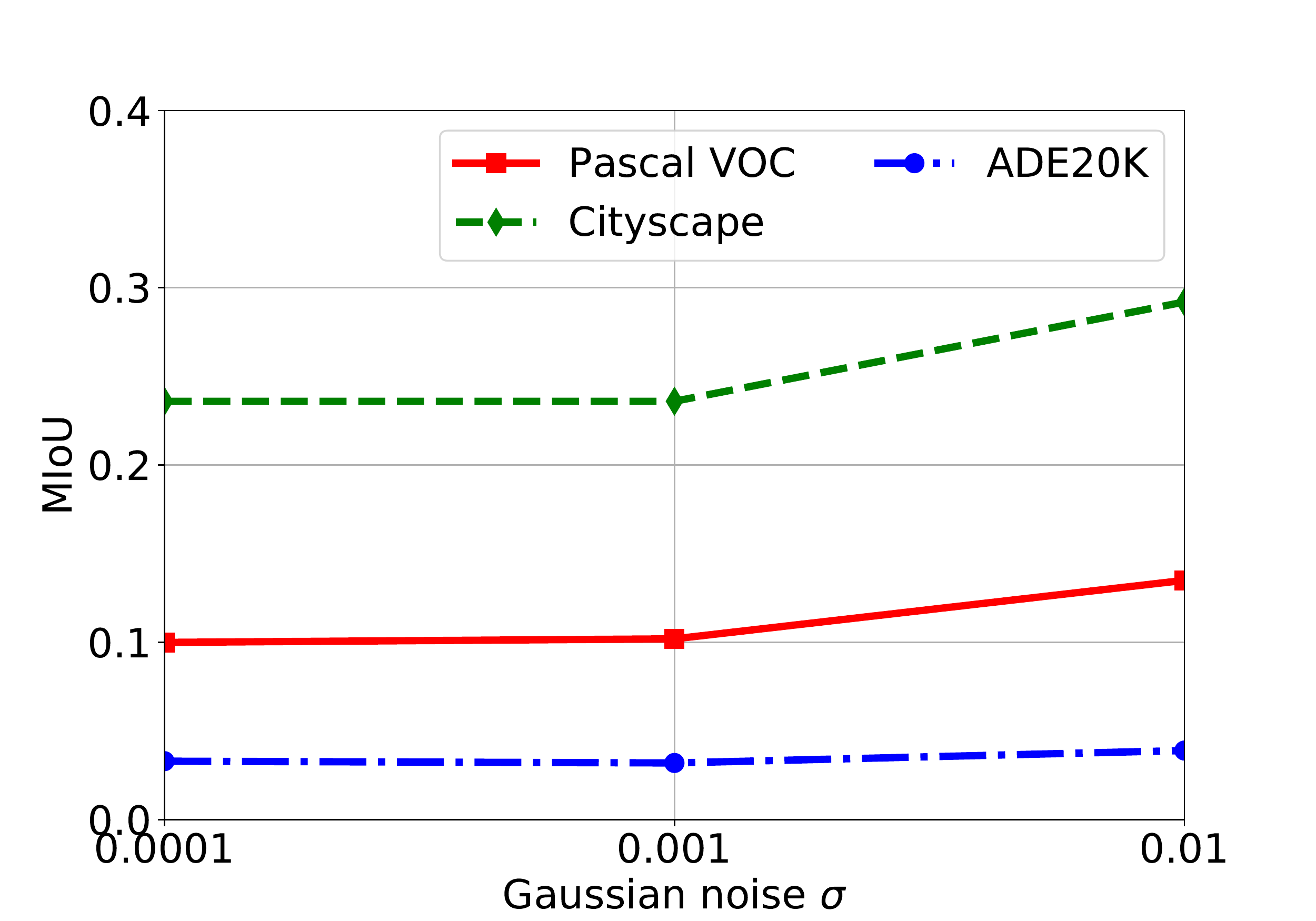}}
\subfigure[$l_\infty$ perturbation]{\includegraphics[width=0.32\textwidth]{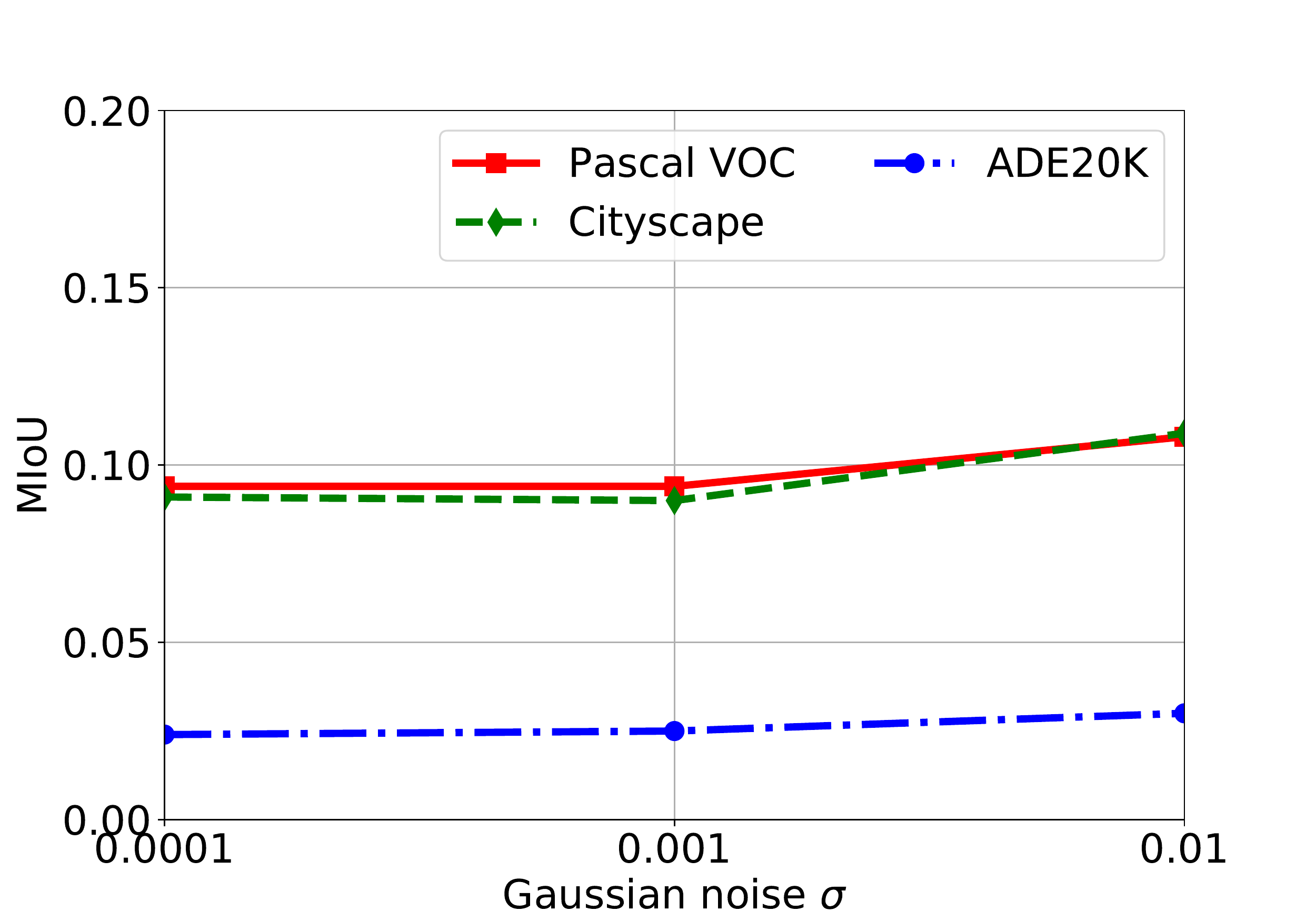}}
\vspace{-2mm}
\caption{MIoU: Impact of the Gaussian noise $\sigma$ on our CR-PGD attack with different $l_p$ perturbations.} 
\label{fig:white_param_miou_s}
\vspace{-4mm}
\end{figure*}

\begin{figure*}[!t]
\centering
\vspace{-2mm}
\subfigure[{Pascal VOC}]{\includegraphics[width=0.32\textwidth]{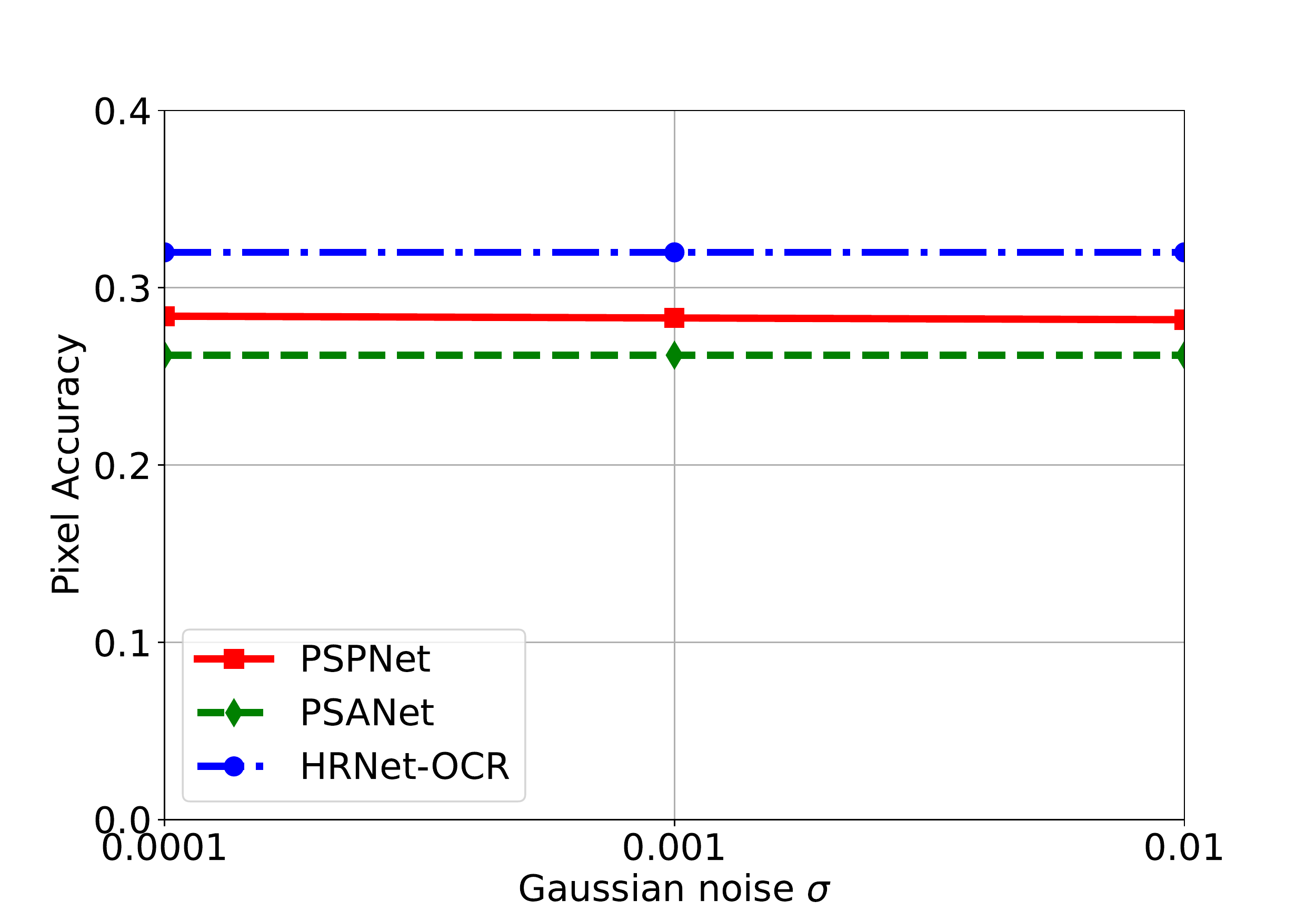}}
\subfigure[{Cityscape}]{\includegraphics[width=0.32\textwidth]{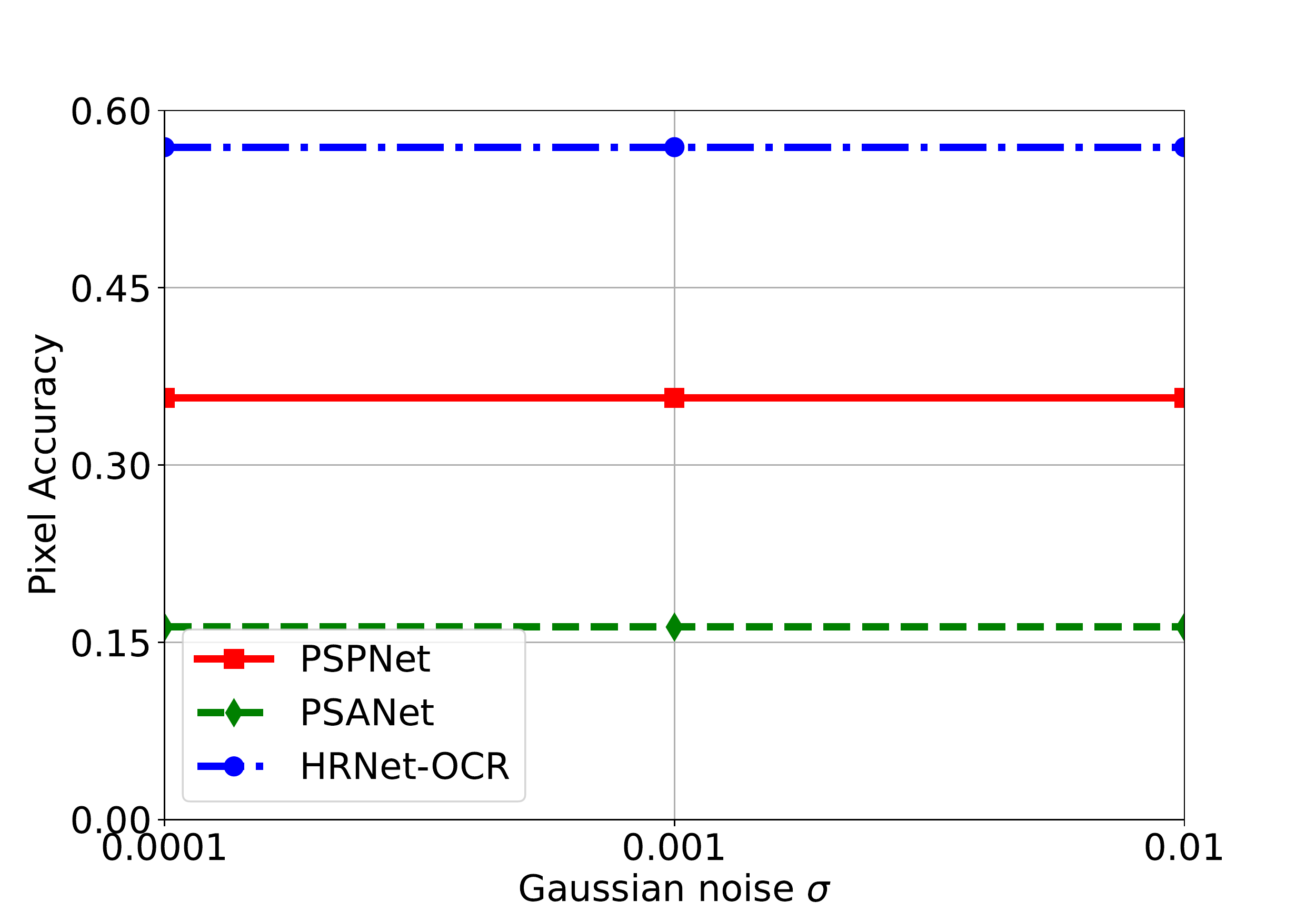}}
\subfigure[{ADE20K}]{\includegraphics[width=0.32\textwidth]{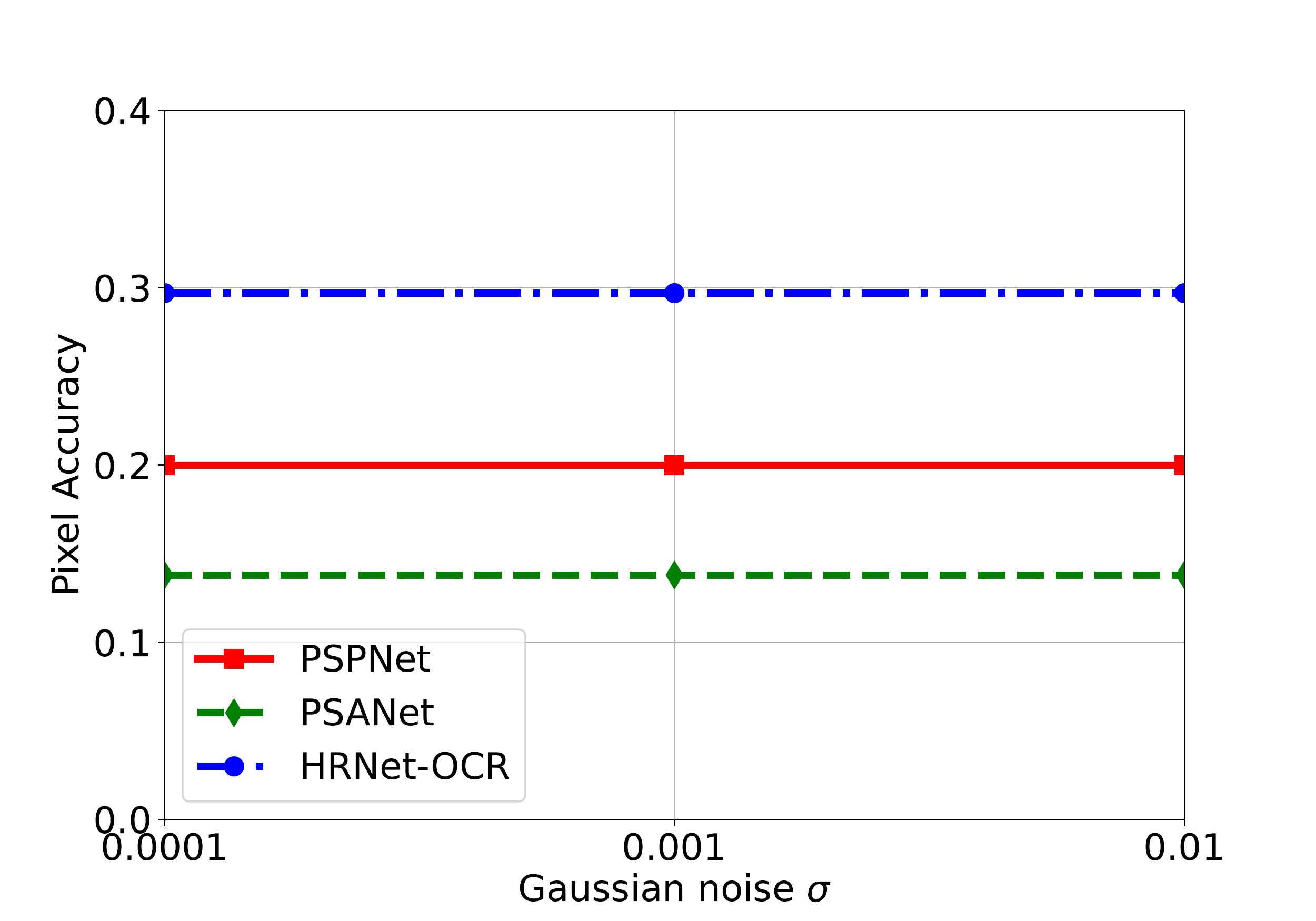}}
\vspace{-2mm}
\caption{{Impact of the Gaussian noise $\sigma$ on our CR-PGD attack with different datasets and models, under $l_1$ perturbation.} }
\label{fig:white_param_acc_sl1}
\vspace{-4mm}
\end{figure*}

\begin{figure*}[!t]
\centering
\vspace{-2mm}
\subfigure[{Pascal VOC}]{\includegraphics[width=0.32\textwidth]{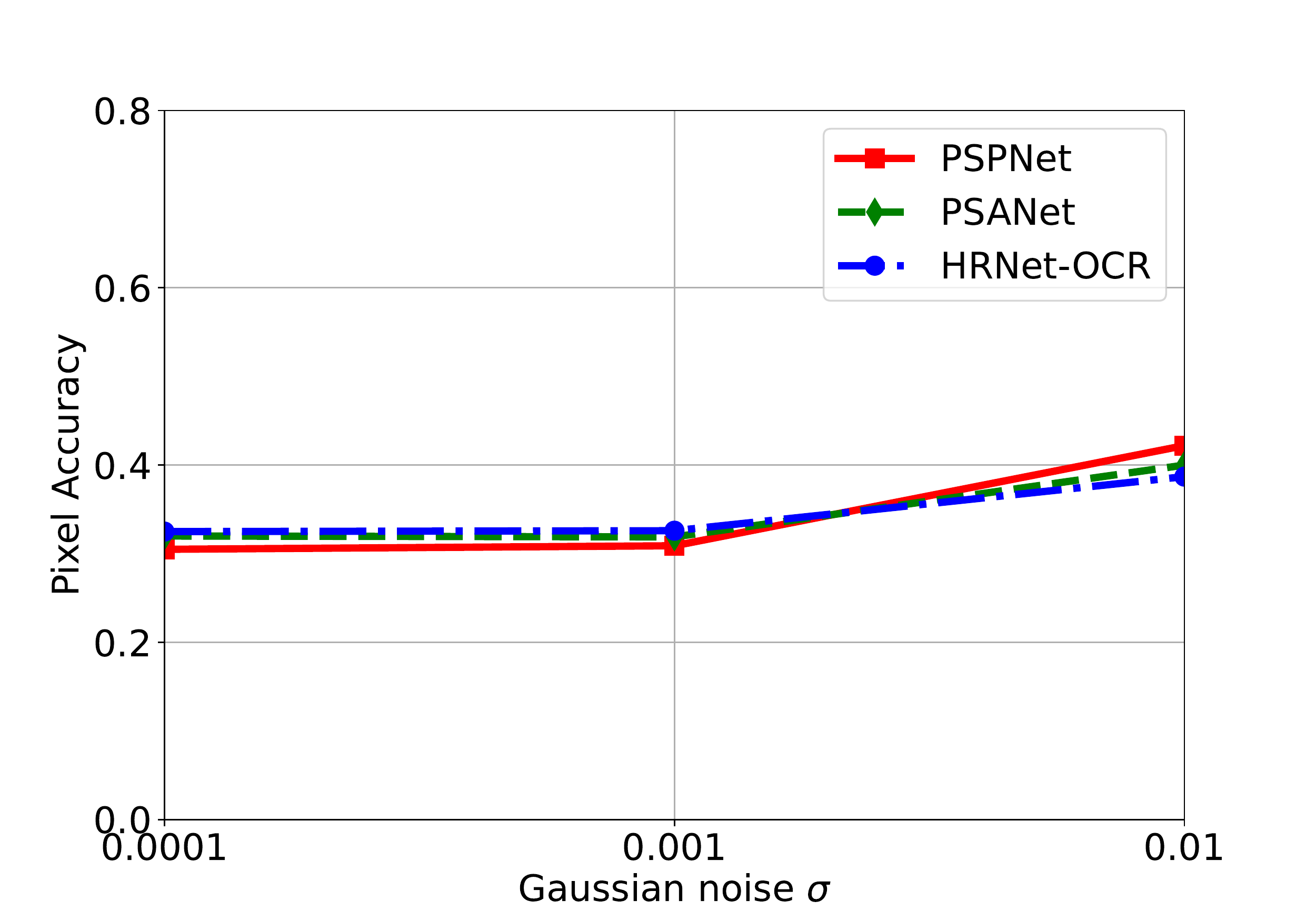}}
\subfigure[{Cityscape}]{\includegraphics[width=0.32\textwidth]{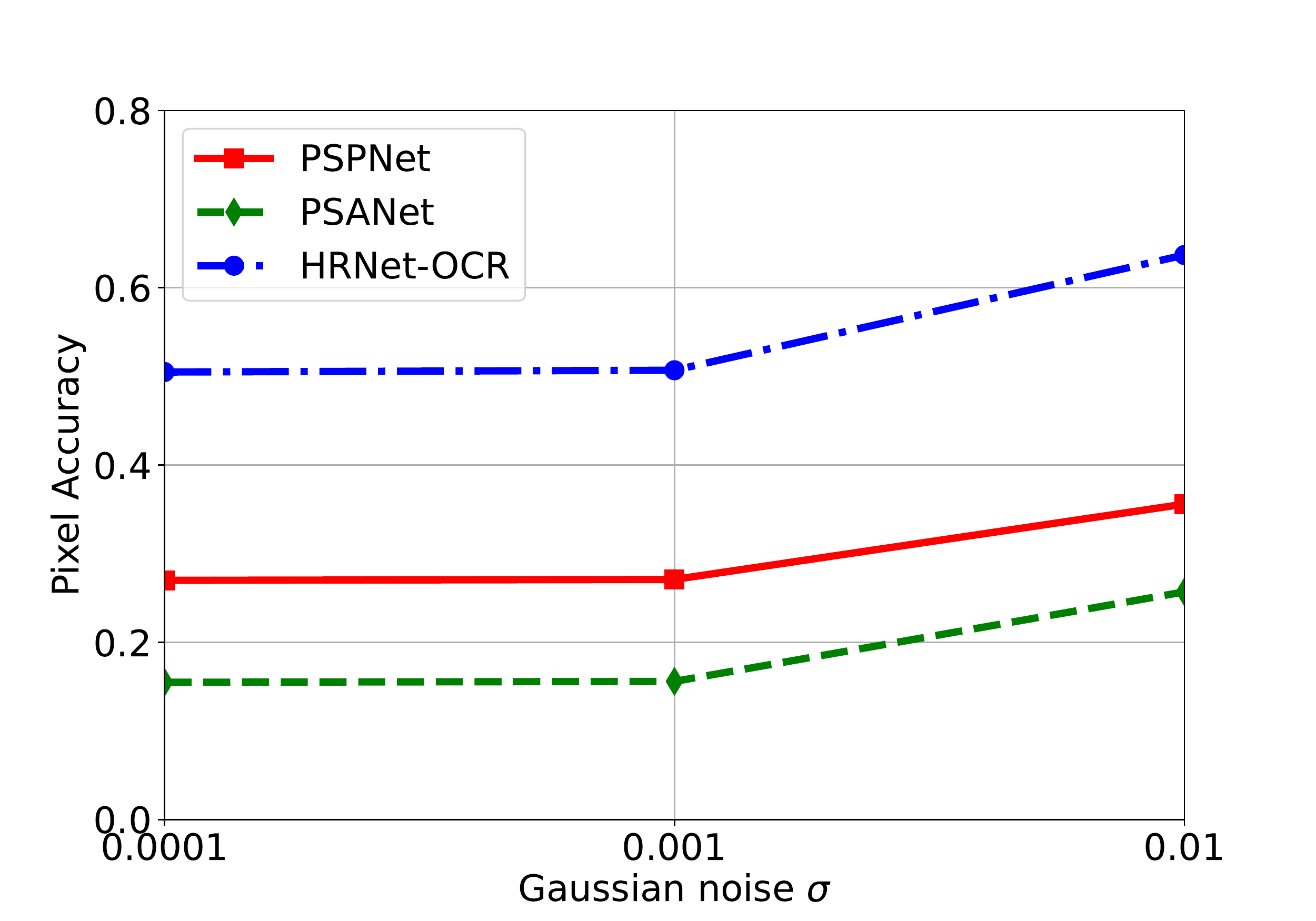}}
\subfigure[{ADE20K}]{\includegraphics[width=0.32\textwidth]{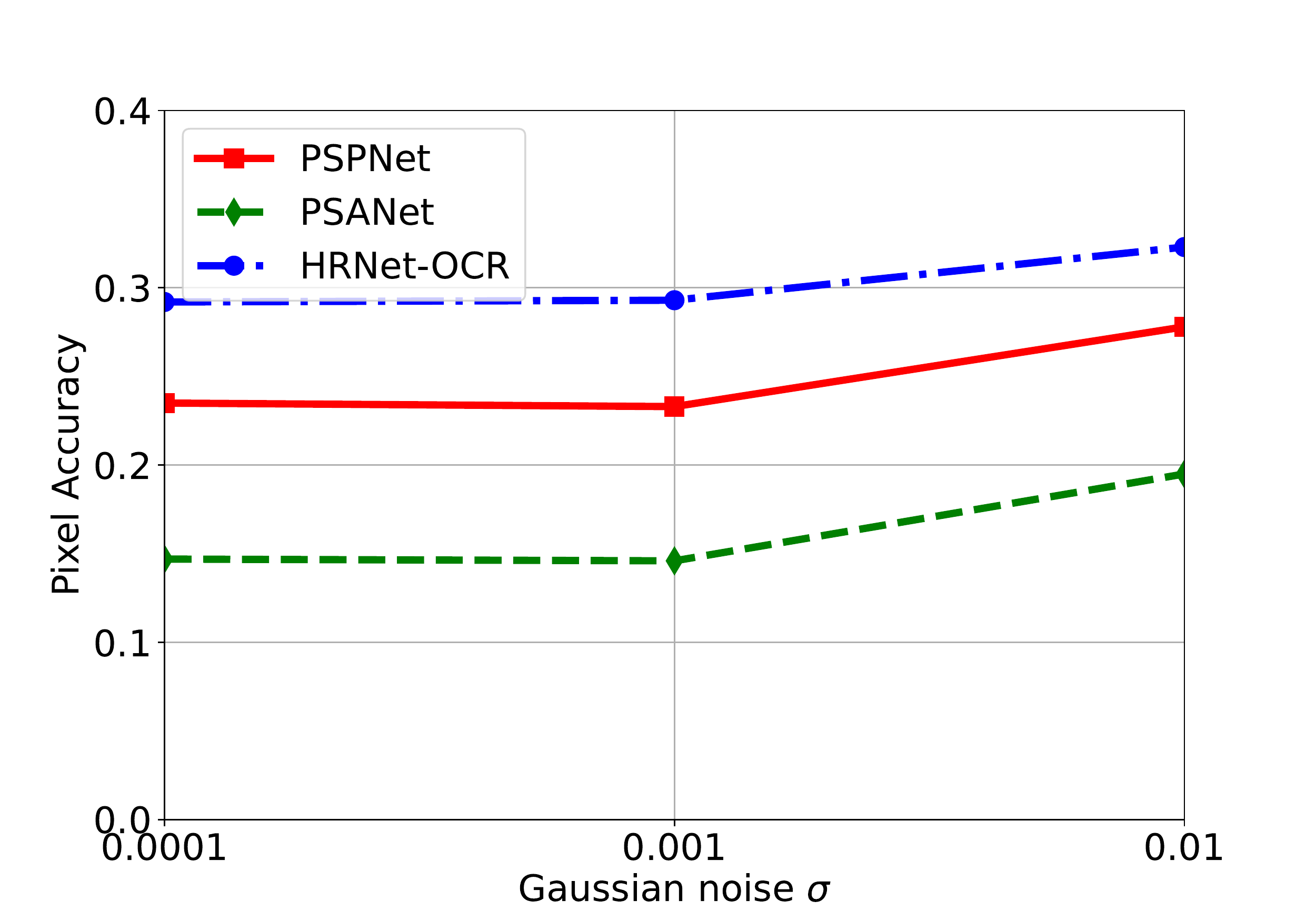}}
\vspace{-2mm}
\caption{{ Impact of the Gaussian noise $\sigma$ on our CR-PGD attack with different datasets and models, under $l_2$ perturbation.} } 
\label{fig:white_param_acc_sl2}
\vspace{-4mm}
\end{figure*}

\begin{figure*}[!t]
\centering
\vspace{-2mm}
\subfigure[{Pascal VOC}]{\includegraphics[width=0.32\textwidth]{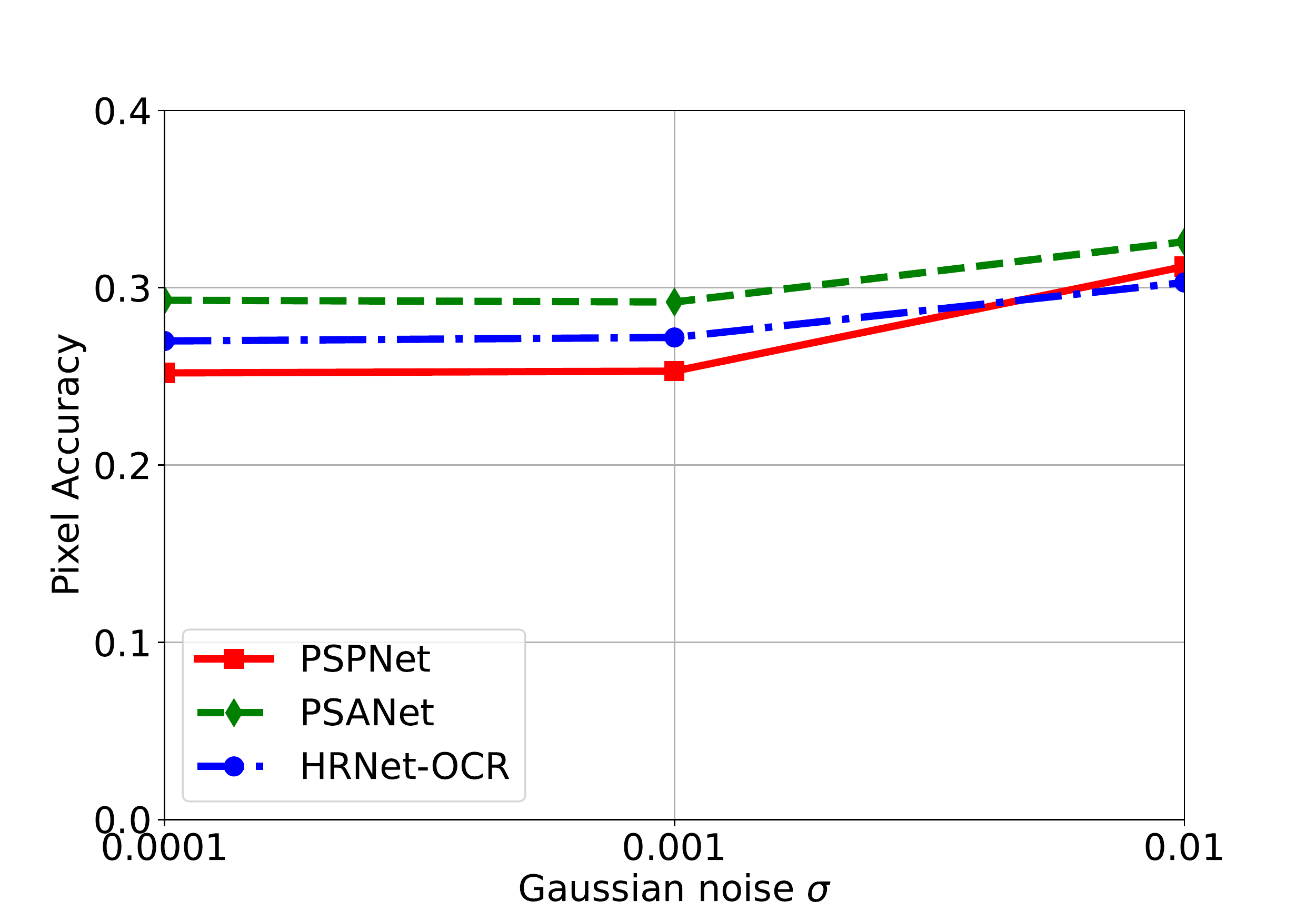}}
\subfigure[{Cityscape}]{\includegraphics[width=0.32\textwidth]{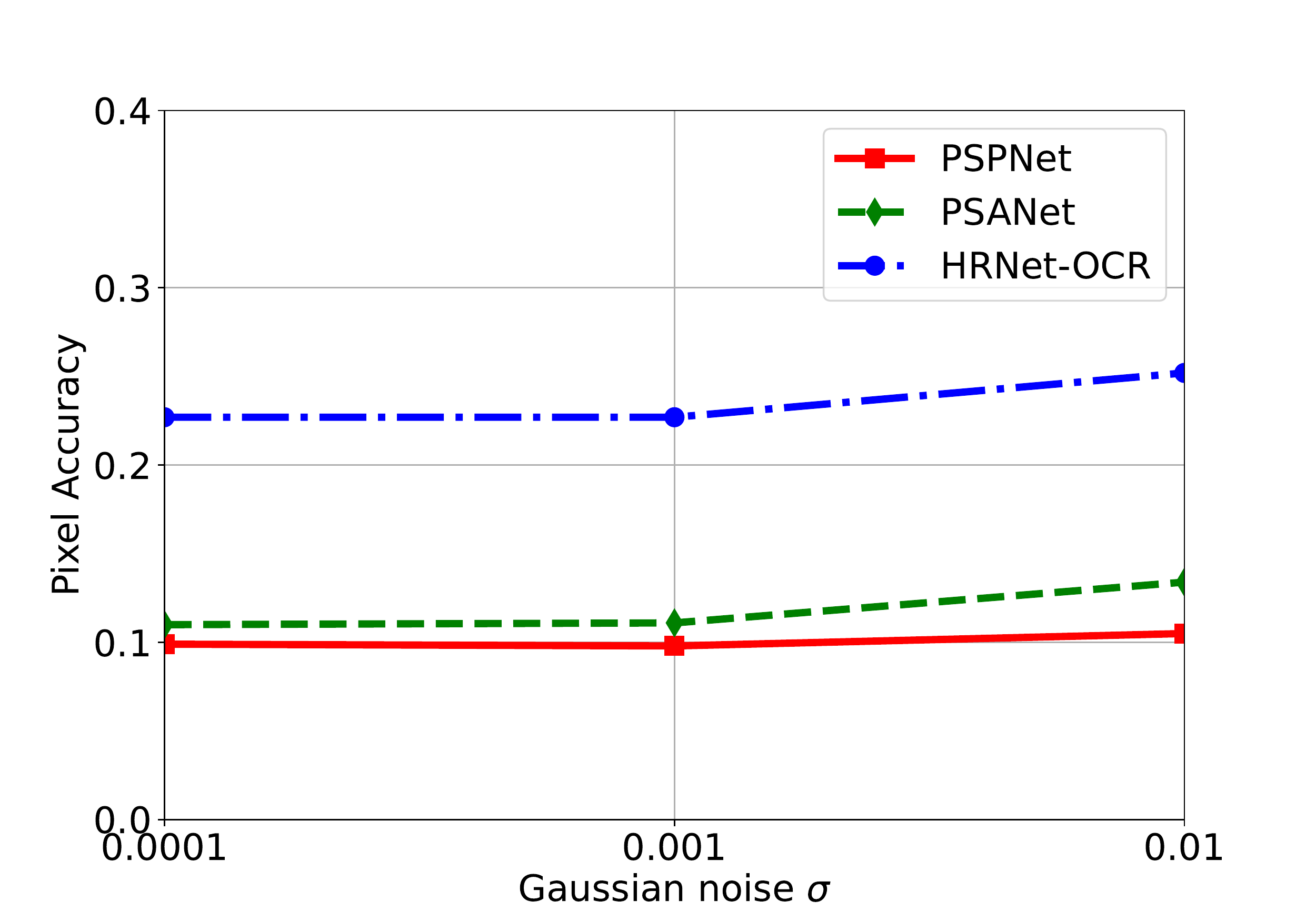}}
\subfigure[{ADE20K}]{\includegraphics[width=0.32\textwidth]{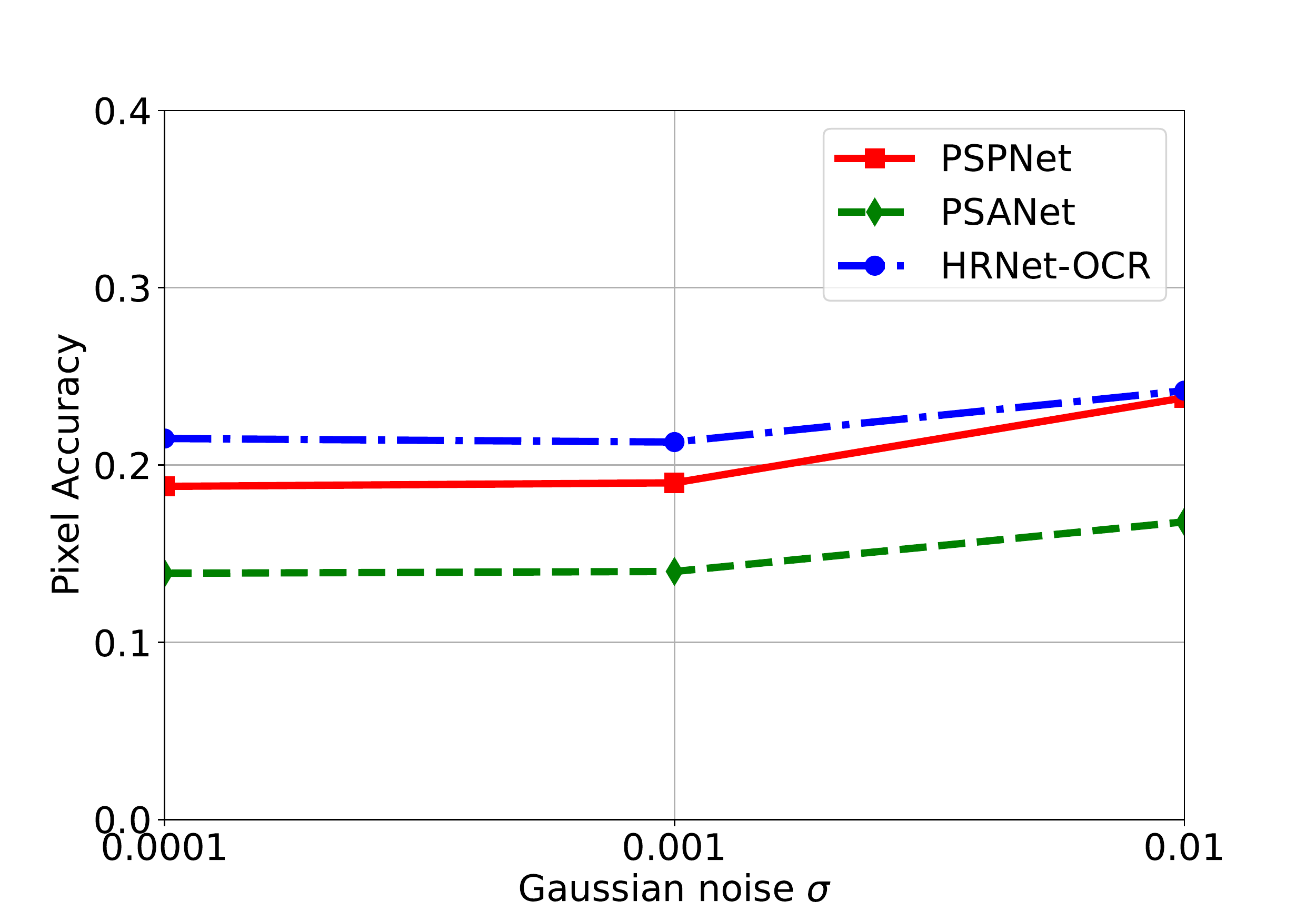}}
\vspace{-2mm}
\caption{{ Impact of the Gaussian noise $\sigma$ on our CR-PGD attack with different datasets and models, under $l_\infty$ perturbation.} } 
\label{fig:white_param_acc_slinf}
\vspace{-4mm}
\end{figure*}

\vspace{-2mm}
\subsection{Attacks and Defense}
\label{supp:methods}

We first briefly introduce the Fast Gradient Sign Method (FGSM)~\cite{goodfellow2014explaining}, Projected Gradient Descent (PGD)~\cite{madry2018towards}, and Dense Adversary Generation (DAG)~\cite{xie2017adversarial} attacks. Then, we introduce the fast adversarial training (FastADT) defense~\cite{wong2020fast}.

Suppose we are given a segmentation model $F_\theta$, a loss function $L$, and a testing image $x$ with groudtruth segmentation labels $y$, and the perturbation budget $\epsilon$.

\noindent {\bf Projected Gradient Descent (PGD).} 
PGD \emph{iteratively} finds the $l_p$ perturbations by increasing the loss of the model on a testing image $x$:
\begin{align}
\delta \leftarrow \textrm{Proj}_{\mathbb{B}} (\delta + \alpha \cdot  \nabla L(F_\theta(x+\delta),y)), 
\end{align}
where $\alpha$ is the learning rate in PGD and $\mathbb{B} = \{ \delta: ||\delta||_p \le \epsilon \}$ is the allowable perturbation set. $\textrm{Proj}_\mathbb{B}$ projects the adversarial perturbation to the allowable set $\mathbb{B}$. 
The specific forms of $l_1$, $l_2$, and $l_\infty$ perturbations are as follows:
\begin{align}
& L_1:  \delta \leftarrow \textrm{Clip}( \delta + \alpha \cdot \frac{\nabla L(F_\theta(x+\delta),y)}{||\nabla L(F_\theta(x+\delta),y)||_1})       \\
&L_2:  \delta \leftarrow \textrm{Clip}(\delta + \alpha \cdot  \frac{\nabla L(F_\theta(x+\delta),y)}{||\nabla L(F_\theta(x+\delta),y)||_2}) \\
& L_\infty:  \delta \leftarrow \textrm{Clip}( \delta + \alpha \cdot  \textrm{sign}(\nabla L(F_\theta(x+\delta),y), \epsilon),       
\end{align}
where $\textrm{Clip}(a, \epsilon)$ makes each $a_n$ in the range $[-\epsilon, \epsilon]$.

\begin{figure*}[!t]
\centering
\subfigure[$l_1$ perturbation ]{\includegraphics[width=0.32\textwidth]{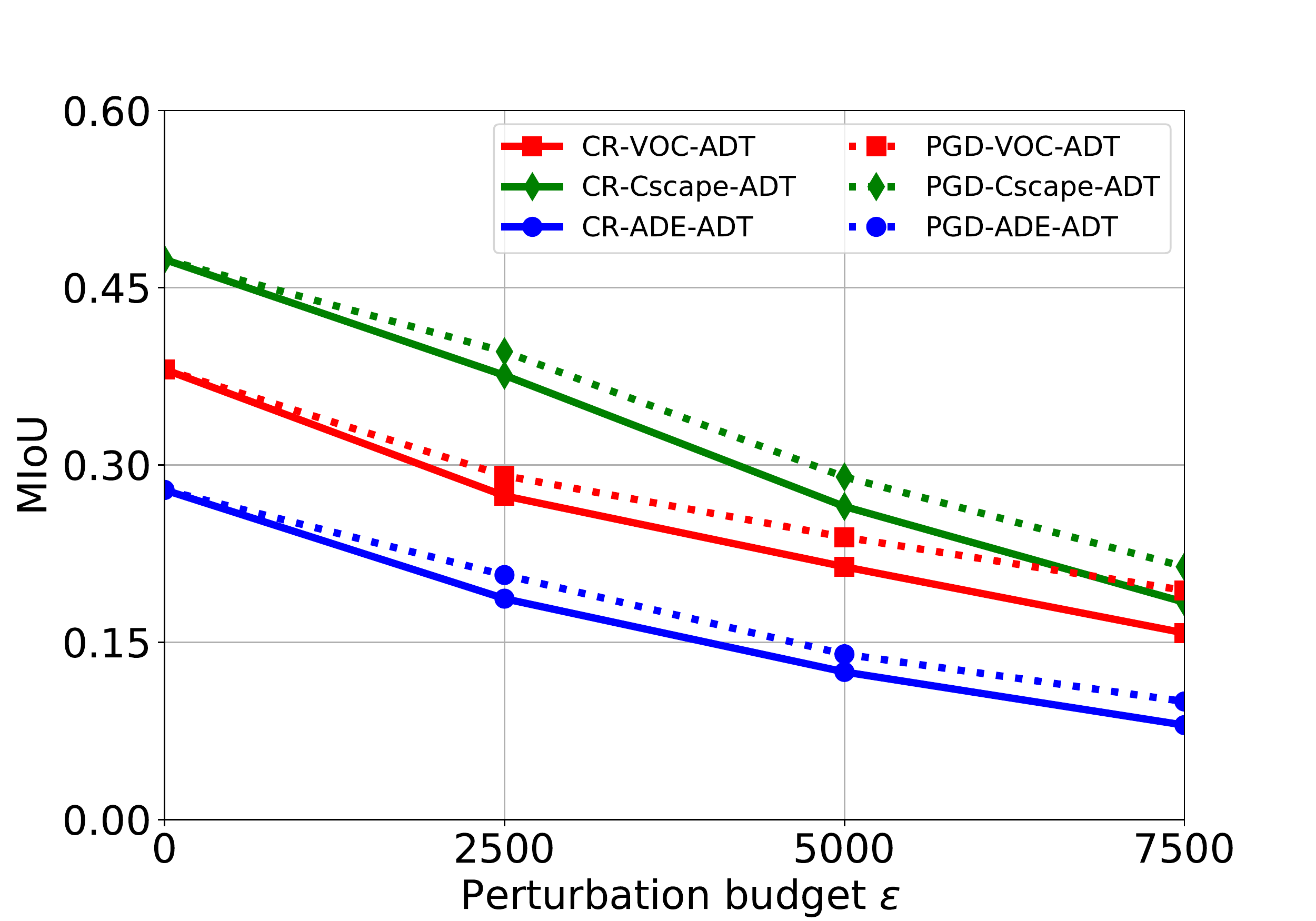}}
\subfigure[$l_2$ perturbation]{\includegraphics[width=0.32\textwidth]{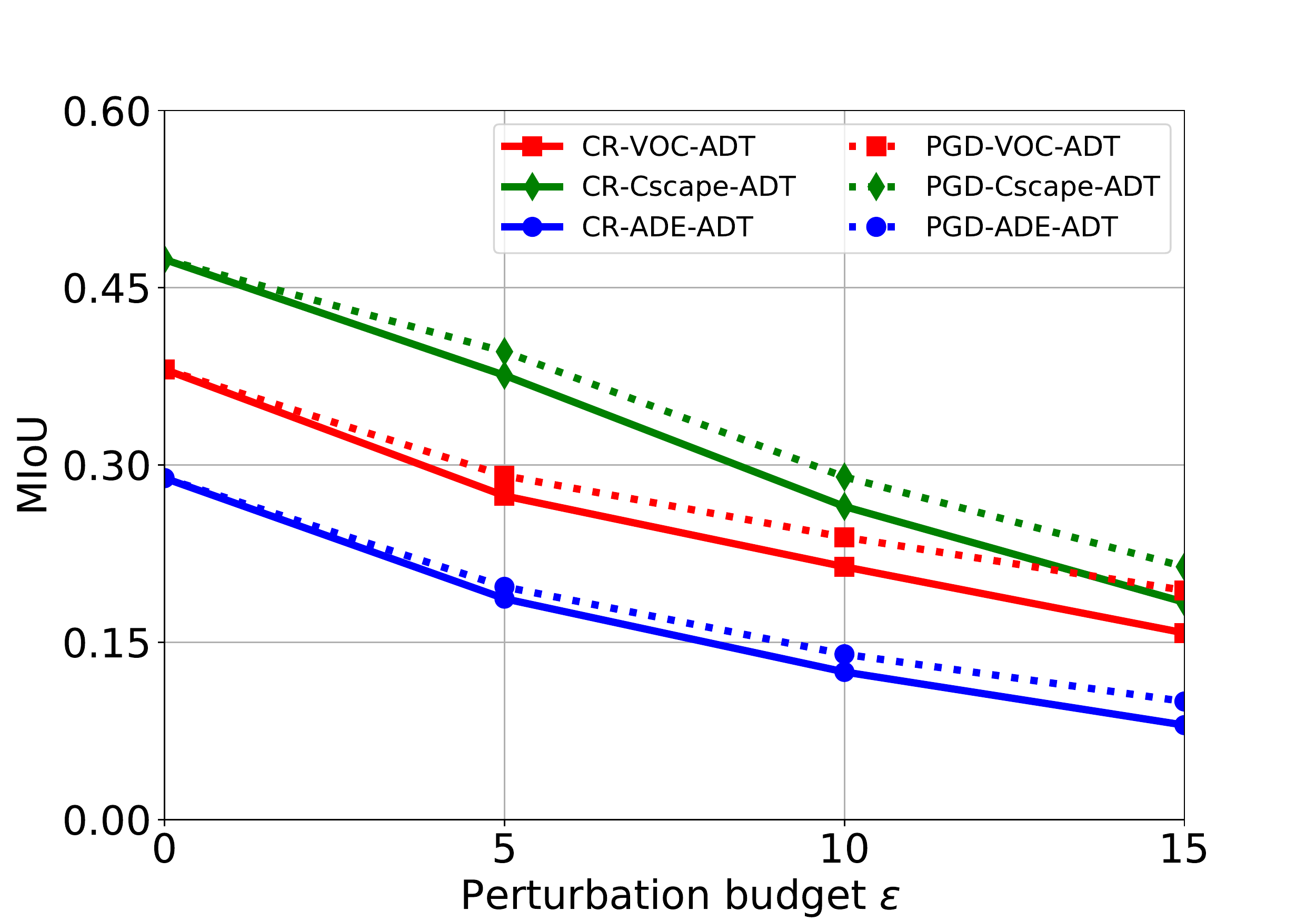}}
\subfigure[$l_\infty$ perturbation]{\includegraphics[width=0.32\textwidth]{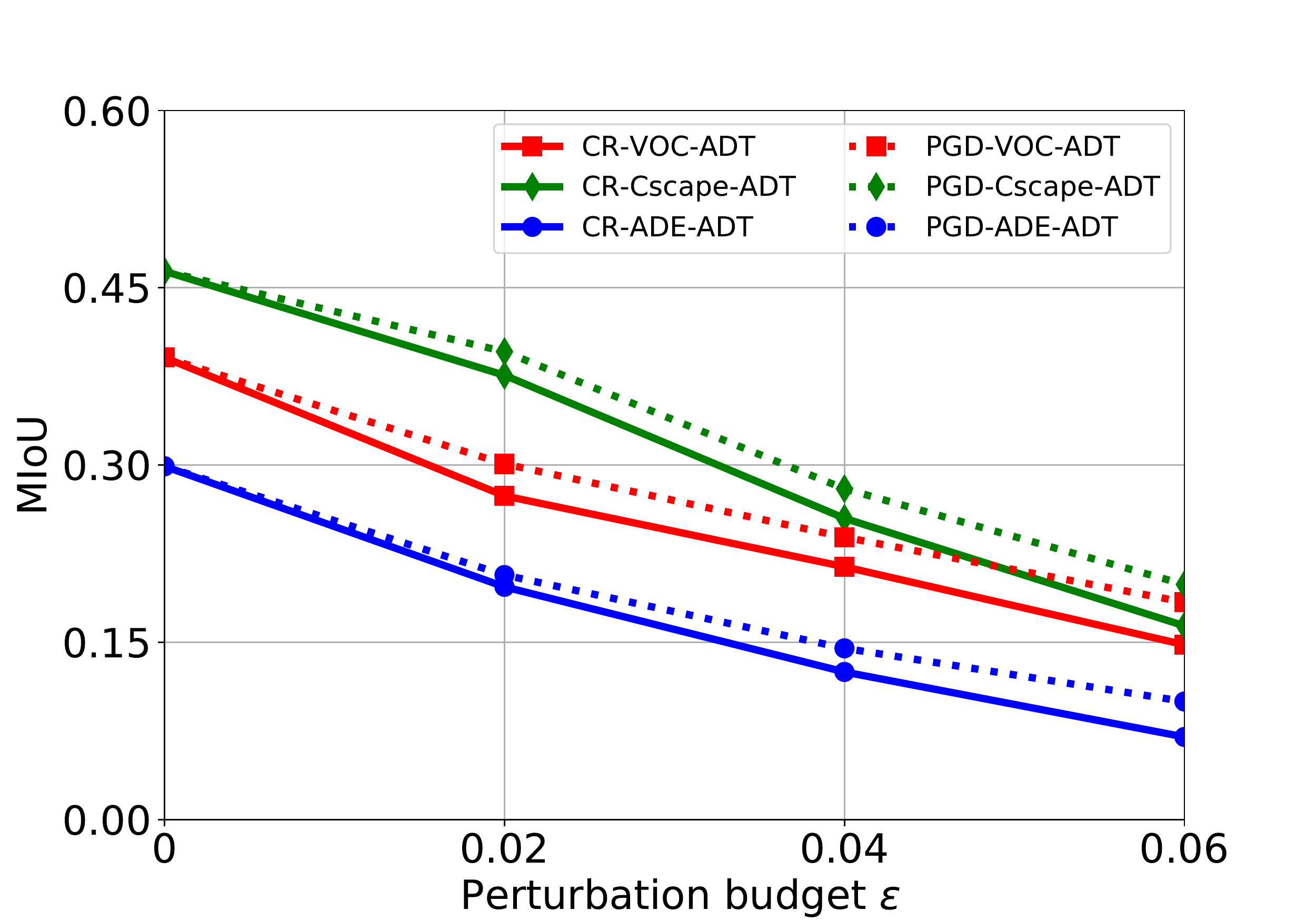}}
\vspace{-2mm}
\caption{{MIoU: Defending against our white-box CR-PGD attack via fast adversarial training.}} 
\label{fig:defense_miou}
\vspace{-2mm}
\end{figure*}

\begin{figure}[!t]
\centering
\vspace{-2mm}
{\includegraphics[width=0.36\textwidth]{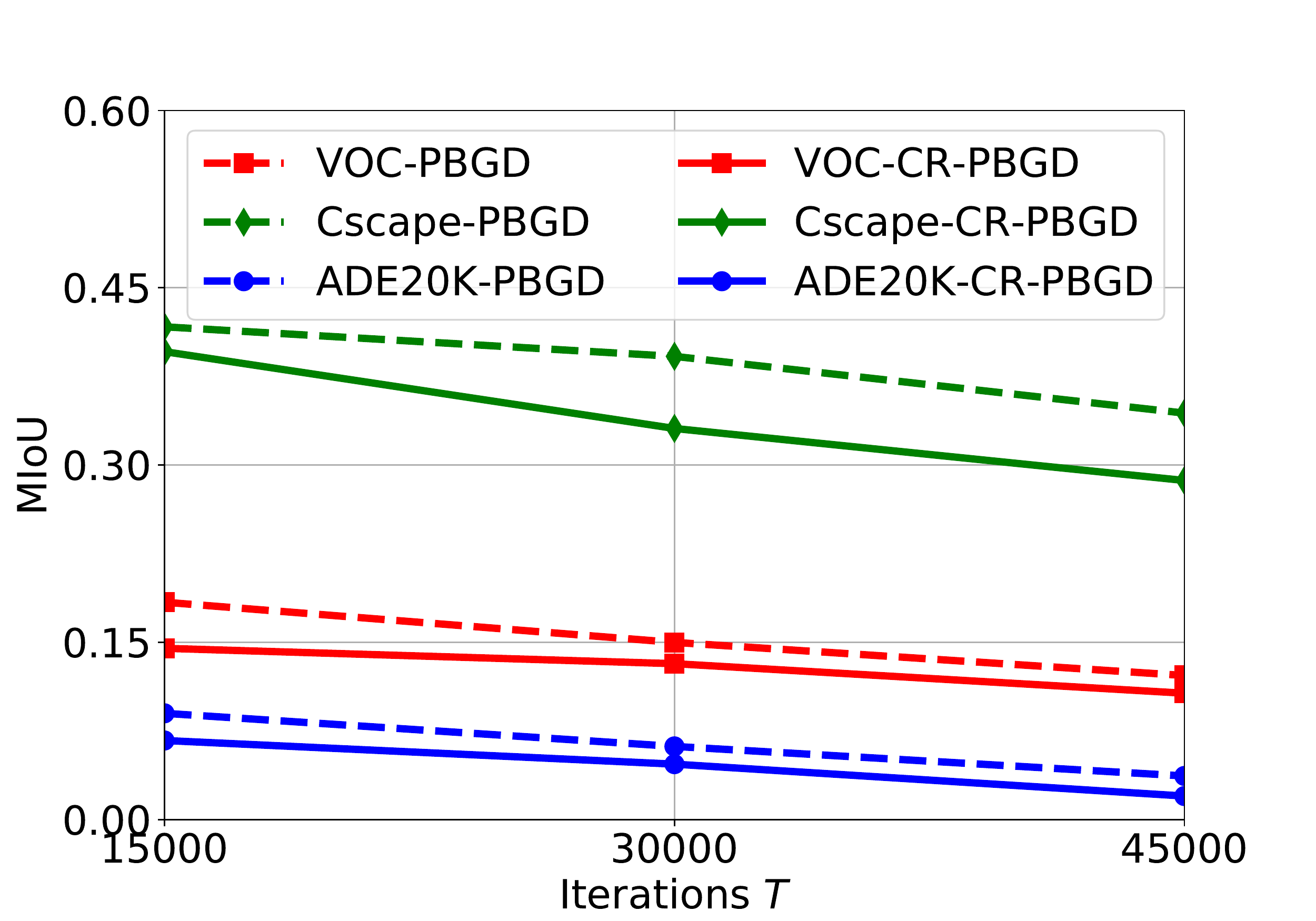}}
\vspace{-2mm}
\caption{MIoU: PGD and CR-PGD attacks with $l_2$ perturbation vs. $T$.} 
\label{fig:query_mio}
\end{figure}

\noindent {\bf Fast Gradient Sign (FGSM).}
FGSM is a variant of PGD, where it generates $l_p$ perturbations with a \emph{single step}.

\noindent {\bf Dense Adversary Generation (DAG).}
DAG generates an $l_\infty$ perturbation to make the predictions of all pixels be wrong (see Algorithm~\ref{alg:DAG} for the details). 
The perturbation $\delta$ 
is iteratively generated as follows:
\begin{align}
\footnotesize
    \delta \leftarrow \nabla_{\delta} L(F_\theta(x + \delta,y)) - \nabla_{\delta} L(F_\theta(x,y)), \,
    \delta' = \frac{0.5 \delta}{||\delta||_\infty}.
\end{align}
Algorithm~\ref{alg:DAG} shows the details of DAG.

\begin{algorithm}[t]
\small
\caption{FastADT: FGSM adv. training}
\label{alg:FastADT}
\LinesNumbered
\KwIn{Segmentation model $F_\theta$, perturbation budget $\epsilon$, learning rate $\alpha$, \#epochs $T$,
a training set $\mathbb{D}_{tr}$.}
\KwOut{Model parameters $\theta$.}
Initialize: $\delta^{(0)} = 0 \in \mathbb{B} = \{ \delta:  \|\delta\|_p \leq \epsilon \}$. \\
\For{$t=1,2,\cdots,T$}{
    \For{$(x, y) \in \mathbb{D}_{tr}$}{ 
    // Perform FGSM adversarial attack \\
    $\delta = \textrm{Uniform}(-\epsilon, \epsilon)$ \\
    
    $\delta = \textrm{Clip}(\delta + \alpha \cdot \textrm{sign} (\nabla_{\delta} L(F_\theta(x+\delta),y)))$ \\
    
    //  Update model parameters using, e.g., SGD \\
    
    $\theta = \theta - \nabla_{\theta} L(F_\theta(x+\delta),y)$ \\
    }
}
\Return{$\theta$}
\end{algorithm}

\noindent {\bf Fast Adversarial Training (FastADT).}
FastADT aims to 
speeds up PGD-based adversarial training (AT)~\cite{madry2018towards}.  It combines FGSM-based AT 
with random initialization. Though simple and efficient, it shows comparable defense performance with PGD-based AT. 
Algorithm~\ref{alg:FastADT} illustrates the FastADT.

\begin{definition}[Lipschitz continuous]
\label{def:lips}
A real-valued function $f(\cdot)$ is $C$-Lipschitz continuous with respect to $l_p$ norm, if, for any two inputs $z_1$ and $z_2$, 
\begin{equation}\label{eq:Lipschitz}
|f(z_1) - f(z_2)| \le {C} ||z_1 - z_2||_p.
\end{equation}
\end{definition}

\vspace{-2mm}
\begin{definition}[Convex function]
\label{def:cvx}
A real-valued function $f(\cdot)$ is convex if and only if the following inequality hold for any two inputs $z_1$ and $z_2$ in domain and $\rho\in(0,1)$,
\begin{equation}\label{eq:Convex}
f(\rho z_1+(1-\rho)z_2) \le \rho f(z_1) + (1-\rho) f(z_2).
\end{equation}
\end{definition}

\begin{lemma}[\cite{salman2019provably}]
\label{lem:lem1}
For any measurable function $g: \mathbb{R}^{N} \to [0,1]$, defining $\hat g(x)={E_{ \varepsilon \sim N(0,{\sigma ^2})}}g(x + \varepsilon)$, then $x \mapsto {\Phi ^{ - 1}}( \hat g(x))$ is $1/\sigma$-Lipschitz, where $\Phi ^{ - 1}$ is the inverse of the standard {Gaussian} 
CDF.
\end{lemma}

\subsection{Proof of Theorem \ref{thm:pwcr}}
\label{supp:pwcr}

\begin{proof} 
Our proof is based on Lemma \ref{lem:lem1} and Definition~\ref{def:lips}. 

\noindent Given a testing image $x$ with pixel labels $y$. 
For each pixel $x_n$, $F_\theta(x)_{n,c}$ is the model's probability of predicting label $c$ on $x_n$, and it satisfies $F_\theta{(x)_{n,c}}:x\to[0,1]$. From the definition of the smoothed model $G_\theta$ of $F_\theta$, we have $G_\theta{(x)_{n,c}} = {E_{\beta \sim N(0,{\sigma ^2})}}F_\theta{(x + \beta )_{n,c}}$. 
Let $y_n = \max_{c} G_\theta{(x)_{n,c}}$ and $y'_n = \max_{c \neq y_n} G_\theta{(x)_{n,c}}$. 
Using Lemma~\ref{lem:lem1}, we know $x \mapsto {\Phi ^{ - 1}}(G_\theta{(x)_{n,c}})$ is $1/\sigma$-Lipschitz. Then, for any pixel perturbation $||\delta||_2$, we have
\begin{small}
\begin{align}
& {\Phi ^{ - 1}}(G_\theta{(x + \delta )_{{n,y_n}}}) > {\Phi ^{ - 1}}(G_\theta{(x)_{{n,y_n}}}) - \frac{1}{\sigma }||\delta||_2.
\label{equ champ} \\
& {\Phi ^{ - 1}}(G_\theta{(x)_{n,y'_n}}) + \frac{1}{\sigma }||\delta||_2  > {\Phi ^{ - 1}}(G_\theta{(x + \delta )_{n,y'_n}}).
\label{equ runner}
\end{align}
\end{small}%
To guarantee that the perturbed pixel $(x + \delta)_n$ is also correctly predicted, i.e., $y_n = \arg \max_{c} G_\theta(x+\delta)_{n,c}$, we need that for any perturbation $\delta$, 
\begin{align}
\small
{\Phi ^{ - 1}}(G_\theta{(x + \delta )_{{n,y_n}}}) > {\Phi ^{ - 1}}(G_\theta{(x + \delta )_{n,y'_n}}). 
\label{equ jk}
\end{align}
In other words, we require
\begin{small}
\begin{align}
{\Phi ^{ - 1}}(G_\theta{(x)_{{n,y_n}}}) - \frac{1}{\sigma }||\delta||_2  > {\Phi ^{ - 1}}(G_\theta{(x)_{n,y'_n}}) + \frac{1}{\sigma }||\delta||_2,
\label{equ jk}
\end{align}
\end{small}%
which implies 
\begin{align}
\small
\label{cr_delta}
||\delta||_2 < \frac{\sigma}{2} (\Phi ^{ - 1}(G_\theta{(x)_{{n,y_n}}}) - \Phi ^{ - 1}(G_\theta{(x)_{n,y'_n}})).
\end{align}
According the Definition~\ref{def:pwcr}, we then have the pixel-wise $l_2$ certified radius as follows: 
\begin{small}
\begin{align*}
    cr(x_n) = \max {\|\delta\|_2} = \frac{\sigma}{2} (\Phi ^{ - 1}(G_\theta{(x)_{{n,y_n}}}) - \Phi ^{ - 1}(G_\theta{(x)_{n,y'_n}})). 
\end{align*}
\end{small}%
Note that $G_\theta{(x)_{n,y'_n}} \leq  1 - G_\theta{(x)_{{n,y_n}}}$ and $\Phi^{-1}$ is an increasing function.  By setting $G_\theta{(x)_{n,y'_n}} =   1 - G_\theta{(x)_{{n,y_n}}}$, we have 
\begin{align}
\label{eqn:pwcr_final}
cr(x_n) = \sigma \Phi ^{ - 1}(G_\theta{(x)_{n,y_n})}.
\end{align}
\end{proof}

\vspace{-6mm}

\subsection{Proof of Theorem \ref{th:grab}}
\label{supp:grab}

\begin{proof}
We first prove that $\hat{g}^{TPGE}$ is an unbiased estimator of $\nabla\hat{L}(\delta)$. The proof that $\hat{g}_{cr}^{TPGE}$ is an unbiased estimator of $\nabla\hat{L}_{cr}(\delta)$ is similar. Recall that $\hat{g}^{TPGE} = \frac{N}{2\gamma}\big(L(\delta+\gamma\hat{u})-L(\delta-\gamma\hat{u})\big)\hat{u}$ for a sampled $\hat{u} \in \mathcal{S}_p$. 
According to Equation~\ref{eq:est1}, we have the following two equations for directions $\bm{u}$ and $-\bm{u}$,
\begin{small}
\begin{align}
& \nabla\hat{L}(\delta) = \mathbb{E}_{\bm{u}\in\mathcal{S}_p}[\frac{N}{\gamma}L(\delta+\gamma\bm{u})\bm{u}],\\
& \nabla\hat{L}(\delta) = -\mathbb{E}_{\bm{u}\in\mathcal{S}_p}[\frac{N}{\gamma}L(\delta-\gamma\bm{u})\bm{u}].
\end{align}
\end{small}%
Combining the two equations and rearranging terms, we have $\nabla\hat{L}(\delta) = \mathbb{E}_{\bm{u}\in\mathcal{S}_p}[\frac{N}{2\gamma}\big(L(\delta+\gamma\bm{u})-L(\delta-\gamma\bm{u})\big)\bm{u}]$. Thus, we have $\nabla\hat{L}(\delta) = \mathbb{E}_{\hat{u}} [\hat{g}^{TPGE}]$, proving the unbiasedness of $\hat{g}^{TPGE}$.

Second, we show that $\hat{g}^{TPGE}$ has a bounded norm. The proof that $\hat{g}_{cr}^{TPGE}$ has a bounded norm is similar.
As $L(\cdot)$ is $\hat{C}$-Lipschitz continuous (See Definition~\ref{def:lips}) with respect to $l_p$ norm, we have  
$|L(\delta_1) - L(\delta_2)| \le \hat{C}||\delta_1 - \delta_2||_p.$
Then, 
\begin{equation}
\footnotesize
\begin{split}
||\hat{g}^{TPGE}||_p &= ||\frac{N}{2\gamma}\big(L(\delta+\gamma\bm{u})-L(\delta-\gamma\bm{u})\big)\bm{u}||_p\\
&=\frac{N}{2\gamma}|L(\delta+\gamma\bm{u})-L(\delta-\gamma\bm{u})|\cdot||\bm{u}||_p\\
&\stackrel{(a)}{=}\frac{N}{2\gamma}|L(\delta+\gamma\bm{u})-L(\delta-\gamma\bm{u})|\\
&\stackrel{(b)}{\le}\frac{N}{2\gamma}\hat{C}|| (\delta+\gamma\bm{u}) - (\delta-\gamma\bm{u})||_p\\
&=\frac{N}{2\gamma}\hat{C}||2\gamma\bm{u}||_p = N\hat{C},
\end{split}
\end{equation}
where (a) is due to that $u$ is a random unit vector and (b) is due to the Lipschitz continuity of $L$. 
\end{proof}

\subsection{Proof of Theorem \ref{th:bgd}}
\label{supp:main}

\begin{proof}
Before analyzing the regret bound of our proposed PBGD and CR-PBGD black-box attack, we first show the regret bound in the white-box setting where the Lipschitz continuous and convex loss function is revealed to the attacker such that the gradient can be exactly derived. Specifically, Zinkevich \cite{zinkevich2003online} shows that the regret bound is $\mathcal{O}(\sqrt{T})$ using online projected gradient descent. When using the stochastic gradient whose expectation is equal to the true gradient, the online projected gradient descent still achieves an $\mathcal{O}(\sqrt{T})$ regret bound, which is presented in the following theorem:
\begin{theorem}[\cite{zinkevich2003online}]\label{th:ogd}
For a $\tilde{C}$-Lipschitz convex loss function $L$ and letting $\tilde{g}^{(t)}$ be the stochastic gradient in round $t$ satisfying $\mathbb{E}[\tilde{g}^{(t)}] = \nabla L$ and $||\tilde{g}^{(t)}||\le \tilde{C}$. By setting a learning rate $\alpha = \frac{\sqrt{N}}{2\tilde{C}\sqrt{T}}$ in the online projected gradient descent for any total iteration $T>0$, then the regret 
incurred by the online projected gradient descent is no more than $\sqrt{N}\tilde{C}\sqrt{T}/2$, i.e.,
\begin{equation}
\small
\mathbb{E}\{\sum_{t=1}^TL(\delta^{(t)})\} - \max_{\delta}\sum_{t=1}^TL(\delta)\le\sqrt{N}\tilde{C}\sqrt{T}/2 = \mathcal{O}(\sqrt{T}),
\end{equation}
where $\delta^{(t)} = \textrm{Proj}_{\mathbb{B}}(\delta^{(t-1)} + \lambda \tilde{g}^{(t)})$.
\end{theorem}

Next, we derive the regret bound for PBGD for simplicity and the proof for CR-PBGD  is similar. 
Note that PBGD performs the exact gradient descent on the smoothed  function $\hat{L}$ because $\hat{g}^{TPGE}$ is an unbiased gradient estimator of $\nabla \hat{L}$. 
Based on Theorem \ref{th:grab}, we have $||\hat{g}||_p\le N\hat{C}$. From Theorem \ref{th:ogd}, we can bound the expected regret on $\hat{L}(\delta)$ as follows,
\begin{small}
\begin{equation}\label{eq:initRB}
\sum_{t=1}^T\mathbb{E}[\hat{L}(\delta^{(t)})] - \max_\delta \sum_{t=1}^T\hat{L}(\delta) \le N^{3/2}\hat{C}\sqrt{T}/2.
\end{equation} 
\end{small}%
According to the definition of $\hat{L}(\delta)$ in Equation \ref{eqn:smoothedloss}, we can bound the difference between $\hat{L}(\delta)$ and $L(\delta)$ as,
\begin{equation}
\small
\begin{split}
|\hat{L}(\delta) - L(\delta)| &\stackrel{(a)}{=} |\mathbb{E}_{\bm{v}\in\mathcal{B}_p}[L(\delta +\gamma\bm{v})] - L(\delta)|\\
&=\mathbb{E}_{\bm{v}\in\mathcal{B}_p}[|L(\delta +\gamma\bm{v}) - L(\delta)|]\\
&\stackrel{(b)}{\le}\mathbb{E}_{\bm{v}\in\mathcal{B}_p}[\hat{C}||\gamma\bm{v}||_p]
=\hat{C}\gamma,
\end{split}
\label{loss_conv}
\end{equation}
where (a) is due to the definition of $\hat{L}$ and (b) is due to the Lipschitz continuity of the loss function $L$. 

Then, we can bound 
$|\hat{L}(\delta^{(t)}) - L(\delta)|$ as follows:
\begin{equation}
\small
\begin{split}
|\hat{L}( \delta^{(t)}) - L(\delta)| &= |\hat{L}(\delta^{(t)}) - L(\delta^{(t)}) + L(\delta^{(t)}) - L(\delta)|\\
&\le |\hat{L}(\delta^{(t)}) - L(\delta^{(t)})| + |L(\delta^{(t)}) - L(\delta)|\\
&\stackrel{(a)}{\le} \hat{C}\gamma + \hat{C}\gamma = 2\hat{C}\gamma,
\end{split}
\end{equation}
where we apply Equation~\ref{loss_conv} to the first term and the Lipschitz continuity of attack loss $L(\delta)$ for the second term.

With the above inequality, we can obtain the lower bound of  Equation~\ref{eq:initRB} as,
\begin{small}
\begin{equation}\label{eq:initRB2}
\begin{split}
&\sum_{t=1}^T\mathbb{E}[\hat{L}(\delta^{(t)})] - \max_\delta \sum_{t=1}^T\hat{L}(\delta)\\
&\ge \sum_{t=1}^T\mathbb{E}[L(\delta^{(t)}) - 2\hat{C}\gamma] - \max_\delta \sum_{t=1}^T(L(\delta^{(t)}) + \hat{C}\gamma)
\end{split}
\end{equation}
\end{small}
Combining Equation~\ref{eq:initRB} and Equation~\ref{eq:initRB2}, rearranging terms, we can bound the regret as,
\begin{equation}
\small
\begin{split}
R^{PBGD}_{\mathcal{A}}(T) &=\sum_{t=1}^T\mathbb{E}[L(\delta^{(t)})] - \max_\delta \sum_{t=1}^TL(\delta)\\
& = \sum_{t=1}^T\mathbb{E}[L(\delta^{(t)})] - T L(\delta_*) \\
&\le N^{3/2}\hat{C}\sqrt{T}/2 + 3T\hat{C}\gamma,
\end{split}
\end{equation}
By equalizing the two terms in r.h.s of the above inequality, we can minimize the regret bound, which is $N^{3/2}\hat{C}\sqrt{T}$. By doing so, the approximation parameter $\gamma$ is set to be $\gamma=\frac{N^{3/2}}{6\sqrt{T}}$, thus completing the proof. 
\end{proof}

\subsection{Theoretical Contributions of Our Bound}
\label{sec:regret}
Our derived sublinear regret bound is not only  useful to design effective black-box attacks, but also has theoretical contributions 
for bandit optimization. 
We first review the different techniques used in the existing bandit methods, and then show our contribution.  

\vspace{+0.5mm}
\noindent {\bf Different bandit methods use different techniques.} 
Under the same assumption (i.e., convex loss with loss feedback), 
state-of-the-art bandit methods~\cite{larson2019derivative,ilyas2018black,hazan2014bandit,bubeck2017kernel,suggala2021efficient,saha2011improved} use different techniques to derive the regret bound (except \cite{larson2019derivative,ilyas2018black} that do not have a regret bound). Specifically, \cite{hazan2016optimal,suggala2021efficient} are based on the original loss function, while \cite{saha2011improved} and our method construct a smoothed loss function for the original loss function via randomization. \cite{hazan2016optimal} does not estimate gradients and derives a probabilistic regret bound based on the ellipsoid method. \cite{bubeck2017kernel} does not estimate gradients and aims to recover the full loss function via the loss feedback using kernel methods, but can only derive a probabilistic regret bound. \cite{suggala2021efficient} estimates the gradient and Hessians of the original loss functions from a one-point feedback and uses self-concordant regularizers to iteratively update perturbations; \cite{suggala2021efficient} also derives a probabilistic regret bound. \cite{saha2011improved} estimates (unbiased) gradients for the smoothed function based on one-point feedback, and iteratively obtains perturbations using self-concordant regularizers (but not a gradient-descent based method), similar to \cite{suggala2021efficient}.  Our method estimates unbiased gradients for the smooth loss function based on two-point feedback. In addition, it iteratively updates perturbations based on the graceful gradient-descent framework.  
These two advantages enable our derived regret bound to be deterministic and tight. 

\vspace{+0.5mm}
\noindent {\bf Our regret bound is deterministic and tight.} 
State-of-the-art bandit methods~\cite{larson2019derivative,ilyas2018black,hazan2014bandit,bubeck2017kernel,suggala2021efficient,saha2011improved} 
either do not have a regret bound,  or they have a probabilistic regret bound, or a loose deterministic regret bound---indicating that their performance is suboptimal when applied to the black-box attack problem. Specifically, 
\cite{larson2019derivative,ilyas2018black} do not derive a regret bound; \cite{hazan2014bandit,bubeck2017kernel,suggala2021efficient} have a probabilistic regret bound  $O(T^{1/2} \log(1/\epsilon)$  with probability $1-\epsilon$; and \cite{saha2011improved} has a loose deterministic regret bound  $O(T^{2/3})$. 
In contrast, our bandit algorithm derives a tight and deterministic $O(T^{1/2})$ regret bound, indicating our attack reaches optimal performance with probability 100\% and the optimal rate.

\end{document}